\documentclass{article}

\usepackage{amsthm}

\newtheorem{theorem}{Theorem}
\newtheorem{lemma}{Lemma}
\newtheorem{proposition}{Proposition}
\newtheorem{corollary}{Corollary}
\newtheorem{example}{Example}
\newtheorem{remark}{Remark}

\newtheorem{assumption}{Assumption}

\usepackage[nonatbib,main,final]{neurips_2026}

\usepackage{enumitem}
\usepackage[utf8]{inputenc} % allow utf-8 input
\usepackage[T1]{fontenc}    % use 8-bit T1 fonts
\usepackage{url}            % simple URL typesetting
\usepackage{booktabs}       % professional-quality tables
\usepackage{amsfonts}       % blackboard math symbols
\usepackage{nicefrac}       % compact symbols for 1/2, etc.
\usepackage{microtype}      % microtypography
\usepackage{xcolor}

\usepackage[hidelinks]{hyperref}
\usepackage{breqn}
\usepackage{siunitx}

\usepackage{comment}
\usepackage{wrapfig}
\usepackage{epsfig}
\usepackage{graphicx}
\usepackage{amsmath}
\usepackage{amssymb}
\usepackage{xspace}
\usepackage{algorithm2e}
\usepackage{algpseudocode}
\algrenewcommand\algorithmiccomment[1]{\hfill$\triangleright$ #1}

\usepackage{tabularx}
\usepackage{multirow}
\usepackage{subcaption}
\usepackage{booktabs}
\usepackage{xspace}
\usepackage[capitalise]{cleveref}
\SetKwInput{Input}{Input}
\SetKwProg{kwServer}{ServerUpdate}{:}{}
\SetKwProg{kwClient}{ClientUpdate}{:}{}

\crefname{theorem}{Theorem}{Theorems}
\crefname{lemma}{Lemma}{Lemmas}
\crefname{proposition}{Proposition}{Propositions}
\crefname{corollary}{Corollary}{Corollaries}
\crefname{example}{Example}{Examples}
\crefname{remark}{Remark}{Remarks}
\crefname{definition}{Definition}{Definitions}
\crefname{assumption}{Assumption}{Assumptions}

\SetKwProg{kwServer}{\textcolor{blue}{Server Update}}{}{}
\SetKwProg{kwClient}{\textcolor{blue}{Client Update}}{}{}

\DeclareMathOperator*{\argmin}{arg\;min}
\usepackage{makecell}
\usepackage{bbm}
\usepackage{wrapfig}
\usepackage{graphicx}
\usepackage{ifthen}
\usepackage{xspace}
\usepackage{upgreek}
\usepackage{pifont}% http://ctan.org/pkg/pifont
\usepackage{colortbl}
\usepackage{cite}
\usepackage{multirow}
\usepackage{dsfont}
\usepackage{float}
\usepackage[numbers]{natbib} 

\usepackage{tikz}
\usepackage{pgfplots}
\pgfplotsset{compat=newest}
\pgfplotsset{every axis legend/.append style={cells={anchor=west}}}
\usetikzlibrary{arrows}
\tikzset{>=stealth'}
\definecolor{paperggplotspo}{RGB}{248, 118, 109}
\definecolor{paperggplotrspo}{RGB}{0, 191, 196}
\definecolor{paperggplotdelta}{RGB}{0, 0, 0}

\title{FedRSPO+: A Heterogeneity-aware Algorithm for Decision-focused Federated Learning}

\author{%
  Konstantinos Ziliaskopoulos \\
  Auburn University \\
  \texttt{kzz0034@auburn.edu} \\
  \And
  Alexander Vinel \\
  Auburn University \\
  \texttt{alexander.vinel@auburn.edu} \\
  \And
  Jiaqi Wang\thanks{Corresponding author.} \\
  Auburn University \\
  \texttt{jqwang@auburn.edu} \\
}

\begin{document}

\maketitle
\setcounter{footnote}{0}
\begin{abstract}
Decision-focused learning (DFL) trains predictive models for downstream optimization, but existing methods largely assume centralized data. In cross-silo settings, federated learning is a natural alternative, yet standard federated pipelines optimize prediction quality rather than decision quality and do not account for client heterogeneity in downstream objectives or feasible sets. Heterogeneity is especially problematic for DFL methods: under heterogeneous polyhedral decision problems, small perturbations can cause discontinuous changes in optimal decisions, leading to unstable client updates and aggregation. We propose \textbf{FedRSPO+}, a heterogeneity-aware framework for decision-focused federated learning. Our approach is built on \textbf{RSPO+}, a regularized predict-then-optimize surrogate that smooths the decision map through projection, enabling stable optimization even when clients face different objectives and constraints. We show that RSPO+ upper bounds downstream decision error and regret of the regularized decision and, under exact regularization and consistent LP solution selection, the original LP decision. We also demonstrate that FedRSPO+ has cross-client heterogeneity bounds that \textit{(i)} scale with both objective and feasible-set heterogeneity, \textit{(ii)} vanish at homogeneity, and \textit{(iii)} do not require strong convexity. We develop an annealed and modular training procedure that is compatible with standard federated personalization and aggregation methods. We evaluate FedRSPO+ on three experimental settings: controlled synthetic knapsack instances that isolate multiple axes of heterogeneity, a shortest path DFL benchmark, and an energy pricing case study based on real data. Across these experiments, we compare against prediction-only federated learning and DFL frameworks under varying heterogeneity and communication budgets. Together, our results suggest that smoothing is a useful ingredient for stable collaborative decision learning and provide a heterogeneity-aware foundation for federated DFL\footnote{The code for the paper is available at \url{https://github.com/kziliask/RSPO_Plus}.}.
\end{abstract}

\section{Introduction}
Decision-focused learning (DFL) has emerged as a common approach to tackle data-driven contextual optimization  \citep{mandi2024decision,qi2022integrating,sadana2025survey,mivsic2020data}. Among DFL methods, surrogate loss functions allow practitioners to take advantage of deep learning innovations in operational contexts, with the Smart Predict-then-Optimize (SPO) framework \citep{elmachtoub2022smart} being widely adopted. Its convex surrogate loss, SPO+, provides a tractable upper bound on the downstream decision loss and has seen broad use \citep{zhang2025real,lui2024gradient,wang2026spo,tian2023smart}. However, cross-silo applications of DFL face a key constraint: collaborative analytics are often prohibited due to regulatory, privacy, and competitive considerations. 

Federated learning (FL), which enables collaborative model training without centralizing data, offers a natural training procedure for addressing these constraints \citep{kairouz2021advances,chai2024survey,liu2024recent}. 
In cross-silo FL, a global model is trained across multiple organizations without centralizing data: each client keeps its local dataset and shares only model updates (or related statistics), which are aggregated to update a global model \citep{mcmahan2017communication}.
There exists an extensive FL literature on mitigating data, model, system, and task heterogeneity \citep{li2020federated,wang2023towards,li2021ditto,li2024unveiling,ye2023heterogeneous,fraboni2023general,wang2024bridging}.
\begin{wrapfigure}{r}{0.4\textwidth}
  \centering
  \vspace{-10pt}
  \includegraphics[
  width=\linewidth,
    trim={0 107 0 175},  % left bottom right top
    clip]{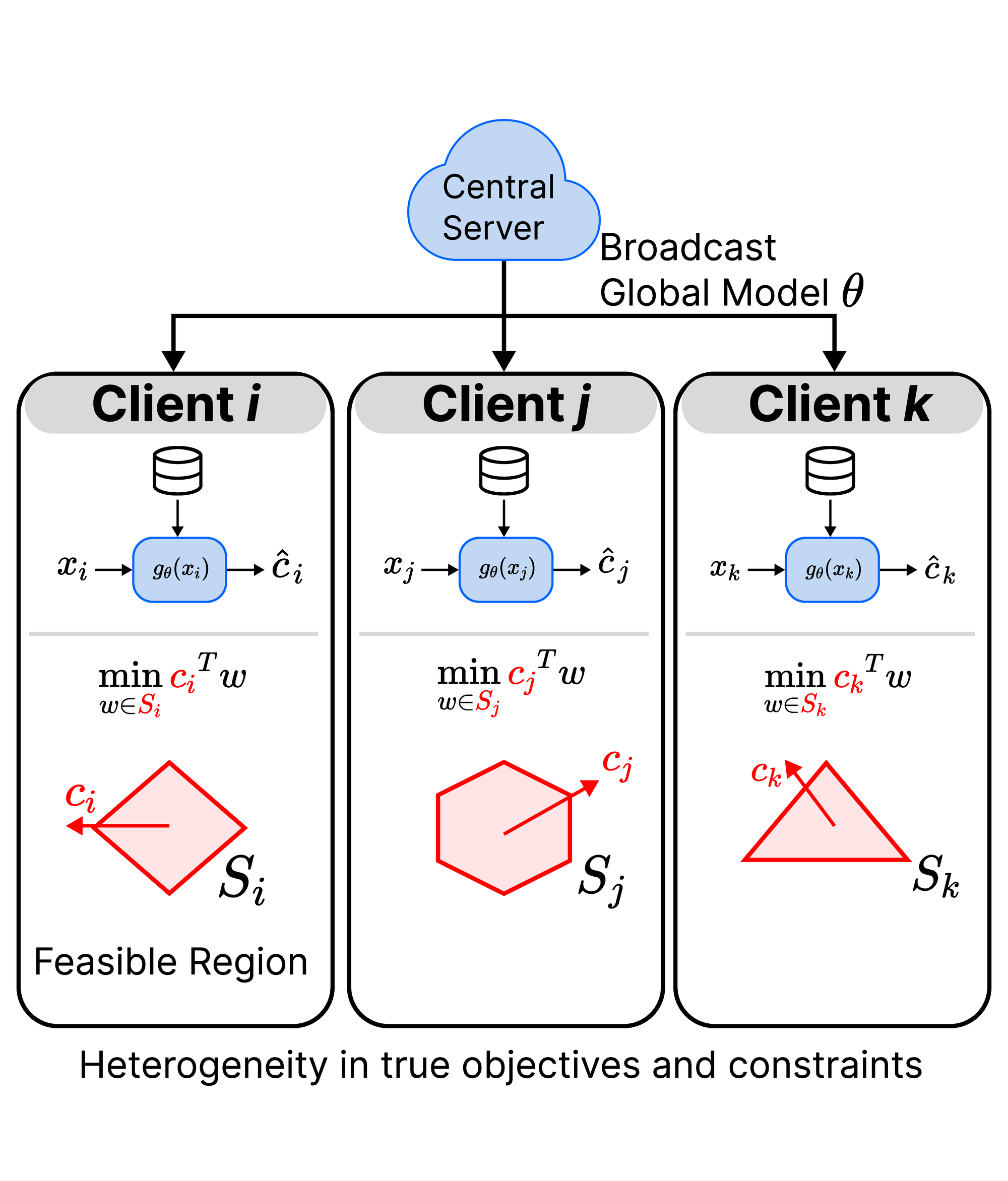}
  \caption{DFFL setting: a central server broadcasts a global model to heterogeneous clients, each with its own objective function and feasible region.}
  \label{fig:intro-dffl}
  \vspace{-4.8mm}
\end{wrapfigure}
In the DFL setting, though, when organizations collaboratively train DFL models while keeping data local, heterogeneity in client objectives and constraints creates a fundamental challenge: standard federated aggregation becomes brittle because the optimal solution map can be discontinuous across clients. 
Let client $i$ have covariates $x \sim \mathcal{X}_i$, a client-specific true objective map 
$c_i : \mathcal{X}_i \to \mathbb{R}^d$, and a shared federated predictor $g \in \mathcal{H}$ used in the downstream optimization model $\min_{w \in S_i} g(x)^T w$.
Here, the objective map $c_i(\cdot)$ and feasible region $S_i$ can differ across clients,  
which induces \emph{objective} and \emph{constraint} heterogeneity that direct naive implementation of a decision-focused federated learning (DFFL) does not address (\cref{fig:intro-dffl}). For example, SPO+, the most prominent DFL approach, relies on an optimal solution map that can be discontinuous when clients have different linear objectives over polyhedral feasible sets. 
Near boundaries where the optimal basis changes, arbitrarily small perturbations can switch the optimizer to a different extreme point, producing a $O(\mathrm{diam}(S_i))$ change in the optimal decision. 

We propose RSPO+, a smoothed surrogate that replaces the discontinuous optimal solution map with a projection-based decision, and FedRSPO+, a federated training framework built on RSPO+ with per-client calibration and co-annealed scheduling. Our contributions are summarized as follows:

\begin{itemize}[leftmargin=8pt, itemsep=2pt, topsep=2pt, parsep=0pt]
    \item We introduce \textbf{RSPO+} and \textbf{FedRSPO+}. RSPO+ is a regularized
    surrogate loss that smooths the decision map while retaining a provable connection
    to the original downstream regret. FedRSPO+ builds on RSPO+ to provide a modular
    federated algorithm compatible with standard FL personalization and aggregation
    methods \citep{t2020personalized,zhang2022personalized,li2020federated,li2021ditto}.
    \item We formally show why naive federation of SPO+ can fail and how FedRSPO+ avoids the same issues, developing point-wise cross-client loss bounds that degrade gracefully with heterogeneity, vanish at homogeneity, and require no strong convexity.
    \item We evaluate the approach on three tasks \citep{ho2022risk,pogancic2020Differentiation,pjm_lmp_2025}: synthetic knapsack, shortest path, and PJM energy pricing. We compare against FedAvg, naive SPO+ federation and perturbation methods across communication budgets and controlled heterogeneity, demonstrating strong improvements in convex resource-allocation settings and competitive performance under structural mismatch.
\end{itemize}
\section{Related Work}
\paragraph{Decision-Focused Learning and Contextual Optimization}
Decision-focused federated learning is largely unexplored. The closest prior work is \citep{ziliaskopoulos2026decisionfocusedfederatedlearningheterogeneous}, which analyzed federating SPO+ under client-level objective and feasible-set heterogeneity. Their results showed that naive SPO+-based federation can be brittle in heterogeneous polyhedral problems: optimizer instability induces point-wise heterogeneity bounds with an irreducible diameter term, while stronger geometric assumptions such as strong convexity yield tighter stability guarantees. They also proposed Interp-DFFL, which interpolates between local and federated models using client validation data, which is complementary to our method, and we compare against it in \cref{sec:experiments}. We build on this foundation by introducing a smoothed surrogate, RSPO+, and a modular federated training procedure that replaces discontinuous optimizer jumps with stable projection-based decisions, thereby removing the diameter penalty and yielding bounds that vanish at homogeneity.

More broadly, DFL is a sub-topic of the broader area of contextual optimization \citep{qi2022integrating,sadana2025survey}. SPO+ is among the most widely used surrogates in this literature, with strong theoretical support, including generalization bounds \citep{el2019generalization} and calibration guarantees \citep{liu2021risk,ho2022risk}. Other decision-aware surrogates include perturbation-based gradient losses \citep{gupta2024decision} and Fenchel--Young losses \citep{blondel2020learning}, while black-box and implicit differentiation methods differentiate through the solver directly rather than introducing an explicit surrogate \citep{pogancic2020Differentiation,sahoo2023backpropagation,blondel2022efficient,agrawal2019differentiable}. Recent work has also begun to study DFL across multiple downstream tasks and changing objectives, suggesting a growing interest in DFL beyond the single-problem setting, though federated DFL still remains underexplored \citep{capitaine2026online,tang2023multi,pmlr-v286-rodriguez-diaz25a}.

\paragraph{Federated Learning under Heterogeneity}
FL, in turn, includes extensive literature on heterogeneity, including proximal penalties \citep{li2020federated}, control variates \citep{karimireddy2020scaffold}, localized or partially personalized models \citep{li2021fedbn,wang2023knowledge,arivazhagan2019federated,wang2024pfedclub,collins2021exploiting,oh2022fedbabu}, and regularized bilevel formulations \citep{t2020personalized}. However, these methods are designed primarily for heterogeneity in data distributions, model heterogeneity, or predictive objectives, not for heterogeneity in downstream optimization structure. This distinction is especially important in DFL, where small perturbations can induce large changes in the optimal decision. Notably, \citep{t2020personalized} use Moreau envelopes as a personalization regularizer, whereas we use Moreau smoothing on the decision-focused surrogate itself to stabilize the decision map. More generally, Bregman divergences are widely used in machine learning and optimization as geometry-aware discrepancy measures, supporting methods such as mirror descent, clustering, and structured prediction \citep{banerjee2005clustering,beck2003mirror,raskutti2015information}. We leverage this to construct an adaptive surrogate aligned with each feasible set. Taken together, prior work addresses either decision-aware learning in centralized settings or heterogeneity-aware federated learning for predictive objectives, but not federated DFL with heterogeneous downstream feasible sets.
\section{Method}
\label{sec:method}
\begin{figure}
    \centering
    \includegraphics[width=0.85\linewidth,
    trim={120 200 102 170},  % left bottom right top
    clip
    ]{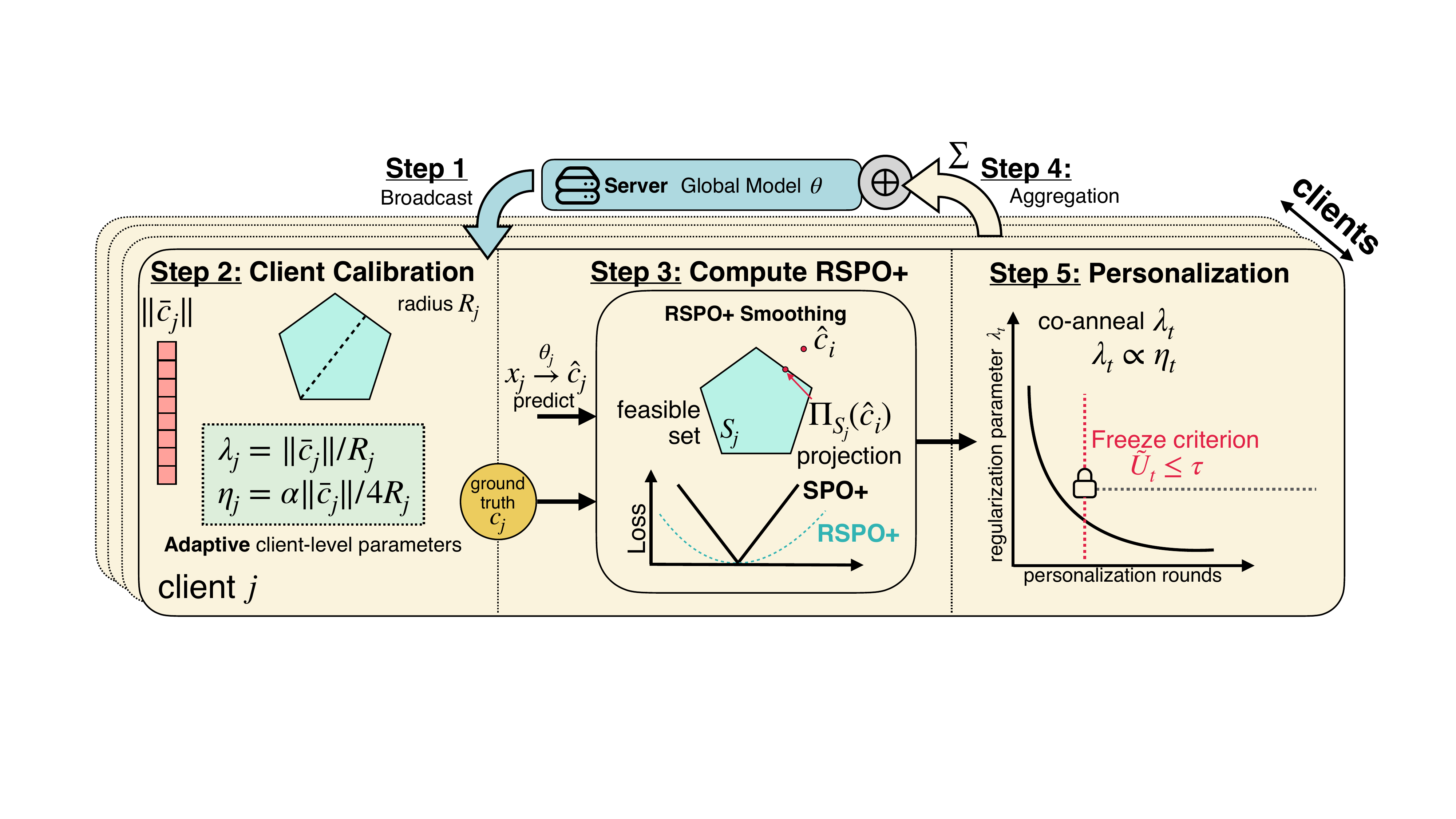}
    \caption{\textbf{FedRSPO+: A heterogeneity-aware federated decision-focused learning pipeline.} The server broadcasts a global model, clients perform personalized calibration, local prediction, and projection-based smoothing (RSPO+), and updates are aggregated. During personalization, co-annealed regularization enables stable training under heterogeneous objectives and feasible sets.}
    \label{fig:methodology}
    \vspace{-3mm}
\end{figure}
\begin{table}[t]
\caption{Overview of the four stages of FedRSPO+, along with the frequency and setting of each stage. We elaborate on each component throughout \cref{sec:method}.}
\label{tab:training-procedure}
\centering
\footnotesize
\begin{tabular}{llll}
\toprule
\textbf{Component} & \textbf{Execution} & \textbf{Overview} & \textbf{Communication} \\
\midrule
\textbf{(1) Calibration}
& client, once
& set $\lambda_j,\eta_j$ (Eq.~\eqref{eq:calibration})
& -- \\

\textbf{(2) Federated training}
& server, $R_{\mathrm{fed}}$ rounds
& SGD on $F_j(\theta;\lambda_j)$, FedAvg ($\propto n_j$)
& model weights $\theta$  \\

\textbf{(3) Personalization}
& client, $R_{\mathrm{per}}$ rounds
& SGD on $F_j(\theta;\lambda_{j,t})$, co-anneal $\lambda,\eta$
& -- \\

\textbf{(4) Freezing}
& client
& stop annealing when $\tilde U_{j,t}\le\tau$
  (Eq.~\eqref{eq:stopping-condition})
& -- \\
\bottomrule
\end{tabular}
\vspace{4pt}
\end{table}
% \vspace{-16pt}
\cref{fig:methodology} and \cref{tab:training-procedure} provide an overview of the four components of the FedRSPO+ pipeline.
We begin by defining the notation used throughout this paper, the tools we use to derive our methods, and by stating the assumptions for the client feasible sets and ground truth objectives, which are typical in DFL literature \citep{mulamba2021contrastive,amos2017optnet,elmachtoub2022smart,agrawal2019differentiable}. All proofs are given in the Appendix.

%\paragraph{Preliminaries and Definitions.}
We use $\| \cdot \| = \| \cdot \|_2$ throughout this study. Each feasible set $S$ has a diameter and radius,$ \ D = \sup_{w,w' \in S}\|w-w'\|, \  R = \sup_{w \in S} \|w\|.$
The downstream optimization problem is defined as $
z_{\mathrm{LP}} = \min_{w \in S} c^Tw$,
and the quadratically regularized linear program (QLP) as $
    z_{\mathrm{QLP}} = \min_{w \in S} c^Tw + \frac{\lambda}{2}\|w\|^2.
$
The support function $\xi$ of feasible set $S$ is defined as $
    \xi_S(u) = \sup_{v \in S} \langle u, v\rangle .$ Recall the generalized Bregman divergence is defined as $D^s_\phi(x,y) = \phi(x) - \phi(y) - s^T(x-y)$, where $s \in \partial \phi(y)$.
The Moreau envelope of a function with smoothing parameter $\lambda$ is defined as $
    M_{\lambda f} (v)= \inf_{u \in \mathbb{R}^d} \left\{ f(u) + \frac{1}{2\lambda}\|u - v\|^2 \right\}.
$
The SPO+ loss function \citep{elmachtoub2022smart} is defined as $
    \ell_{\mathrm{SPO+}}(c,\hat{c}) = \xi_S(c-2\hat{c}) - 2w^{\star}(c)(c-\hat{c}) + z^{\star}(c),
$
where $S$ is the feasible set of the downstream optimization problem, $\hat{c}= g(x)$ is the model output, $c$ the ground truth cost vector, $w^{\star} = \argmin_{w \in S} c^Tw $ and $z^{\star}$ the optimal objective function value.

\begin{assumption}[Bounded, non-empty feasible sets]
\label{assum:bounded-set}The feasible sets of interest are nonempty and compact, so $\xi_S(\cdot)$ is finite and maxima/minima
in the sets are attained.
\end{assumption}

\begin{assumption}[Convex feasible sets]\label{assum:convex} For each instance, the feasible sets of interest $S \subseteq \mathbb{R}^d$ are convex.
\end{assumption}
\begin{assumption}[Bounded, non-zero objective]\label{assum:bounded-objective}The objective function vector's norm $0<C_{\min}\leq \|c\| \leq C_{\max}$ is bounded above and strictly bounded away from zero, while the prediction vector's norm $\|\hat{c}\| \leq B$ is also bounded above.
\end{assumption}

\subsection{RSPO+}
\label{subsec:rspo_plus}

\begin{proposition}[SPO+ as a Bregman divergence]
\label{prop:spo-bregman}
The SPO+ loss admits the representation
$\ell_{\mathrm{SPO+}}(c,\hat{c}) = D^s_{\xi}(c-2\hat{c},-c), \text{where} \ s = w^\star(c) \in \partial \xi_S(-c)$ \citep{ziliaskopoulos2026decisionfocusedfederatedlearningheterogeneous}. 
\end{proposition}

Motivated by Proposition~\ref{prop:spo-bregman}, we define a smoothed variant by replacing $\xi_S$ with its $\lambda$-Moreau envelope, $
M_{\lambda\xi_S}(u) = \tfrac{1}{2\lambda}\|u\|^2 - \tfrac{\lambda}{2}\,d_S^2\!\Big(\tfrac{u}{\lambda}\Big)$,
where $d_S^2(u)=\inf_{s\in S}\|u-s\|^2$. We call the resulting loss \textbf{RSPO+}:
\[
\ell_{\mathrm{RSPO+}}(c,\hat{c};\lambda) := D_{M_{\lambda\xi_S}}(c-2\hat{c},-c).
\]

\begin{proposition}[Closed form of RSPO+]
\label{prop:rspo}
Let $u=(c-2\hat{c})/\lambda$, $v=-c/\lambda$, and $\Pi_S$ denote projection onto $S$. Then
\[
\boxed{\ell_{\mathrm{RSPO+}}(c,\hat{c};\lambda) = \tfrac{\lambda}{2}\Big(\|u-\Pi_S(v)\|^2 - \|u-\Pi_S(u)\|^2\Big).}
\]
\end{proposition}

\begin{remark}[Derivation of RSPO+]
    RSPO+ admits an equivalent derivation from decision-focused optimization perspective by applying the SPO+ construction to the QLP, resulting in  $\ell_{\mathrm{RSPO+}}$.
\end{remark}

\begin{corollary}[Gradient of Smoothed Loss]
\label{cor:rspo-gradient}
    From \cref{prop:rspo}, the gradient of RSPO+ is
\[
\boxed{\nabla_{\hat{c}} \ell_{\mathrm{RSPO+}}(c,\hat{c};\lambda) = -2\Big(\Pi_S(u)-\Pi_S(v)\Big).
}
\]
\end{corollary}

\begin{proposition}[Characteristics of RSPO+]
\label{prop:characteristics}
    Define $\ell_{\mathrm{RSPO+}}$ from \cref{prop:rspo}. It has the following characteristics:
    \begin{itemize}
        \item $\ell_{\mathrm{RSPO+}} \geq 0$ as a divergence.
        \item $\ell_{\mathrm{RSPO+}}$ is globally convex with respect to $\hat{c}$.
        \item $\ell_{\mathrm{RSPO+}}$ is 2D-Lipschitz continuous under compactness of $S$.
        \item $\ell_{\mathrm{RSPO+}}$ is $4/\lambda-$smooth, and therefore $\lambda/4$-cocoercive.
    \end{itemize}
\end{proposition}

\begin{theorem}
\label{thm:decision-bound}
    Define $f(c,w) = c^Tw + \lambda/2 \|w\|^2$, $w^{\star} = \Pi_S(-c/\lambda)$ the optimal action, and $\hat{w} = \Pi_S(-\hat{c}/\lambda)$ the predicted action. The true regularized regret is $R(c,\hat{w}) = f(c,\hat{w}) - f(c,w^{\star}),$ therefore, 
\begin{equation}
\label{eq:rspo-regret-bound}
\ell_{\mathrm{RSPO+}} \geq R(c,\hat{w}) + \lambda\|\hat{w} - w^{\star}\|^2. 
\end{equation}
Furthermore, because $R(c,\hat{w}) \geq 0$, we also have,
\begin{equation}
\label{eq:rspo-decision-bound}
    \|\hat{w} - w^{\star}\| \leq \sqrt{\frac{\ell_{\mathrm{RSPO+}}}{\lambda}}.
\end{equation}
\end{theorem}
\begin{corollary}[Unregularized Regret Bound for RSPO+]
\label{cor:unreg-bound}
Let $z^\star = \min_{w \in S} c^T w$ be the unregularized objective, and $w^\star_{\mathrm{LP}}$ be the minimizer. For any prediction $\hat{c}$ and regularization parameter $\lambda$, the true unregularized regret $R_{\mathrm{unreg}}(c, \hat{w}) = c^T \hat{w} - z^\star$ of the predicted decision $\hat{w} = \Pi_S(-\hat{c}/\lambda)$ satisfies:
\begin{equation}
    \label{eq:unreg-worst-case}
    R_{\mathrm{unreg}}(c, \hat{w}) \leq \ell_{\mathrm{RSPO+}}(c, \hat{c}; \lambda) + \frac{\lambda}{2}R^2.
\end{equation}
\end{corollary}
\subsection{FedRSPO+}
\cref{subsec:rspo_plus} established bounds on the \emph{regularized} regret and decision error, with \cref{cor:unreg-bound} connecting the RSPO+ loss to the unregularized regret of the regularized decision $\hat{w}_\lambda = \Pi_S(-\hat{c}/\lambda).$
In our experiments, however, evaluation is performed using the original downstream decision rule,
$\hat{w}_{\mathrm{LP}}(\hat{c}) \in \arg\min_{w\in S} \hat{c}^{\top}w,$
rather than the regularized decision. Thus, \cref{cor:unreg-bound} should be interpreted as a direct guarantee for the regularized decision and as a surrogate-level guarantee for the evaluated LP decision. For polyhedral feasible sets, this distinction disappears once $\lambda$ is sufficiently small under a problem-dependent exact-regularization threshold: the quadratically regularized solution is then also LP-optimal, with the regularizer selecting a particular LP optimum \citep{friedlander2008exact,gonzalez2025quantitative}. Under this condition, the RSPO+ bound also applies to the LP decision used at evaluation. In practice, our freeze rule in \cref{eq:stopping-condition} stops annealing before $\lambda$ reaches this threshold, and we quantify this residual gap in \cref{appsubsec:qp-lp} in the Appendix.

More generally, as $\lambda\rightarrow 0$, the bias term $\frac{\lambda}{2}R^2$ in \cref{eq:unreg-worst-case} vanishes, motivating an annealing schedule that gradually transitions from a smooth, stable surrogate toward the original downstream LP objective. In this section, we build on this connection by introducing a regularization parameter scheduler, which we refer to as an annealing schedule. Finally, we derive bounds on cross-client loss discrepancy under optimization heterogeneity and compare them with bounds on SPO+ loss to show why FedRSPO+ can mitigate discontinuities from vertex jumping.
\paragraph{Annealing Schedule.} Corollary 2 shows that RSPO+ controls the original unregularized regret up to an $O(\lambda R^2)$ bias. Thus, $\lambda$ induces a bias-stability tradeoff:
larger $\lambda$ smooths the decision map and stabilizes optimization, while smaller $\lambda$ better aligns the surrogate with the original LP objective. This motivates annealing $\lambda$ during local training. Specifically, we gradually decrease $\lambda$ so that the regularized decision tracks the original downstream optimum (i.e., as a homotopy continuation from the regularized to the non-regularized objective). We justify this schedule by analyzing how the RSPO+ loss and gradient change under perturbations of $\lambda$ in \cref{prop:lambda-bound}.

\begin{proposition}[RSPO+ Loss and Gradient under perturbation of $\lambda$]
\label{prop:lambda-bound}
For $\lambda_1 > \lambda_2 > 0$ regularization parameters of RSPO+ loss, we can bound the difference of loss under parameter shift $\lambda_1 \rightarrow \lambda_2$ as,
\begin{equation}
\label{eq:rspo-loss-lambda-bound}
    |\ell_{\mathrm{RSPO+}}(\lambda_1) - \ell_{\mathrm{RSPO+}}(\lambda_2)| \leq R^2|\lambda_1 - \lambda_2| + 2\|c\| \|c-\hat{c}\| \left|\frac{1}{\lambda_1} - \frac{1}{\lambda_2}\right|,
\end{equation}
 and difference in RSPO+ gradients as,
 \begin{equation}
 \label{eq:rspo-grad-lambda-bound}
     \|\nabla \ell_{\mathrm{RSPO+}}(\lambda_1) - \nabla \ell_{\mathrm{RSPO+}}(\lambda_2)\| \leq 2\Big(\|c-2\hat{c}\| + \|c\|\Big) \left|\frac{1}{\lambda_1} - \frac{1}{\lambda_2}\right|.
 \end{equation}
\end{proposition}
\begin{remark}[Inverse time scheduling]
\label{rem:scheduling}
\cref{prop:lambda-bound} shows that RSPO+ is stable to changes in $\lambda$ primarily through the reciprocal $\frac{1}{\lambda}$. This motivates parameterizing the schedule linearly in $1/\lambda_t$, i.e. $\lambda_t = \lambda_0 / (1 + \kappa t)$ which yields a constant per-step reciprocal increment $|\frac{1}{\lambda_{t+1}} - \frac{1}{\lambda_{t}}| = \kappa /\lambda_0$ and therefore uniformly controls both its loss and gradient drift.
\end{remark}

\textbf{Stopping condition.}
As reducing $\lambda$ weakens the smoothness properties of RSPO+ in \cref{prop:characteristics}, we devise a stopping condition for the $\lambda$-schedule in FedRSPO+ that depends on the regularization bias $c_i^T w_{i,t}^{\star,\mathrm{reg}} - z_i^\star$.

For a subset of the client dataset $\mathcal{V}_c$, at local step $t$ with $\lambda_t$, define
$\hat{w}_{i,t} = \Pi_{\mathcal{S}_i}(-\hat{c}_{i,t}/\lambda_t)$ and
$w_{i,t}^{\star,\mathrm{reg}} = \Pi_{\mathcal{S}_i}(-c_i/\lambda_t)$.
Let $z_i^\star = \min_{w \in \mathcal{S}_i} c_i^T w$ (pre-solved once on $\mathcal{V}_c$). 
We stop annealing when the normalized average regularization bias over $\mathcal{V}_c$ falls below a tolerance $\tau$:
\begin{equation}
\label{eq:stopping-condition}
\tilde{U}_t
= \frac{\frac{1}{|\mathcal{V}_c|} \sum_{i \in \mathcal{V}_c}
\left(c_i^T w_{i,t}^{\star,\mathrm{reg}} - z_i^\star\right)}
{\frac{1}{|\mathcal{V}_c|} \sum_{i \in \mathcal{V}_c} |z_i^\star| + \varepsilon}
\leq \tau.
\end{equation}
Intuitively, this criterion freezes $\lambda_t$ once the surrogate regularized target is sufficiently aligned with the true downstream LP objective. Thus, $\tilde{U}_t$ empirically tracks the QP--LP regularization bias discussed above and gives a data-dependent stopping rule. The full FedRSPO+ algorithm is provided in pseudo-code in \cref{alg:fedrspoplus} in the Appendix.

\begin{remark}[MSE--DFL homotopy]
\label{rem:mse-spo}
When $0 \in \operatorname{int}(S)$, there exists $\bar\lambda$ such that for all
$\lambda \ge \bar\lambda$, the projections in Corollary 1 act locally as the
identity, and $\nabla_{\hat c}\ell_{\rm RSPO+}$ reduces to a scaled MSE gradient.
As $\lambda$ decreases, the projection becomes active and the surrogate becomes
increasingly decision-focused. Together with \cref{cor:unreg-bound}, this gives a continuous
bias-stability path from prediction-focused learning toward the original LP
objective.
\end{remark}

\begin{theorem}[FedRSPO+ Convergence Rate with linear predictors]
\label{thm:fedrspo-convergence-short}
Under standard regularity assumptions, the weighted average iterate
$\hat{\Theta}_T$ of FedRSPO+ satisfies
\[
\mathbb{E}\!\left[F_T(\hat{\Theta}_T)-F_T(\Theta^\star)\right]
=
O\!\left(
\frac{1}{\sum_{t=0}^{T-1}\eta_t}
+
\frac{\sum_{t=0}^{T-1}\eta_t^2}{\sum_{t=0}^{T-1}\eta_t}
+
\frac{\sum_{t=0}^{T-1}\eta_t\Delta_t}{\sum_{t=0}^{T-1}\eta_t}
\right),
\]
with an $E^2$-dependent constant from local updates. For a fixed-$\lambda$ objective and $\eta_t\asymp 1/t$, the bound reduces to
$O(1/\log T)$. For the actual annealed objective, the additional tracking term appears, which vanishes after freezing and the bound becomes the standard constant-stepsize form
$O(1/(T\eta_f)+\eta_f)$. The formal statement and proof are in the Appendix.
\end{theorem}

We now characterize how the point-wise RSPO+ loss degrades under both objective shift and feasible set shift. This parallels the SPO+ heterogeneity analysis of \citep{ziliaskopoulos2026decisionfocusedfederatedlearningheterogeneous}, but has improved stability due to the smoothness of the RSPO+. We measure objective heterogeneity by $c_d = \|c_{\mathrm{i}} - c_{\mathrm{j}}\|$ and feasible-set heterogeneity by the Hausdorff distance $\delta = d_H(S_{\mathrm{i}}, S_{\mathrm{j}})$. Unlike SPO+, the origin-centered quadratic penalty in the Moreau envelope breaks translation invariance, so we use Hausdorff distance directly rather than the shape distance $\delta_N$.  Let $R_{\max} = \max\{R_{\mathrm{i}}, R_{\mathrm{j}}\}$ denote the joint bounding radius.
\begin{wraptable}{r}{0.48\textwidth}
\centering
\footnotesize
\caption{Heterogeneity bound comparison. $c_d = \|c_{\mathrm{ref}} - c_{\mathrm{cl}}\|$ is the objective shift, $\delta$ is the Hausdorff set distance, $\delta_N$ is the translation-invariant Hausdorff set distance, and $D_{\min}$ is the minimum set diameter. \textbf{Q1}: Rate?; \textbf{Q2}: Vanishes at homogeneity?; \textbf{Q3}: Requires strong convexity?}
\label{tab:bound-comparison}
\begin{tabular}{lcc}
\toprule
 & \textbf{SPO+} & \textbf{RSPO+ (Ours)} \\
\midrule
Q1 & $O\!\big(c_d + \delta_N + D_{\min}\big)$ & $O\!\big(c_d/\lambda + \sqrt{\delta/\lambda} + \delta\big)$ \\[4pt]
Q2 & No ($D_{\min}$ persists) & \textbf{Yes} \\[4pt]
Q3 & Yes, for basic stability & \textbf{No} \\
\bottomrule
\end{tabular}
\end{wraptable}
The key difference from SPO+ is in how decisions respond to perturbations. In SPO+, the optimal decision $w^\star(c) \in \arg\min_{w \in S}\, c^T w$ selects a vertex of $S$. Under infinitesimal perturbations, this vertex can jump discontinuously across the entire feasible set, introducing a diameter penalty $D_{\min}$ that does not vanish even as heterogeneity tends to zero. In RSPO+, the optimal decision is the Euclidean projection
$v^\star_{\lambda,S}(c) = \Pi_S(-c/\lambda)$,
which is stable under small changes to the feasible set (measured in Hausdorff distance). For any compact convex $S_1, S_2$ with $d_H(S_1, S_2) \le \delta$:
\begin{equation}
\label{eq:proj-stability}
\big\|\Pi_{S_1}(p) - \Pi_{S_2}(p)\big\| \;\le\; \sqrt{2\delta\big(\|p\| + R_{\max}\big)}.
\end{equation}
This $\sqrt{\delta}$ stability holds for \emph{all} convex sets, including polyhedrals, without any strong convexity assumption.

\begin{example}[Why SPO+ struggles with polyhedral clients]
Consider two clients with the same polyhedral feasible set $S$ and cost vectors $c^{(i)}, c^{(j)}$, differing by $\epsilon$. For SPO+, the optimizer can switch between extreme points of $S$, so the decision gap between clients remains $\Omega(D(S))$ even as $\|c^{(i)} - c^{(j)}\| \rightarrow 0$. By contrast, RSPO+ uses projected decisions $v^\star_{\lambda S}(c) = \Pi_S(-c/\lambda),$ and since projection is non-expansive,
\begin{dmath*}
    \|v^\star_{\lambda S}(c^{(i)}) - v^\star_{\lambda S}(c^{(j)})\| \leq \frac{1}{\lambda} \|c^{(i)} - c^{(j)}\|.
\end{dmath*}
Thus, RSPO+ removes the discontinuous vertex-switching in polyhedral instances.
\end{example}

\begin{theorem}[RSPO+ loss discrepancy under heterogeneity]
\label{thm:main-rspo}
Let $(c_{\mathrm{i}}, S_{\mathrm{i}})$ and $(c_{\mathrm{j}}, S_{\mathrm{j}})$ satisfy our assumptions, with $c_d = \|c_{\mathrm{i}} - c_{\mathrm{j}}\|$, $c_\min = \min\{\|c_i\|,\|c_j\|\}$, and $\delta = d_H(S_{\mathrm{i}}, S_{\mathrm{j}})$. Then for any prediction $\hat{c}$:
\begin{equation}
\label{eq:rspo-main}
\big|\ell_{S_{\mathrm{i}}}(\hat{c}, c_{\mathrm{i}}) - \ell_{S_{\mathrm{j}}}(\hat{c}, c_{\mathrm{j}})\big| \;\le\; T(\hat{c}) + B(\hat{c}),
\end{equation}
where the \textbf{envelope shift} $T(\hat{c})$ and \textbf{decision shift} $B(\hat{c})$ are defined as follows.

The envelope shift captures how the Moreau envelope values change due to combined objective and set perturbation:
\begin{dmath}
\label{eq:T-bound}
T(\hat{c}) = D_i c_d + \delta\big(\|c_{\mathrm{j}} - 2\hat{c}\| + \|c_{\mathrm{j}}\| + 2\lambda R_{\max}\big),
\end{dmath}

and the decision shift captures how the projection-based decisions drift, amplified by the prediction:
\begin{dmath}
\label{eq:B-bound}
B(\hat{c}) = \big(2\|\hat{c}\| + 2\lambda R_{\max}\big) \frac{c_d}{\lambda} + \sqrt{2\delta\!\bigg(\frac{c_\min}{\lambda} + R_{\max}\bigg)}.
\end{dmath}

\end{theorem}

Combining the regularization bias in \cref{cor:unreg-bound} with the explicit $\lambda$-dependence in \cref{thm:main-rspo} gives the following informal bias--stability tradeoff:
\[
\mathrm{error}(\lambda)
\lesssim
\lambda R^2
+
\frac{c_d}{\lambda}
+
\sqrt{\frac{\delta}{\lambda}},
\]
up to constants depending on the bounded costs, predictions, and feasible-set radii. The first term is the $O(\lambda R^2)$ regularization bias, while the latter two terms come from the objective and feasible-set heterogeneity terms in \cref{thm:main-rspo}. Thus, decreasing $\lambda$ improves alignment with the original LP objective but worsens stability under client heterogeneity. This motivates annealing only until the regularized objective is sufficiently aligned with the original downstream objective, rather than reducing $\lambda$ to zero and replacing RSPO+ with SPO+. We formalize this relationship in the Appendix.

\paragraph{Calibration of $\lambda$ and $\eta$.}
From the RSPO+ loss structure, $\lambda$ scales the cost vectors relative to the feasible set: too large and the loss reduces to MSE, too small and it reverts to SPO+ (see \cref{rem:mse-spo}).  For the loss to alternate between regimes, $\lambda$ should place the scaled vector near the feasible-set boundary.  Combined with the $4/\lambda$-smoothness of RSPO+ (Proposition~\ref{prop:characteristics}), which suggests a step size of order $\lambda/4$~\citep{spiliopoulos2025mathfoundations}, we initialize per client~$j$:
\begin{equation}
    \lambda_j := \frac{\|\bar{c}_j\|}{R_j}, \qquad
    \eta_j := \frac{\alpha\,\|\bar{c}_j\|}{4\,R_j},
    \label{eq:calibration}
\end{equation}
where $\|\bar{c}_j\|$ is the average cost-vector magnitude in client $j$'s dataset, $R_j$ is the feasible-set radius, and $\alpha$ is a small positive scalar (e.g.\ $\alpha = 0.01$).  When annealing $\lambda$, we co-anneal $\eta$ with $\kappa^{\mathrm{LR}} \propto \kappa^\lambda$ to maintain this step-size relationship. This is also validated by the convergence analysis in \cref{appsubsec:convergence}.  A sensitivity analysis over $\kappa$ is in Appendix \cref{appsec:implementation-details}.
\section{Experiment}
\label{sec:experiments}
With each experiment, we test a different aspect of FedRSPO+. Our synthetic fractional knapsack experiment uses controlled heterogeneity across all three axes (feasible set, objective, and data) to test \cref{thm:main-rspo} empirically and the performance of our method against the state-of-the-art. Our PJM experiment uses real price dynamics and heterogeneous feasible sets to test our method on real datasets. Finally, our Warcraft experiment has heterogeneous starts and destinations to vary feasible sets and datasets across clients to evaluate how FedRSPO+ performs when the optimization task relies on the structure of the unregularized objective, and how RSPO+'s regularization affects performance.

For all experiments, we use a federated warm-start followed by local personalization. Our model is a shallow neural network with 64 hidden units for the fractional knapsack and PJM tasks, and a truncated ResNet-18 for the Warcraft task following prior work \citep{mandi2024decision}. We report relative regret in \cref{tab:results} and the optimality gap in \cref{tab:warcraft-results}. While FedRSPO+ is trained on a regularized objective, we report on the actual downstream regret. We also compare with popular perturbation DFL methods, perturbed Fenchel-Young Loss (PFYL, \citep{blondel2020learning}) and Differentiable Perturbation Optimizers (DPO, \citep{berthet2020learning}). We also compare with Interp-DFFL \citep{ziliaskopoulos2026decisionfocusedfederatedlearningheterogeneous} and an OptNet-style baseline \citep{amos2017optnet}. Implementation details, hyperparameters, and solver configurations are provided in the Appendix \ref{appsec:implementation-details}, and additional centralized and local-only baselines in Appendix \cref{appsubsec:additional-baselines}.

\begin{table}[t]
\centering
\caption{Mean client relative regret (\%) $\downarrow$ across heterogeneity configurations. Each axis ($\eta_{\mathrm{o}}$:objective$, \eta_{\mathrm{c}}$:constraint$, \eta_{\mathrm{d}}$:data) is varied independently while others are fixed at 0. ``Fed'': global federated model; ``+P'': after local personalization. Results: mean $\pm$ std over 5 seeds.}
\label{tab:results}
\sisetup{
  separate-uncertainty   = true,
  table-format           = 2.1(1),
  table-align-uncertainty= true,
  detect-weight          = true,
  detect-inline-weight   = math,
}
\setlength{\tabcolsep}{3pt}
\resizebox{\textwidth}{!}{%
\begin{tabular}{l cc cc cc cc cc}
\toprule
 & \multicolumn{2}{c}{MSE} & \multicolumn{2}{c}{PFYL} & \multicolumn{2}{c}{DPO} & \multicolumn{2}{c}{SPO+} & \multicolumn{2}{c}{RSPO+ (Ours)} \\
\cmidrule(lr){2-3} \cmidrule(lr){4-5} \cmidrule(lr){6-7} \cmidrule(lr){8-9} \cmidrule(lr){10-11}
$(\eta_{\mathrm{o}}, \eta_{\mathrm{c}}, \eta_{\mathrm{d}})$ & Fed & +P & Fed & +P & Fed & +P & Fed & +P & Fed & +P \\
\midrule
$(0, 0, 0)$   & 15.9 $\pm$ 3.3 & 10.5 $\pm$ 2.8 & 8.5 $\pm$ 1.5 & 6.1 $\pm$ 1.1 & 16.3 $\pm$ 2.7 & 12.6 $\pm$ 2.2 & 9.7 $\pm$ 2.1 & 5.6 $\pm$ 1.4 & \textbf{2.3 $\pm$ 0.4} & \textbf{1.9 $\pm$ 0.3} \\ \addlinespace
$(0.5, 0, 0)$ & 16.0 $\pm$ 3.4 & 10.6 $\pm$ 2.8 & 8.8 $\pm$ 1.5 & 6.3 $\pm$ 1.2 & 16.7 $\pm$ 2.5 & 12.8 $\pm$ 2.3 & 10.0 $\pm$ 2.1 & 5.8 $\pm$ 1.5 & \textbf{2.5 $\pm$ 0.3} & \textbf{1.9 $\pm$ 0.3} \\
$(1, 0, 0)$   & 16.4 $\pm$ 3.4 & 10.8 $\pm$ 2.9 & 9.9 $\pm$ 1.6 & 6.9 $\pm$ 1.3 & 17.4 $\pm$ 2.5 & 13.4 $\pm$ 2.2 & 10.9 $\pm$ 2.2 & 6.3 $\pm$ 1.6 & \textbf{3.0 $\pm$ 0.3} & \textbf{2.0 $\pm$ 0.3} \\ \addlinespace
$(0, 0.5, 0)$ & 13.2 $\pm$ 3.5 & 9.8 $\pm$ 2.7 & 13.3 $\pm$ 2.4 & 8.1 $\pm$ 1.9 & 19.6 $\pm$ 3.6 & 13.1 $\pm$ 3.1 & 11.9 $\pm$ 2.5 & 7.3 $\pm$ 2.1 & \textbf{3.9 $\pm$ 0.6} & \textbf{2.3 $\pm$ 0.4} \\
$(0, 1, 0)$   & 13.3 $\pm$ 2.3 & 10.5 $\pm$ 2.0 & 15.0 $\pm$ 2.3 & 9.1 $\pm$ 1.5 & 20.8 $\pm$ 3.3 & 14.1 $\pm$ 2.4 & 12.9 $\pm$ 2.0 & 8.1 $\pm$ 1.5 & \textbf{4.6 $\pm$ 0.5} & \textbf{2.6 $\pm$ 0.4} \\ \addlinespace
$(0, 0, 0.5)$ & 17.5 $\pm$ 3.0 & 9.6 $\pm$ 2.2 & 14.0 $\pm$ 1.8 & 8.2 $\pm$ 1.2 & 20.6 $\pm$ 3.2 & 14.4 $\pm$ 2.4 & 14.1 $\pm$ 2.0 & 6.9 $\pm$ 1.3 & \textbf{6.5 $\pm$ 0.8} & \textbf{3.8 $\pm$ 0.2} \\
$(0, 0, 1)$   & 19.1 $\pm$ 2.8 & 8.4 $\pm$ 1.7 & 18.0 $\pm$ 1.9 & 8.8 $\pm$ 0.7 & 22.9 $\pm$ 2.9 & 14.1 $\pm$ 1.8 & 17.7 $\pm$ 2.1 & 6.9 $\pm$ 0.8 & \textbf{9.3 $\pm$ 1.2} & \textbf{3.4 $\pm$ 0.2} \\
\bottomrule
\end{tabular}
}
\end{table}

\subsection{Fractional Knapsack}\label{subsec:frac-knap} To test FedRSPO+ against different axes of heterogeneity in the client dataset, we implement a fractional knapsack experiment, using a synthetic dataset for feature size $p = 8$, cost vector size $d = 50$, and 20 clients. We test 3 different levels of heterogeneity, represented by $(\eta_{\mathrm{o}}, \eta_{\mathrm{c}}, \eta_{\mathrm{d}})$ in \cref{tab:results}: no heterogeneity, moderate heterogeneity, and extreme heterogeneity. Overall, RSPO+ has the lowest relative regret across the board. The largest gap is during federation, meaning RSPO+ succeeds in developing an effective model for all clients despite heterogeneity. Across all methods, data heterogeneity was most detrimental to model performance, with constraint heterogeneity moderately affecting regret and objective heterogeneity minimally. RSPO+ was found resilient to all forms of heterogeneity compared to the baseline methods.

We also test additional baselines. As shown in \cref{tab:interp-optnet-res},
FedRSPO+ has lower regret than Interp-DFFL in the homogeneous, objective-heterogeneous, and feasible-set-heterogeneous settings. Under pure data heterogeneity $(0,0,1)$, however, Interp-DFFL is stronger (2.47\% vs.\ 3.37\%), consistent with its explicit local--federated interpolation, which targets statistical heterogeneity. As RSPO+ is a loss and Interp-DFFL an aggregation rule, they are not mutually exclusive. The OptNet-style baseline, which uses the same regularized QP with a different training signal, is substantially weaker (9.4--11.3\%), suggesting that the gains come from the RSPO+ surrogate, not just quadratic regularization.

\begin{wraptable}{r}{0.48\linewidth}
\caption{
Comparison of FedRSPO+ against Interp-DFFL and OptNet across fractional knapsack heterogeneity settings and PJM cases.
Values are mean $\pm$ standard deviation across training seeds of mean absolute regret.
}
\label{tab:interp-optnet-res}
\centering
\footnotesize
\setlength{\tabcolsep}{4pt}
\begin{tabular}{lccc}
\toprule
Setting & Interp-DFFL & OptNet-style & FedRSPO+ \\
\midrule
$(0,0,0)$ & $2.42\pm0.30$ & $9.4\pm1.9$ & $\mathbf{1.86\pm0.25}$ \\
$(1,0,0)$ & $2.42\pm0.33$ & $9.7\pm2.0$ & $\mathbf{2.02\pm0.26}$ \\
$(0,1,0)$ & $3.51\pm0.77$ & $11.3\pm2.0$ & $\mathbf{2.58\pm0.38}$ \\
$(0,0,1)$ & $\mathbf{2.47\pm0.39}$ & $11.3\pm1.9$ & $3.37\pm0.22$ \\
\midrule
PJM (\$) & $69.6\pm1.7$ & $69.8\pm1.3$ & $\mathbf{60.7\pm0.8}$ \\
\bottomrule
\end{tabular}
\end{wraptable} We also run a sweep across relevant parameters, shown in \cref{fig:multifig-main}, to test the performance of our method against SPO+ through different fractions of active clients per round, number of clients, and different values of $\alpha$. We see that the communication budget has a modest impact on the performance of FedRSPO+, with a general improvement over SPO+ as the fraction increases. The parameter $\alpha$ was moderately more impactful on FedRSPO+ performance, with relative improvement ranging from 1.1\% for $\alpha=10^{-3}$ to 5.4\% for $\alpha=0.1$. However, the most surprising result is for increasing number of clients. There FedRSPO+ dramatically surpasses the federated SPO+ setup, with a difference of 12.5\% at 50 clients. Extra sweep results for $\kappa^\lambda, \kappa^{\mathrm{LR}},$ and $\tau$ are available in Appendix \cref{appsec:implementation-details}. In all cases, FedRSPO+ outperforms the SPO+ implementation.

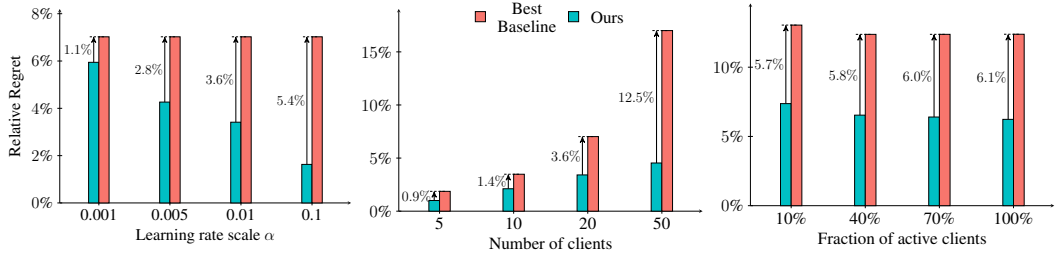
\begin{figure}
    \centering
    \begin{minipage}{0.333\linewidth}
        \centering
        \resizebox{\linewidth}{!}{\newcommand{\figfont}{\large}
\begin{tikzpicture}[font=\Large]
\begin{axis}[
scale only axis,
  legend style={
    at={(0.5,1.05)}, anchor=south,
    draw=none, fill=none,
    font=\figfont,
    legend columns=-1,
    column sep=8pt
  },
  ylabel={Relative Regret},
  xlabel={Learning rate scale $\alpha$},
  width=9cm, height=6cm,
  xmin=0.45, xmax=4.55,
  ymin=0, ymax=0.082,
  xtick={1, 2, 3, 4},
  xticklabels={0.001, 0.005, 0.01, 0.1},
    yticklabel = {%
  \pgfmathparse{100*\tick}%
  \pgfmathprintnumber[fixed,precision=0]{\pgfmathresult}\%
},
  ytick={0, 0.02, 0.04, 0.06, 0.08},
  tick align=outside,
  scaled x ticks=false,
  scaled y ticks=false,
  ymajorgrids=false,
  xmajorgrids=false,
  axis line style={draw=black},
  tick style={draw=black},
  clip=false,
  axis lines=left,
]

\addplot+[
  ybar, bar width=9pt, bar shift=-4.5pt,
  fill=paperggplotrspo, draw=black, line width=0.4pt,
  mark=none
] coordinates {
  (1.0, 0.059450096245213996)
  (2.0, 0.042641115285860266)
  (3.0, 0.03413376074714526)
  (4.0, 0.016292524119291565)
};

\addplot+[
  ybar, bar width=9pt, bar shift=4.5pt,
  fill=paperggplotspo, draw=black, line width=0.4pt,
  mark=none
] coordinates {
  (1.0, 0.07023837314775863)
  (2.0, 0.07023837314775863)
  (3.0, 0.07023837314775863)
  (4.0, 0.07023837314775863)
};

% Alpha = 1
\draw[->, line width=0.9pt, paperggplotdelta]
  ([xshift=-4.5pt]axis cs:1, 0.059450) -- ([xshift=-4.5pt]axis cs:1, 0.070238);
\draw[dash pattern=on 2pt off 1.5pt, line width=0.4pt, paperggplotdelta]
  ([xshift=4.5pt]axis cs:1, 0.070238) -- ([xshift=-9pt]axis cs:1, 0.070238);
\node[font=\figfont, inner sep=1pt, text=paperggplotdelta, anchor=west]
  at ([xshift=-32pt]axis cs:1, 0.064844) {$1.1\%$};

% Alpha = 2
\draw[->, line width=0.9pt, paperggplotdelta]
  ([xshift=-4.5pt]axis cs:2, 0.042641) -- ([xshift=-4.5pt]axis cs:2, 0.070238);
\draw[dash pattern=on 2pt off 1.5pt, line width=0.4pt, paperggplotdelta]
  ([xshift=4.5pt]axis cs:2, 0.070238) -- ([xshift=-9pt]axis cs:2, 0.070238);
\node[font=\figfont, inner sep=1pt, text=paperggplotdelta, anchor=west]
  at ([xshift=-32pt]axis cs:2, 0.056440) {$2.8\%$};

% Alpha = 3
\draw[->, line width=0.9pt, paperggplotdelta]
  ([xshift=-4.5pt]axis cs:3, 0.034134) -- ([xshift=-4.5pt]axis cs:3, 0.070238);
\draw[dash pattern=on 2pt off 1.5pt, line width=0.4pt, paperggplotdelta]
  ([xshift=4.5pt]axis cs:3, 0.070238) -- ([xshift=-9pt]axis cs:3, 0.070238);
\node[font=\figfont, inner sep=1pt, text=paperggplotdelta, anchor=west]
  at ([xshift=-32pt]axis cs:3, 0.052186) {$3.6\%$};

% Alpha = 4
\draw[->, line width=0.9pt, paperggplotdelta]
  ([xshift=-4.5pt]axis cs:4, 0.016293) -- ([xshift=-4.5pt]axis cs:4, 0.070238);
\draw[dash pattern=on 2pt off 1.5pt, line width=0.4pt, paperggplotdelta]
  ([xshift=4.5pt]axis cs:4, 0.070238) -- ([xshift=-9pt]axis cs:4, 0.070238);
\node[font=\figfont, inner sep=1pt, text=paperggplotdelta, anchor=west]
  at ([xshift=-32pt]axis cs:4, 0.043265) {$5.4\%$};

\end{axis}
\end{tikzpicture}}
    \end{minipage}%
    \begin{minipage}{0.333\linewidth}
        \centering
        \resizebox{\linewidth}{!}{\newcommand{\figfont}{\large}
\begin{tikzpicture}[font=\Large]
\begin{axis}[
scale only axis,
  legend style={
    at={(0.5,0.85)}, anchor=south,
    align=center,
    draw=none, fill=none,
    font=\Large,
    legend columns=-1,
    column sep=8pt
  },
  ylabel={},
  xlabel={Number of clients},
  width=9cm, height=6cm,
  xmin=0.45, xmax=4.55,
  ymin=0, ymax=0.190455,
  xtick={1, 2, 3, 4},
  xticklabels={5, 10, 20, 50},
    yticklabel = {%
  \pgfmathparse{100*\tick}%
  \pgfmathprintnumber[fixed,precision=0]{\pgfmathresult}\%
},
  tick align=outside,
  scaled x ticks=false,
  scaled y ticks=false,
  ymajorgrids=false,
  xmajorgrids=false,
  axis line style={draw=black},
  tick style={draw=black},
  clip=false,
  axis lines=left,
]
\addplot+[
  ybar, bar width=9pt, bar shift=-4.5pt,
  fill=paperggplotrspo, draw=black, line width=0.4pt,
  mark=none,
  legend image code/.code={
    \draw[draw=black, line width=0.4pt, fill=paperggplotspo]
      (0cm,-0.1cm) rectangle (0.3cm,0.2cm);}
] coordinates {
  (1.0, 0.010046619212065505)
  (2.0, 0.021111353383989395)
  (3.0, 0.03413376074714526)
  (4.0, 0.04537133115943295)
};
\addlegendentry{Best \\ Baseline}
\addplot+[
  ybar, bar width=9pt, bar shift=4.5pt,
  fill=paperggplotspo, draw=black, line width=0.4pt,
  mark=none,
  legend image code/.code={
    \draw[draw=black, line width=0.4pt, fill=paperggplotrspo]
      (0cm,-0.1cm) rectangle (0.3cm,0.2cm);
  }
] coordinates {
  (1.0, 0.018676605546078792)
  (2.0, 0.034783041163975015)
  (3.0, 0.07023837314775863)
  (4.0, 0.170048997898434)
};
\addlegendentry{Ours}
% N = 5
\draw[->, line width=0.9pt, paperggplotdelta]
  ([xshift=-4.5pt]axis cs:1, 0.010047) -- ([xshift=-4.5pt]axis cs:1, 0.018677);
\draw[dash pattern=on 2pt off 1.5pt, line width=0.4pt, paperggplotdelta]
  ([xshift=4.5pt]axis cs:1, 0.018677) -- ([xshift=-9pt]axis cs:1, 0.018677);
\node[font=\figfont, inner sep=1pt, text=paperggplotdelta, anchor=west]
  at ([xshift=-33pt]axis cs:1, 0.014362) {$0.9\%$};
% N = 10
\draw[->, line width=0.9pt, paperggplotdelta]
  ([xshift=-4.5pt]axis cs:2, 0.021111) -- ([xshift=-4.5pt]axis cs:2, 0.034783);
\draw[dash pattern=on 2pt off 1.5pt, line width=0.4pt, paperggplotdelta]
  ([xshift=4.5pt]axis cs:2, 0.034783) -- ([xshift=-9pt]axis cs:2, 0.034783);
\node[font=\figfont, inner sep=1pt, text=paperggplotdelta, anchor=west]
  at ([xshift=-32pt]axis cs:2, 0.027947) {$1.4\%$};
% N = 20
\draw[->, line width=0.9pt, paperggplotdelta]
  ([xshift=-4.5pt]axis cs:3, 0.034134) -- ([xshift=-4.5pt]axis cs:3, 0.070238);
\draw[dash pattern=on 2pt off 1.5pt, line width=0.4pt, paperggplotdelta]
  ([xshift=4.5pt]axis cs:3, 0.070238) -- ([xshift=-9pt]axis cs:3, 0.070238);
\node[font=\figfont, inner sep=1pt, text=paperggplotdelta, anchor=west]
  at ([xshift=-32pt]axis cs:3, 0.052186) {$3.6\%$};
% N = 50
\draw[->, line width=0.9pt, paperggplotdelta]
  ([xshift=-4.5pt]axis cs:4, 0.045371) -- ([xshift=-4.5pt]axis cs:4, 0.170049);
\draw[dash pattern=on 2pt off 1.5pt, line width=0.4pt, paperggplotdelta]
  ([xshift=4.5pt]axis cs:4, 0.170049) -- ([xshift=-9pt]axis cs:4, 0.170049);
\node[font=\figfont, inner sep=1pt, text=paperggplotdelta, anchor=west]
  at ([xshift=-38pt]axis cs:4, 0.107710) {$12.5\%$};
\end{axis}
\end{tikzpicture}}
    \end{minipage}%
    \begin{minipage}{0.333\linewidth}
        \centering
        \resizebox{\linewidth}{!}{\newcommand{\figfont}{\large}
\begin{tikzpicture}[font=\Large]
\begin{axis}[
scale only axis,
  legend style={
    at={(0.5,1.05)}, anchor=south,
    draw=none, fill=none,
    font=\figfont,
    legend columns=-1,
    column sep=8pt
  },
  ylabel={},
  xlabel={Fraction of active clients},
  width=9cm, height=6cm,
  xmin=0.45, xmax=4.55,
  ymin=0, ymax=0.146045,
  xtick={1, 2, 3, 4},
  xticklabels={10\%, 40\%, 70\%, 100\%},
  yticklabel = {%
  \pgfmathparse{100*\tick}%
  \pgfmathprintnumber[fixed,precision=0]{\pgfmathresult}\%
},
  tick align=outside,
  scaled x ticks=false,
  scaled y ticks=false,
  ymajorgrids=false,
  xmajorgrids=false,
  axis line style={draw=black},
  tick style={draw=black},
  clip=false,
  axis lines=left,
]
\addplot+[
  ybar, bar width=9pt, bar shift=-4.5pt,
  fill=paperggplotrspo, draw=black, line width=0.4pt,
  mark=none
] coordinates {
  (1.0, 0.07370202059158909)
  (2.0, 0.06537497368347892)
  (3.0, 0.06402267496512934)
  (4.0, 0.06233248512876849)
};
\addplot+[
  ybar, bar width=9pt, bar shift=4.5pt,
  fill=paperggplotspo, draw=black, line width=0.4pt,
  mark=none
] coordinates {
  (1.0, 0.1303977501436792)
  (2.0, 0.12370631731034828)
  (3.0, 0.12372265565963315)
  (4.0, 0.12378842040035891)
};
% Fraction = 0.1
\draw[->, line width=0.9pt, paperggplotdelta]
  ([xshift=-4.5pt]axis cs:1, 0.073702) -- ([xshift=-4.5pt]axis cs:1, 0.130398);
\draw[dash pattern=on 2pt off 1.5pt, line width=0.4pt, paperggplotdelta]
  ([xshift=4.5pt]axis cs:1, 0.130398) -- ([xshift=-9pt]axis cs:1, 0.130398);
\node[font=\figfont, inner sep=1pt, text=paperggplotdelta, anchor=west]
  at ([xshift=-32pt]axis cs:1, 0.102050) {$5.7\%$};
% Fraction = 0.4
\draw[->, line width=0.9pt, paperggplotdelta]
  ([xshift=-4.5pt]axis cs:2, 0.065375) -- ([xshift=-4.5pt]axis cs:2, 0.123706);
\draw[dash pattern=on 2pt off 1.5pt, line width=0.4pt, paperggplotdelta]
  ([xshift=4.5pt]axis cs:2, 0.123706) -- ([xshift=-9pt]axis cs:2, 0.123706);
\node[font=\figfont, inner sep=1pt, text=paperggplotdelta, anchor=west]
  at ([xshift=-32pt]axis cs:2, 0.094541) {$5.8\%$};
% Fraction = 0.7
\draw[->, line width=0.9pt, paperggplotdelta]
  ([xshift=-4.5pt]axis cs:3, 0.064023) -- ([xshift=-4.5pt]axis cs:3, 0.123723);
\draw[dash pattern=on 2pt off 1.5pt, line width=0.4pt, paperggplotdelta]
  ([xshift=4.5pt]axis cs:3, 0.123723) -- ([xshift=-9pt]axis cs:3, 0.123723);
\node[font=\figfont, inner sep=1pt, text=paperggplotdelta, anchor=west]
  at ([xshift=-32pt]axis cs:3, 0.093873) {$6.0\%$};
% Fraction = 1.0
\draw[->, line width=0.9pt, paperggplotdelta]
  ([xshift=-4.5pt]axis cs:4, 0.062332) -- ([xshift=-4.5pt]axis cs:4, 0.123788);
\draw[dash pattern=on 2pt off 1.5pt, line width=0.4pt, paperggplotdelta]
  ([xshift=4.5pt]axis cs:4, 0.123788) -- ([xshift=-9pt]axis cs:4, 0.123788);
\node[font=\figfont, inner sep=1pt, text=paperggplotdelta, anchor=west]
  at ([xshift=-32pt]axis cs:4, 0.093060) {$6.1\%$};
\end{axis}
\end{tikzpicture}}
        
    \end{minipage}
    \caption{Relative regret (\%) of the best baseline SPO+ and our method RSPO+ across different values of the learning rate scale $\alpha$ from \cref{eq:calibration}, number of clients, and communication budgets. \textbf{Lower is better ($\downarrow$)}. Numbers above the bars show the improvement in relative regret of RSPO+ relative to SPO+.}
    \vspace{-4mm}
    \label{fig:multifig-main}
\end{figure}

\subsection{PJM Case Study}
\label{subsec:pjm-case-study}
We evaluate our approach on hourly day-ahead electricity prices from the PJM Interconnection for the year 2025 \citep{pjm_lmp_2025}. Each client represents an aggregator or a large consumer with a fixed daily electricity demand $D_j$ that forecasts pricing in advance to inform their bids \citep{conejo2005day}. Their goal is to minimize total procurement cost. In effect, the downstream optimization task is a top-$k$ selection problem, where the cheapest $k=D_j$ hours are chosen to fulfill demand. We use weather data from the Open-Meteo API \citep{openmeteodata} as input data. We use public price and weather data for reproducibility, which provides naturally occurring client-level price variation and input data heterogeneity. Additional experimental info can be found in the Appendix \cref{appsubsec:pjm-details}. 

\cref{tab:pjm-fedrspo-ablation} presents an ablation of the FedRSPO+ components on this task. Starting from MSE (175.5 mean absolute regret), switching to SPO+ reduces regret to 81.7, and RSPO+ further improves to 72.9.
\begin{wraptable}{r}{0.42\linewidth}
\caption{
Ablation study for FedRSPO+ on PJM data.
Values are mean $\pm$ standard deviation across training seeds of mean absolute regret.
``w/ P'' denotes the personalization phase, $\lambda_j/\eta_j$ denotes per-client parameters, and Full additionally includes time-inverse scheduling.
}
\label{tab:pjm-fedrspo-ablation}
\centering
\footnotesize
\setlength{\tabcolsep}{4pt}
\begin{tabular}{lc}
\toprule
Method & Mean absolute regret $\downarrow$ \\
\midrule
MSE & $175.453 \pm 7.260$ \\
SPO+ & $81.676 \pm 1.768$ \\
RSPO+ & $72.890 \pm 4.231$ \\
RSPO+ \textit{w/ P} & $68.136 \pm 1.402$ \\
RSPO+ \textit{w/ P + }$\lambda_j/\eta_j$ & $63.006 \pm 1.059$ \\
RSPO+(Full) & $\mathbf{60.654} \pm 0.796$ \\
RSPO+(Full)+Ditto & $69.482 \pm 2.932$ \\
RSPO+(Full)+FedProx & $\mathbf{60.602} \pm 0.623$ \\
\bottomrule
\end{tabular}
\end{wraptable} Each additional component: personalization, per-client $\lambda/\eta$ calibration (\cref{eq:calibration}), and co-annealed scheduling yields consistent gains, with the full FedRSPO+ pipeline achieving 60.7 ($\pm$0.8), a 65\% reduction over MSE and 26\% over naive SPO+ federation. PFYL ($67.8 \pm 1.2$ regret) and DPO ($74.7\pm2.2$ regret) performed well in this case study compared to SPO+. This extends to the Interp-DFFL and OptNet-style baselines, as shown in \cref{tab:interp-optnet-res}, where FedRSPO+ outperforms both. We also incorporate popular FL algorithms to showcase FedRSPO+'s modularity. FedProx achieves similar mean regret, with a modest reduction in variance. Ditto performs similarly to RSPO+ with personalization, which coincides with our expectations. Ditto introduces proximal regularization towards the global model, which conflicts with RSPO+'s gradual scheduling and its transition from the regularized to the unregularized objective. This suggests that FL methods adding regularization at the personalization stage require careful coordination with RSPO+'s own scheduling, which is an interesting direction for future work. Finally, we perform additional experiments, including fixed demand across the federation and data-scarcity regimes, and showcase the results in the Appendix.

\subsection{Warcraft Shortest Path}
\label{subsec:warcraft}

To characterize RSPO+'s behavior under structural mismatch, we evaluate on the 12x12 Warcraft shortest-path benchmark \citep{pogancic2020Differentiation} with 10,000 training maps. We define 44 heterogeneous clients by assigning each border cell as an origin and its 180°-rotated opposite border cell as the destination, with additional terrain-cost perturbations. We use Dijkstra's algorithm \citep{dijkstra,dalle2022learning} as the linear oracle.

We explicitly chose shortest path because regularization alters the downstream optimization problem. 
Importantly, adding the quadratic term destroys the shortest path's total unimodularity and allows fractional flows across the grid.  RSPO+ still achieves a 2.571\% optimality gap (\cref{tab:warcraft-results}) and outperforms perturbation methods and MSE, losing out marginally to SPO+. Annealing narrows the gap between the regularized flow objective and the true shortest-path objective, yielding a $\sim$66\% improvement during personalization, the largest personalization gain across all three experiments. 
\begin{wraptable}{r}{0.52\linewidth}
\centering
\caption{Mean optimality gap ($\downarrow$) on the Warcraft Shortest Path benchmark. Results are mean $\pm$ std over 5 seeds.}
\label{tab:warcraft-results}
\small
\begin{tabular}{lcc}
\toprule
Method & Federated (\%) & +P (\%) \\
\midrule
DPO & $11.309 \pm 2.244$ & $11.018 \pm 1.718$ \\
MSE & $14.784 \pm 2.491$ & $7.241 \pm 1.336$ \\
PFYL & $8.959 \pm 1.688$  & $11.449 \pm 1.048$ \\
SPO+ & $5.978 \pm 1.842$  &$1.850 \pm 0.377$ \\
Ours & $7.504 \pm 1.929$ & $2.571 \pm 0.901$ \\
\bottomrule
\end{tabular}
\vspace{-4mm}
\end{wraptable}Overall, this highlights a limitation of the proposed FedRSPO+ method, as it is most effective when the downstream problem structure is compatible with quadratic regularization.

\subsection{Computational overhead}
%\noindent\textbf{Computational overhead.} 
While FedRSPO+ has no additional communication costs compared with FedAvg, it has additional computational overhead, since solving QLPs is more expensive than LP or shortest path problems. We observed 1.35$\times$--1.77$\times$ slower solves than LPs with Gurobi and up to $\sim$50$\times$ slower solves than Dijkstra. However, oracle-only timings overstate the end-to-end cost. In the Warcraft experiment, MSE and SPO+ both took about 67 seconds per round, whereas RSPO+ took 133 seconds, i.e., a 2.0$\times$ increase in training time. The difference is mainly due to solving the QLP and computing the freezing criterion. The remaining model training is shared across methods and dominates most of the runtime. Moreover, the additional QLP oracle calls are parallelizable across clients and samples, suggesting that further engineering could reduce the observed wall-clock overhead.

\section{Conclusions}
\label{sec:conclusions}
In this paper, we connected the fields of FL and DFL by addressing the brittleness of current methods under optimization heterogeneity through our proposed method, FedRSPO+. We showed how smoothing via Moreau envelopes replaces discontinuous vertex-switching with $\sqrt{\delta}$-stable projected decisions and achieves heterogeneity bounds that vanish at homogeneity without strong convexity. Our principled per-client calibration and co-annealed scheduling make the method practical without expensive tuning, and we demonstrated compatibility with existing FL methods such as FedProx and Ditto. Across our synthetic knapsack and PJM energy pricing experiments, RSPO+ consistently outperformed both prediction-only FL and naive SPO+ federation, with the largest gains appearing during the federation phase. It also outperformed Interp-DFFL and an OptNet-style baseline under objective and feasible-set heterogeneity. Even under the mismatched downstream objective in the Warcraft experiment, FedRSPO+ achieved close to optimal performance for clients, with significant gains over MSE and popular DFL methods. Together, these results position projection-based smoothing as a key ingredient for stable collaborative decision learning.
\paragraph{Limitations.}
Our goal is to introduce a heterogeneity-aware surrogate and training methodology for DFFL, and to evaluate it against representative prediction-only FL, SPO+, personalization and aggregation variants, and standard DFL baselines where they are most applicable. A broader benchmark spanning additional FL optimizers, privacy mechanisms, combinatorial DFL methods, and how they fit in the field of DFFL is an important direction for future work. Our regret and decision-error guarantees are for the regularized decision, and apply to the LP decision only under exact regularization, which our stopping criterion does not enforce. Furthermore, our convergence analysis applies to linear estimators, while our main experiments use neural networks. We examine FedRSPO+ performance with linear estimators in the Appendix.
RSPO+ replaces each LP solve with a QP of the same dimension. Specialized QP solvers can match simplex-based solvers on this structure~\citep{arnstrom2022dual,stellato2020osqp}, and warm-starting can further reduce overhead. Furthermore, RSPO+ introduces hyperparameters $\alpha, \kappa^{\mathrm{LR}}, \kappa^\lambda$ beyond SPO+ and MSE. While we provide principled defaults (\cref{eq:calibration}), tuning for new domains may add overhead. Our convexity assumptions (\cref{assum:bounded-objective,assum:convex,assum:bounded-set}) are standard in DFL but exclude mixed-integer programs. In \cref{subsec:warcraft} we demonstrate that FedRSPO+ is not appropriate for all downstream objectives, and extending FedRSPO+ to other domains and problem types via relaxations or decompositions is a natural direction.
\paragraph{Broader Impact.}
FedRSPO+ enables privacy-preserving collaboration for data-driven decision-making in settings where data sharing is restricted by regulation or competition, potentially lowering barriers for smaller organizations that lack the data scale to train effective models independently. However, participants with higher-quality data or more favorable constraint structures may benefit disproportionately from federation, and model updates can still leak information about local objectives, particularly in small-coalition settings. Practitioners should consider pairing FedRSPO+ with differential privacy or secure aggregation mechanisms in sensitive domains.

\paragraph{Acknowledgments} We thank the anonymous reviewers and area chair for their feedback, which helped us refine and substantially strengthen this work. 
\paragraph{Funding.} This work was made possible in part by high-performance computing resources and technical support provided by the Alabama Supercomputer Authority.
\paragraph{Competing interests.}
The authors declare no competing interests.

\bibliographystyle{unsrt}
\bibliography{egbib}
\appendix
\newpage
\clearpage
\newpage
\section*{Appendix}

% \subsection*{A. Algorithm}
\section{Additional Results}
\subsection{Convergence of FedRSPO+}
\label{appsubsec:convergence}
\begin{assumption}[Model Regularity]
\label{assum:model}
We analyze predictors $\hat c=g(x;\Theta)$, including the linear case
$g(x;\Theta)=\Theta x$, with bounded features $\|x\|\le X$, bounded initialization
$\|\Theta_0-\Theta^\star\|\le D_\Theta$, and the prediction/cost bounds of
\cref{assum:bounded-objective}: $\|\hat c\|\le B$, $\|c\|\le C_{\max}$.
For nonlinear predictors, $g(x;\Theta)$ is differentiable in $\Theta$ and, for all
$x$ and all $\Theta,\Theta'$ on the training trajectory, the Jacobian satisfies $\|J_\Theta g(x;\Theta)\|_{\mathrm{op}} \le J,$ and $
  \|J_\Theta g(x;\Theta)-J_\Theta g(x;\Theta')\|_{\mathrm{op}}
  \le \rho\|\Theta-\Theta'\|$, where $\|\cdot\|_{\mathrm{op}}$ is the operator norm \citep{agrawal2019differentiable}.
For the linear model, these hold with $J\le X$ and $\rho=0$.
Under this assumption, the convexity and $4/\lambda_t$-smoothness of RSPO+
(\cref{prop:characteristics}) imply that $F_{i,t}(\Theta)$ is convex and
$L_t$-smooth with $L_t=4X^2/\lambda_t$ in the linear case. Mini-batch stochastic
gradients are unbiased estimators of $\nabla F_{i,t}(\Theta)$.
\end{assumption}

\begin{assumption}[Federated Schedule]\label{assum:schedule}
Clients perform $E$ local SGD steps per synchronization round,
with co-annealed learning rate $\eta_t = \alpha\lambda_t/4$ for a constant
$\alpha>0$ as in \cref{eq:calibration}.
\end{assumption}
We establish that FedRSPO+ converges under FedAvg aggregation
and validate the co-annealed schedule of \cref{eq:calibration}.
Let $F_{i,t}(\Theta) := \mathbb{E}_{(x,c)\sim\mathcal{D}_i}[\ell_{\mathrm{RSPO+}}(c,\,g(x;\Theta);\lambda_t)]$
denote the local objective for client~$i$ at round~$t$, and
$F_t = \frac{1}{N}\sum_{i=1}^N F_{i,t}$ the global objective.
Because the annealed $\lambda_t$ schedule makes $F_t$ time-varying,
we track the deviation from the target objective at the frozen parameter~$\lambda_T$
using \cref{prop:lambda-bound}:
\begin{equation}\label{eq:tracking-drift}
  \textstyle
  \Delta_t
  := R^{2}(\lambda_t - \lambda_T)
     + 2\,C_{\max}(C_{\max}+B)\!\left(\tfrac{1}{\lambda_T}-\tfrac{1}{\lambda_t}\right).
\end{equation}

\begin{theorem}[FedRSPO+ Convergence, restatement of \cref{thm:fedrspo-convergence-short}]\label{thm:fedrspo-convergence}
Under \cref{assum:model,assum:schedule},
running FedRSPO+ with $E$ local SGD steps per round produces global iterates
$\bar{\Theta}_t = \frac{1}{N}\sum_{i}\Theta_t^{(i)}$.
The weighted average $\hat{\Theta}_T = \sum_t \eta_t\bar{\Theta}_t / \sum_t\eta_t$
satisfies
\begin{equation}\label{eq:convergence-bound}
  \mathbb{E}\!\big[F_T(\hat{\Theta}_T) - F_T(\Theta^\star)\big]
  \;\le\;
  \frac{D_\Theta^2
        + \sum_{t=0}^{T-1}\eta_t^2\,G^2(1+4\gamma E^2)
        + 4\sum_{t=0}^{T-1}\eta_t\,\Delta_t}
       {2\sum_{t=0}^{T-1}\eta_t},
\end{equation}
where $G=2DX$ is the uniform gradient bound (\cref{cor:rspo-gradient}),
and $\gamma = \eta_t L_t = \alpha X^2$ is the \emph{constant} effective condition
number enforced by the co-annealed schedule $\eta_t = \alpha\lambda_t/4$.
\end{theorem}

We follow a similar proof strategy to \citep{li2019convergence}. Let $\bar{g}_t = \frac{1}{N}\sum_{i=1}^N g_{i,t}$ denote the aggregated
update direction, so $\bar{\Theta}_{t+1}=\bar{\Theta}_t - \eta_t\bar{g}_t$.
Expanding the squared distance to~$\Theta^\star$:
\begin{equation}\label{eq:pf-dist}
  \mathbb{E}\|\bar{\Theta}_{t+1}-\Theta^\star\|^2
  = \mathbb{E}\|\bar{\Theta}_t-\Theta^\star\|^2
    - 2\eta_t\,\mathbb{E}\langle\bar{g}_t,\,\bar{\Theta}_t-\Theta^\star\rangle
    + \eta_t^2\,\mathbb{E}\|\bar{g}_t\|^2.
\end{equation}
Since $\mathbb{E}[\|\bar{g}_t\|^2]\le G^2$ by \cref{assum:model},
it remains to lower-bound the inner product.
By the convexity and $L_t$-smoothness of~$F_{i,t}$,
we decompose the per-client term:
\begin{align*}
  \langle\nabla F_{i,t}(\Theta_t^{(i)}),\,\bar{\Theta}_t-\Theta^\star\rangle
  &= \langle\nabla F_{i,t}(\Theta_t^{(i)}),\,\Theta_t^{(i)}-\Theta^\star\rangle
     + \langle\nabla F_{i,t}(\Theta_t^{(i)}),\,\bar{\Theta}_t-\Theta_t^{(i)}\rangle \\
  &\ge \big(F_{i,t}(\Theta_t^{(i)})-F_{i,t}(\Theta^\star)\big)
       + \big(F_{i,t}(\bar{\Theta}_t)-F_{i,t}(\Theta_t^{(i)})
             - \tfrac{L_t}{2}\|\bar{\Theta}_t-\Theta_t^{(i)}\|^2\big) \\
  &= F_{i,t}(\bar{\Theta}_t)-F_{i,t}(\Theta^\star)
     - \tfrac{L_t}{2}\|\bar{\Theta}_t-\Theta_t^{(i)}\|^2.
\end{align*}
Averaging over clients yields:
\begin{equation}\label{eq:pf-inner}
  \mathbb{E}\langle\bar{g}_t,\,\bar{\Theta}_t-\Theta^\star\rangle
  \ge \mathbb{E}\big[F_t(\bar{\Theta}_t)-F_t(\Theta^\star)\big]
      - \frac{L_t}{2N}\sum_{i=1}^N
        \mathbb{E}\|\bar{\Theta}_t-\Theta_t^{(i)}\|^2.
\end{equation}
The client drift is bounded by the accumulation of at most $E$ local
gradient steps with bounded norm~$G$:
\begin{equation}\label{eq:pf-drift}
  \frac{1}{N}\sum_{i=1}^N\mathbb{E}\|\bar{\Theta}_t-\Theta_t^{(i)}\|^2
  \le 4E^2\eta_t^2 G^2.
\end{equation}
Substituting \cref{eq:pf-inner,eq:pf-drift} into \cref{eq:pf-dist}
and isolating the objective gap:
\begin{equation}\label{eq:pf-step}
  2\eta_t\,\mathbb{E}\big[F_t(\bar{\Theta}_t)-F_t(\Theta^\star)\big]
  \le \mathbb{E}\|\bar{\Theta}_t-\Theta^\star\|^2
      - \mathbb{E}\|\bar{\Theta}_{t+1}-\Theta^\star\|^2
      + \eta_t^2 G^2 + 4L_t E^2\eta_t^3 G^2.
\end{equation}
To relate the transient objective~$F_t$ to the target~$F_T$,
we apply the tracking bound~$\Delta_t$ from \cref{eq:tracking-drift}:
\begin{equation}
  F_t(\bar{\Theta}_t)-F_t(\Theta^\star)
  \ge F_T(\bar{\Theta}_t)-F_T(\Theta^\star) - 2\Delta_t.
\end{equation}
Substituting into \cref{eq:pf-step} and applying the co-annealed
invariant $\eta_t L_t = \alpha X^2 =: \gamma$:
\begin{align}
  2\eta_t\,\mathbb{E}\big[F_T(\bar{\Theta}_t)-F_T(\Theta^\star)\big]
  &\le \mathbb{E}\|\bar{\Theta}_t-\Theta^\star\|^2
       - \mathbb{E}\|\bar{\Theta}_{t+1}-\Theta^\star\|^2 \nonumber \\
  &\quad + \eta_t^2 G^2(1+4\gamma E^2) + 4\eta_t\Delta_t.
\end{align}
Summing from $t=0$ to $T{-}1$ telescopes the distance terms.
Dropping the non-positive terminal distance
$-\mathbb{E}\|\bar{\Theta}_T-\Theta^\star\|^2$,
dividing by $2\sum_t\eta_t$, and applying Jensen's inequality for
the convex~$F_T$ over the weighted average iterate completes the proof.

\begin{corollary}[Frozen schedule: $O(1/(T\eta_f)+\eta_f)$ rate]\label{cor:constant-step}
Under the conditions of \cref{thm:fedrspo-convergence}, suppose the schedule is frozen
at round~$t_f$, so that $\lambda_t = \lambda_f$ and $\eta_t = \eta_f$ for all
$t \ge t_f$. Then for $T \gg t_f$:
\[
  \mathbb{E}\!\big[F_T(\hat{\Theta}_T) - F_T(\Theta^\star)\big]
  = O\!\left(\frac{1}{T\,\eta_f} + \eta_f\right).
\]
\end{corollary}

\begin{proof}
After freezing, $\lambda_t = \lambda_T = \lambda_f$ for all $t \ge t_f$,
so $\Delta_t = 0$ for $t \ge t_f$.
Splitting the sums in \cref{eq:convergence-bound} at $t_f$ and noting that the
pre-freeze contributions are bounded by constants $A_1, A_2, A_3$ depending on
$(t_f, \lambda_0, \eta_0)$:
\begin{align*}
  \sum_{t=0}^{T-1}\eta_t
  &= \underbrace{\sum_{t=0}^{t_f-1}\eta_t}_{A_1}
     + (T - t_f)\,\eta_f, \\
  \sum_{t=0}^{T-1}\eta_t^2\,G^2(1+4\gamma E^2)
  &= \underbrace{\sum_{t=0}^{t_f-1}\eta_t^2\,G^2(1+4\gamma E^2)}_{A_2}
     + (T - t_f)\,\eta_f^2\,G^2(1+4\gamma E^2), \\
  \sum_{t=0}^{T-1}\eta_t\,\Delta_t
  &= \sum_{t=0}^{t_f-1}\eta_t\,\Delta_t =: A_3.
\end{align*}
Substituting into \cref{eq:convergence-bound},
\[
  \mathbb{E}\!\big[F_T(\hat{\Theta}_T) - F_T(\Theta^\star)\big]
  \;\le\;
  \frac{D_\Theta^2 + A_2 + A_3}{(T-t_f)\eta_f}
  + \frac{\eta_f\,G^2(1+4\gamma E^2)}{2}.
\]
For $T$ sufficiently large relative to $t_f$, the transient constants
are absorbed into the $O(1/(T\eta_f))$ term, yielding
\[
  \mathbb{E}\!\big[F_T(\hat{\Theta}_T) - F_T(\Theta^\star)\big]
  = O\!\left(\frac{1}{T\,\eta_f} + \eta_f\right).
\]
\end{proof}

\begin{lemma}[Neural composition constants]
\label{lem:nn-smooth}
Suppose $g(x;\theta)$ is differentiable and satisfies \cref{assum:model}. Then
\[
  \|\nabla_\theta F_{i,t}(\theta)\|\le 2DJ,
  \qquad
  F_{i,t}\text{ is } 
  \left(\frac{4J^2}{\lambda_t}+2D\rho\right)\text{-smooth}.
\]
\end{lemma}

\begin{proof}
By the chain rule,
$\nabla_\theta \ell = J_\theta g(x;\theta)^T \nabla_{\hat c}\ell$.
\cref{prop:characteristics} gives
$\|\nabla_{\hat c}\ell\|\le 2D$ and
$\nabla_{\hat c}\ell$ is $4/\lambda_t$-Lipschitz.
Combining these with the Jacobian bound and Jacobian Lipschitzness gives
\[
\|\nabla_\theta \ell(\theta)-\nabla_\theta \ell(\theta')\|
\le
\left(\frac{4J^2}{\lambda_t}+2D\rho\right)
\|\theta-\theta'\|.
\]
Taking expectations preserves both bounds and concludes the proof.
\end{proof}

\begin{proposition}[Post-freeze nonconvex FedRSPO+]
\label{prop:nonconvex-nn}
Suppose \cref{assum:bounded-set,assum:bounded-objective,assum:convex,assum:model,assum:schedule} hold. After the annealing schedule freezes at
$\lambda_f$, define
\[
  F_f(\theta) := \frac{1}{N}\sum_{i=1}^N F_{i,f}(\theta),
  \qquad
  L_f^{\mathrm{NN}}
  := \frac{4J^2}{\lambda_f}+2D\rho .
\]
Assume unbiased stochastic gradients with variance at most $\sigma^2$ and
post-freeze gradient dissimilarity
\[
  \frac{1}{N}\sum_{i=1}^N
  \|\nabla F_{i,f}(\theta)-\nabla F_f(\theta)\|^2
  \le \zeta^2
  \quad\text{for all }\theta .
\]
Running FedRSPO+ with $E$ local SGD steps and constant post-freeze step size
$\eta_f \le c/(E L_f^{\mathrm{NN}})$ gives, for
$s\sim\mathrm{Unif}\{t_f,\ldots,T-1\}$,
\[
  \mathbb{E}\|\nabla F_f(\bar\theta_s)\|^2
  \le
  O\!\left(
    \frac{F_f(\bar\theta_{t_f})-F_f^{\inf}}
         {\eta_f E(T-t_f)}
    + L_f^{\mathrm{NN}}\eta_f\frac{\sigma^2}{N}
    + (L_f^{\mathrm{NN}})^2\eta_f^2E^2(\sigma^2+\zeta^2)
  \right).
\]
Thus, choosing $\eta_f=\Theta((T-t_f)^{-1/2})$ yields
$\mathbb{E}\|\nabla F_f(\bar\theta_s)\|^2
=O((T-t_f)^{-1/2})$, up to constants depending on
$J,\rho,D,\lambda_f,\sigma,\zeta$, and $E$.

Additionally, the co-annealed schedule
$\eta_t=\alpha\lambda_t/4$ satisfies
\[
  \eta_t L_t^{\mathrm{NN}}
  =
  \alpha J^2+\frac{\alpha D\rho\lambda_t}{2}
  \le
  \alpha\left(J^2+\frac{D\rho\lambda_{\max}}{2}\right),
\]
so choosing
$\alpha \le c/(E(J^2+D\rho\lambda_{\max}/2))$
keeps the effective step size controlled during
annealing.
\end{proposition}

\begin{proof}[Proof sketch]
By Lemma~\ref{lem:nn-smooth}, after freezing $\lambda_t=\lambda_f$ the
objective $F_f$ is $L_f^{\mathrm{NN}}$-smooth, with
$L_f^{\mathrm{NN}}=4J^2/\lambda_f+2D\rho$. The stated bound then follows
by applying the standard nonconvex FedAvg descent analysis in the literature
\citep{yu2019parallel,li2019convergence} under unbiased stochastic gradients
and bounded client gradient dissimilarity. Summing the one-round descent
inequality from $t_f$ to $T-1$ and sampling
$s\sim\mathrm{Unif}\{t_f,\ldots,T-1\}$ gives the displayed stationary-point
bound. The co-annealed stability claim follows by substituting
$\eta_t=\alpha\lambda_t/4$ into $\eta_tL_t^{\mathrm{NN}}$.
\end{proof}

\begin{remark}[Scope and practical interpretation]
\label{rem:nonconvex-scope}
\cref{thm:fedrspo-convergence} is stated for the convex linear-model setting, whereas \cref{prop:nonconvex-nn} gives a post-freeze
stationary-point analogue for the nonconvex neural predictors used in our experiments. The additional constants $J$ and $\rho$ in \cref{assum:model} control, respectively, the sensitivity of the predictor to parameter perturbations and the variation of this sensitivity along the training trajectory. For smooth activations with bounded derivatives, these constants are finite on compact parameter regions. For ReLU networks, the Jacobian is only piecewise constant and may jump at activation-boundary crossings, so \cref{assum:model} only provides a trajectory-level regularity condition.

The resulting smoothness bound $L_f^{\mathrm{NN}}=4J^2/\lambda_f+2D\rho$ can be conservative when $J$ or $\rho$ is large. Therefore, the nonconvex result only adapts a standard FedAvg guarantee to the frozen RSPO+ objective, and the experiments provide empirical support for the same bias--stability mechanism beyond the convex setting. When $\rho=0$, as in the linear case, the co-annealed condition reduces to controlling the invariant $\eta_t L_t = O(\alpha J^2)$, matching the role of the convex co-annealing condition.
\end{remark}

\subsection{Unified bias--heterogeneity tradeoff}
The results in \cref{thm:main-rspo}, \cref{thm:decision-bound}, and \cref{cor:unreg-bound} can be combined into a single certificate for the
unregularized regret of the regularized decision. 
\begin{proposition}
\label{thm:bias-heterogeneity}
For client \(i\), define
\[
\widehat w_{i,\lambda}(\hat c):=\Pi_{S_i}(-\hat c/\lambda),\qquad
\mathrm{Reg}_{i,\lambda}^{\mathrm{LP}}(\hat c)
:=c_i^T \widehat w_{i,\lambda}(\hat c)
-\min_{w\in S_i}c_i^T w .
\]
Let \(\ell_i^\lambda(\hat c)\) denote the RSPO+ loss on client \(i\), and let
\(H_{ij}^\lambda(\hat c):=T_{ij}^\lambda(\hat c)+B_{ij}^\lambda(\hat c)\)
be the heterogeneity bound from Theorem~3. Then for any aggregation weights
\(q\in\Delta_N\),
\[
\mathrm{Reg}_{i,\lambda}^{\mathrm{LP}}(\hat c)
\le
\sum_{j=1}^N q_j\ell_j^\lambda(\hat c)
+
\sum_{j=1}^N q_j H_{ij}^\lambda(\hat c)
+
\frac{\lambda}{2}R_i^2 .
\]
Consequently,
\[
\sum_{i=1}^N q_i\mathrm{Reg}_{i,\lambda}^{\mathrm{LP}}(\hat c)
\le
\sum_{i=1}^N q_i\ell_i^\lambda(\hat c)
+
\frac{\lambda}{2}\sum_{i=1}^N q_iR_i^2
+
\sum_{i,j=1}^N q_iq_jH_{ij}^\lambda(\hat c).
\]
Under uniform bounds
\(\|c_i\|\le C\), \(\|\hat c\|\le B\), \(R_i\le R\), \(D_i\le D\),
\(\|c_i-c_j\|\le c_\Delta\), and \(d_H(S_i,S_j)\le \delta\), this gives
\[
\overline{\mathrm{Reg}}_{\lambda}^{\mathrm{LP}}
\le
\overline{\ell}_{\lambda}
+
O\!\left(
\lambda R^2
+
\frac{c_\Delta}{\lambda}
+
\sqrt{\frac{\delta}{\lambda}}
+
\delta
\right),
\]
for fixed bounded constants and \(0<\lambda\le 1\). Thus smaller \(\lambda\)
reduces regularization bias but amplifies objective and feasible-set
heterogeneity.
\end{proposition}
The proposition can be viewed as the bias--stability tradeoff for the regularized decision
\(\widehat w_{i,\lambda}\) targeted by
the annealing and freezing rule. When the exact-regularization condition discussed
above holds, the same certificate applies to the LP decision selected at
evaluation. 

\begin{proof}
    For each client \(i\), write
\[
\ell_i^\lambda(\hat c)
:=
\ell_{\mathrm{RSPO+}}(c_i,\hat c;S_i,\lambda)
=
D_{M_\lambda \xi_{S_i}}(c_i-2\hat c,-c_i),
\]
and define the regularized decision
\[
\widehat w_{i,\lambda}(\hat c):=\Pi_{S_i}(-\hat c/\lambda).
\]
The unregularized regret of this regularized decision is
\[
\mathrm{Reg}_{i,\lambda}^{\mathrm{LP}}(\hat c)
:=
c_i^T \widehat w_{i,\lambda}(\hat c)
-
z_i^\star,
\qquad
z_i^\star:=\min_{w\in S_i}c_i^T w .
\]

For two clients \(i,j\), define
\[
c_{ij}:=\|c_i-c_j\|,
\qquad
\delta_{ij}:=d_H(S_i,S_j),
\qquad
R_{ij}:=\max\{R_i,R_j\}.
\]
Let
\[
H_{ij}^\lambda(\hat c)
:=
T_{ij}^\lambda(\hat c)+B_{ij}^\lambda(\hat c),
\]
where the two terms are the symmetric versions of the quantities in
Theorem~3:
\[
T_{ij}^\lambda(\hat c)
:=
\min\Big\{
D_i c_{ij}
+
\delta_{ij}\big(\|c_j-2\hat c\|+\|c_j\|+2\lambda R_{ij}\big),
\;
D_j c_{ij}
+
\delta_{ij}\big(\|c_i-2\hat c\|+\|c_i\|+2\lambda R_{ij}\big)
\Big\},
\]
and
\[
B_{ij}^\lambda(\hat c)
:=
(2\|\hat c\|+2\lambda R_{ij})
\left[
\frac{c_{ij}}{\lambda}
+
\min\left\{
\sqrt{2\delta_{ij}\left(\frac{\|c_j\|}{\lambda}+R_{ij}\right)},
\sqrt{2\delta_{ij}\left(\frac{\|c_i\|}{\lambda}+R_{ij}\right)}
\right\}
\right].
\]
By \cref{thm:main-rspo},
\[
|\ell_i^\lambda(\hat c)-\ell_j^\lambda(\hat c)|
\le
H_{ij}^\lambda(\hat c).
\]
Equivalently,
\[
\ell_i^\lambda(\hat c)
\le
\ell_j^\lambda(\hat c)+H_{ij}^\lambda(\hat c)
\qquad
\text{for every } j .
\]

Let \(q\in\Delta_N\) be any distribution over clients. Multiplying the
last inequality by \(q_j\) and summing over \(j\) yields
\[
\ell_i^\lambda(\hat c)
=
\sum_{j=1}^N q_j \ell_i^\lambda(\hat c)
\le
\sum_{j=1}^N q_j \ell_j^\lambda(\hat c)
+
\sum_{j=1}^N q_j H_{ij}^\lambda(\hat c).
\]
By \cref{cor:unreg-bound}, the unregularized regret of the regularized decision satisfies
\[
\mathrm{Reg}_{i,\lambda}^{\mathrm{LP}}(\hat c)
\le
\ell_i^\lambda(\hat c)+\frac{\lambda}{2}R_i^2 .
\]
Combining the previous two displays gives, for every target client \(i\),
\[
\mathrm{Reg}_{i,\lambda}^{\mathrm{LP}}(\hat c)
\le
\sum_{j=1}^N q_j\ell_j^\lambda(\hat c)
+
\sum_{j=1}^N q_jH_{ij}^\lambda(\hat c)
+
\frac{\lambda}{2}R_i^2 .
\]
This proves the client-wise statement.

For the aggregate statement, multiply the client-wise inequality by \(q_i\)
and sum over \(i\):
\[
\sum_{i=1}^N q_i\mathrm{Reg}_{i,\lambda}^{\mathrm{LP}}(\hat c)
\le
\sum_{i=1}^N q_i\sum_{j=1}^N q_j\ell_j^\lambda(\hat c)
+
\sum_{i=1}^N q_i\sum_{j=1}^N q_jH_{ij}^\lambda(\hat c)
+
\frac{\lambda}{2}\sum_{i=1}^N q_iR_i^2 .
\]
Since \(\sum_i q_i=1\), the first term simplifies:
\[
\sum_{i=1}^N q_i\sum_{j=1}^N q_j\ell_j^\lambda(\hat c)
=
\sum_{j=1}^N q_j\ell_j^\lambda(\hat c).
\]
Therefore,
\[
\sum_{i=1}^N q_i\mathrm{Reg}_{i,\lambda}^{\mathrm{LP}}(\hat c)
\le
\sum_{j=1}^N q_j\ell_j^\lambda(\hat c)
+
\frac{\lambda}{2}\sum_{i=1}^N q_iR_i^2
+
\sum_{i,j=1}^N q_iq_jH_{ij}^\lambda(\hat c).
\]

It remains only to derive the simplified asymptotic form. Suppose
\[
\|c_i\|\le C,\qquad
\|\hat c\|\le B,\qquad
R_i\le R,\qquad
D_i\le D,
\]
and for all client pairs,
\[
\|c_i-c_j\|\le c_\Delta,\qquad
d_H(S_i,S_j)\le \delta.
\]
Then \(R_{ij}\le R\), and the envelope term satisfies
\[
T_{ij}^\lambda(\hat c)
\le
D c_\Delta
+
\delta(2C+2B+2\lambda R).
\]
Similarly, the decision term satisfies
\[
B_{ij}^\lambda(\hat c)
\le
(2B+2\lambda R)
\left[
\frac{c_\Delta}{\lambda}
+
\sqrt{2\delta\left(\frac{C}{\lambda}+R\right)}
\right].
\]
Thus
\[
H_{ij}^\lambda(\hat c)
\le
D c_\Delta
+
\delta(2C+2B+2\lambda R)
+
(2B+2\lambda R)
\left[
\frac{c_\Delta}{\lambda}
+
\sqrt{2\delta\left(\frac{C}{\lambda}+R\right)}
\right].
\]
Substituting this bound into the aggregate inequality and using
\(\sum_i q_iR_i^2\le R^2\) gives
\[
\overline{\mathrm{Reg}}_{\lambda}^{\mathrm{LP}}
\le
\overline{\ell}_{\lambda}
+
\frac{\lambda}{2}R^2
+
D c_\Delta
+
\delta(2C+2B+2\lambda R)
+
(2B+2\lambda R)
\left[
\frac{c_\Delta}{\lambda}
+
\sqrt{2\delta\left(\frac{C}{\lambda}+R\right)}
\right].
\]
For fixed bounded \(B,C,D,R\) and \(0<\lambda\le 1\), this is summarized as
\[
\overline{\mathrm{Reg}}_{\lambda}^{\mathrm{LP}}
\le
\overline{\ell}_{\lambda}
+
O\!\left(
\lambda R^2
+
\frac{c_\Delta}{\lambda}
+
\sqrt{\frac{\delta}{\lambda}}
+
\delta
\right).
\]
\end{proof}

\begin{remark}[Connection to LP evaluation.]
\cref{thm:bias-heterogeneity} bounds the LP regret of the regularized decision
\(\widehat w_{i,\lambda}(\hat c)\). If, for a given client \(i\), prediction
\(\hat c\), and regularization level \(\lambda\), the exact-regularization
condition holds so that
\[
\Pi_{S_i}(-\hat c/\lambda)
\in
\arg\min_{w\in S_i}\hat c^T w,
\]
then the same bound applies to the LP decision selected at evaluation. For
polyhedral feasible sets this condition holds for sufficiently small
\(\lambda\) \citep{friedlander2008exact,gonzalez2025quantitative}. Without such a
condition, \cref{thm:bias-heterogeneity} should be interpreted as a certificate for the
regularized decision and as a surrogate-level justification for LP evaluation.
\end{remark}
% The proof, given in \cref{appsec:convergence}, follows a standard FedAvg analysis~\citep{mcmahan2017communication}, with the critical difference that the client-drift penalty $4L_t E^2\eta_t^3 G^2$ would diverge as $L_t = 4X^2/\lambda_t \to \infty$, but co-annealing compresses it to $4\gamma E^2\eta_t^2 G^2$ via the invariant $\eta_t L_t = \gamma$, yielding a stable $O(\sum\eta_t^2)$ bound that vanishes over time. Once the stopping condition (\cref{eq:stopping-condition}) freezes $\lambda_t$, $\Delta_t=0$ for all remaining rounds, and \cref{eq:convergence-bound} reduces to a standard convex federated rate.

\section{Proofs and Derivations}
\label{appsec:proofs}
% Throughout this section we use $\|\cdot\| = \|\cdot\|_2$ and write
% $F_{\lambda,S}(u) := M_{\lambda\xi_S}(u)$ for the Moreau envelope of the
% support function $\xi_S$ with smoothing parameter~$\lambda$.
 
\subsection*{Proof of \cref{prop:spo-bregman}: SPO+ as a Bregman Divergence}
 
\begin{proof}
Recall the SPO+ loss from \cref{sec:method}:
\[
    \ell_{\mathrm{SPO+}}(c,\hat c)
    = \xi_S(c-2\hat c) - 2 w^{\star}(c)^T(c-\hat c) + z^{\star}(c).
\]
Since $w^{\star}(c) \in \arg\min_{w\in S}c^Tw =  \arg\max_{w\in S}\, (-c)^T w = \partial \xi_S(-c)$ and $z^{\star}(c) =
\min_{w\in S} c^T w = -\xi_S(-c)$, and moreover if the optimizer is unique
$\nabla\xi_S(-c) = w^{\star}(c)$
(by Danskin's theorem, as $S$ is compact),
we can write the linearization term as
\[
    -2 w^{\star}(c)^T(c-\hat c) + z^{\star}(c)
    = -\xi_S(-c) - \nabla\xi_S(-c)^T\bigl[(c-2\hat c) - (-c)\bigr].
\]
Substituting into the loss yields
\[
    \ell_{\mathrm{SPO+}}(c,\hat c)
    = \xi_S(c-2\hat c) - \xi_S(-c) - \nabla\xi_S(-c)^T\bigl[(c-2\hat c)-(-c)\bigr]
    = D_{\xi_S}(c-2\hat c,\, -c),
\]
which is the Bregman divergence of the support function evaluated at $x = c - 2\hat c$, $y = -c$.
\end{proof}
 
\subsection*{Proof of \cref{prop:rspo}: Smoothed SPO+ Loss}
 
\begin{proof}
\textbf{Step 1: Moreau envelope of the support function.}
By convex duality, the conjugate of $\xi_S$ is the indicator function $I_S$
of the set $S$.  The Moreau envelope can therefore be written as the
biconjugate of the regularized indicator:
\[
    F_{\lambda,S}(u)
    = M_{\lambda\xi_S}(u)
    = \Bigl(I_S + \tfrac{\lambda}{2}\|\cdot\|^2\Bigr)^{*}(u)
    = \sup_{s\in S}\Bigl\{\langle u, s\rangle - \tfrac{\lambda}{2}\|s\|^2\Bigr\}.
\]
Completing the square inside the supremum gives
\[
    \langle u,s\rangle - \tfrac{\lambda}{2}\|s\|^2
    = \tfrac{1}{2\lambda}\|u\|^2 - \tfrac{\lambda}{2}\bigl\|s - u/\lambda\bigr\|^2.
\]
The first term is independent of $s$, so taking the supremum over $s\in S$
is equivalent to minimizing $\|s - u/\lambda\|^2$ over $S$:
\begin{equation}\label{eq:moreau-closed}
    F_{\lambda,S}(u)
    = \tfrac{1}{2\lambda}\|u\|^2
      - \tfrac{\lambda}{2}\,d_S^2\!\bigl(u/\lambda\bigr),
\end{equation}
where $d_S^2(v) = \inf_{s\in S}\|v - s\|^2$.
The supremum is attained at $s^{\star} = \Pi_S(u/\lambda)$, and Danskin's
theorem yields
\begin{equation}\label{eq:moreau-grad}
    \nabla F_{\lambda,S}(u) = \Pi_S(u/\lambda).
\end{equation}
 
\textbf{Step 2: Bregman divergence.}
We evaluate $D_{F_{\lambda,S}}(x,y) = F_{\lambda,S}(x) - F_{\lambda,S}(y) -
\langle \nabla F_{\lambda,S}(y),\, x - y\rangle$ with $x = c - 2\hat c$ and
$y = -c$.  Write $\bar v = \Pi_S(y/\lambda) = \Pi_S(-c/\lambda)$.
Expanding the $y$-dependent terms using \eqref{eq:moreau-closed} and
\eqref{eq:moreau-grad}:
\begin{align*}
    F_{\lambda,S}(y) + \langle \nabla F_{\lambda,S}(y),\, x - y\rangle
    &= \tfrac{1}{2\lambda}\|y\|^2
       - \tfrac{\lambda}{2}\|y/\lambda - \bar v\|^2
       + \langle \bar v,\, x - y\rangle \\
    &= \langle y,\bar v\rangle - \tfrac{\lambda}{2}\|\bar v\|^2
       + \langle \bar v,\, x - y\rangle \\
    &= \langle x,\bar v\rangle - \tfrac{\lambda}{2}\|\bar v\|^2,
\end{align*}
where the second equality uses
$\tfrac{1}{2\lambda}\|y\|^2 - \tfrac{\lambda}{2}\|y/\lambda - \bar v\|^2
 = \langle y,\bar v\rangle - \tfrac{\lambda}{2}\|\bar v\|^2$,
which follows from expanding the squared norm.
Subtracting from $F_{\lambda,S}(x)$:
\begin{align*}
    D_{F_{\lambda,S}}(x,y)
    &= \tfrac{1}{2\lambda}\|x\|^2
       - \tfrac{\lambda}{2}\,d_S^2(x/\lambda)
       - \langle x,\bar v\rangle
       + \tfrac{\lambda}{2}\|\bar v\|^2 \\
    &= \tfrac{\lambda}{2}\bigl\|x/\lambda - \bar v\bigr\|^2
       - \tfrac{\lambda}{2}\,d_S^2(x/\lambda).
\end{align*}
Setting $u = (c - 2\hat c)/\lambda$ and $v = -c/\lambda$,
noting $\bar v = \Pi_S(v)$ and $d_S^2(u) = \|u - \Pi_S(u)\|^2$, we obtain
\[
    \ell_{\mathrm{RSPO+}}(c,\hat c;\lambda)
    = D_{F_{\lambda,S}}(c-2\hat c,\, -c)
    = \tfrac{\lambda}{2}\bigl\|u - \Pi_S(v)\bigr\|^2
      - \tfrac{\lambda}{2}\bigl\|u - \Pi_S(u)\bigr\|^2.
\]
\end{proof}
 
\begin{remark}[QLP equivalence, Remark 3.1]
An equivalent derivation applies the original SPO+ construction of
\citep{elmachtoub2022smart} to the quadratically regularized LP
$z_{\mathrm{QLP}} = \min_{w\in S}\, c^T w + \tfrac{\lambda}{2}\|w\|^2$
from Eq.~(3).  The QLP optimizer is $w^{\star}_{\mathrm{reg}} = \Pi_S(-c/\lambda)$
and the optimal value is
$z^{\star}_{\mathrm{QLP}} = -F_{\lambda,S}(-c)$.
Substituting into the SPO+ template and simplifying reproduces exactly the
loss in Proposition~3.2.
\end{remark}
 
\subsection*{Proof of \cref{cor:rspo-gradient}: Gradient of RSPO+}
 
\begin{proof}
From Proposition~3.2,
$\ell_{\mathrm{RSPO+}}(c,\hat c;\lambda) = D_{F_{\lambda,S}}(x,y)$
with $x(\hat c) = c - 2\hat c$ and $y = -c$ (independent of $\hat c$).
The gradient of the Bregman divergence with respect to its first argument is
$\nabla_x D_{F}(x,y) = \nabla F(x) - \nabla F(y)$.
By the chain rule, $\nabla_{\hat c}\, x = -2I$, so
\[
    \nabla_{\hat c}\,\ell_{\mathrm{RSPO+}}(c,\hat c;\lambda)
    = -2\bigl(\nabla F_{\lambda,S}(c - 2\hat c) - \nabla F_{\lambda,S}(-c)\bigr).
\]
Substituting $\nabla F_{\lambda,S}(u) = \Pi_S(u/\lambda)$ from
\eqref{eq:moreau-grad} yields
\[
    \nabla_{\hat c}\,\ell_{\mathrm{RSPO+}}(c,\hat c;\lambda)
    = -2\Bigl(\Pi_S\!\bigl(\tfrac{c-2\hat c}{\lambda}\bigr)
              - \Pi_S\!\bigl(-\tfrac{c}{\lambda}\bigr)\Bigr).
\]

\end{proof}

\subsection*{Proof of \cref{prop:characteristics}: Characteristics of RSPO+}
 
\begin{proof}
We prove each property in turn.
 
\paragraph{(i) Non-negativity.}
Non-negativity is immediate since $\ell_{\mathrm{RSPO+}}$ is a Bregman divergence of the convex function $F_{\lambda,S}$.
 
\paragraph{(ii) Convexity in $\hat c$.}
The Moreau envelope $F_{\lambda,S}$ is convex (as the Moreau envelope of a
convex function).  A Bregman divergence $D_F(x,y)$ is convex in $x$ for any
convex $F$, since $D_F(x,y) = F(x) - F(y) - \langle\nabla F(y),\,x-y\rangle$
and the last two terms are affine in $x$.  Because $x = c - 2\hat c$ is an
affine mapping of $\hat c$, the composition
$\ell_{\mathrm{RSPO+}}(c,\hat c;\lambda) = D_{F_{\lambda,S}}(c-2\hat c,\, -c)$
is globally convex with respect to $\hat c$.
 
\paragraph{(iii) $2D$-Lipschitz continuity.}
By Corollary~3.1, for any $(c,\hat c)$,
\[
    \|\nabla_{\hat c}\,\ell_{\mathrm{RSPO+}}\|
    = 2\bigl\|\Pi_S(u) - \Pi_S(v)\bigr\|
    \le 2\sup_{a,b\in S}\|a - b\|
    = 2D,
\]
since both projections land in $S$.  A function whose gradient norm is
uniformly bounded by $L$ is $L$-Lipschitz.
 
\paragraph{(iv) $4/\lambda$-smoothness and $\lambda/4$-cocoercivity.}
The Euclidean projection $\Pi_S$ onto a closed convex set is
\emph{firmly non-expansive}:
\[
    \langle \Pi_S(a) - \Pi_S(b),\, a - b\rangle
    \ge \|\Pi_S(a) - \Pi_S(b)\|^2
    \quad\forall\, a,b.
\]
Let $\hat c_1, \hat c_2$ be two prediction vectors and define
$x_k = (c - 2\hat c_k)/\lambda$ for $k=1,2$.
By Corollary~3.1,
$\nabla_{\hat c}\,\ell(\hat c_k) = -2(\Pi_S(x_k) - \Pi_S(v))$
where $v = -c/\lambda$ is common.  Thus
\[
    \nabla\ell(\hat c_1) - \nabla\ell(\hat c_2)
    = -2\bigl(\Pi_S(x_1) - \Pi_S(x_2)\bigr),
\]
and
$x_1 - x_2 = -2(\hat c_1 - \hat c_2)/\lambda$.
Evaluating the inner product:
\begin{align*}
    \langle \nabla\ell(\hat c_1) - \nabla\ell(\hat c_2),\,
            \hat c_1 - \hat c_2\rangle
    &= -2\langle \Pi_S(x_1)-\Pi_S(x_2),\,\hat c_1-\hat c_2\rangle \\
    &= \lambda\langle \Pi_S(x_1)-\Pi_S(x_2),\, x_1-x_2\rangle \\
    &\ge \lambda\,\|\Pi_S(x_1)-\Pi_S(x_2)\|^2 \\
    &= \tfrac{\lambda}{4}\,\|\nabla\ell(\hat c_1)-\nabla\ell(\hat c_2)\|^2,
\end{align*}
where the inequality uses firm non-expansiveness and the last equality uses
$\|\nabla\ell(\hat c_1)-\nabla\ell(\hat c_2)\| = 2\|\Pi_S(x_1)-\Pi_S(x_2)\|$.
This establishes $\lambda/4$-cocoercivity.  By the Baillon--Haddad theorem \citep{bauschke2017convex}, a
convex function with a $\beta$-cocoercive gradient is $1/\beta$-smooth, so
$\ell_{\mathrm{RSPO+}}$ is $4/\lambda$-smooth.
\end{proof}
 
\subsection*{Proof of \cref{thm:decision-bound}: Regret and Decision-Error Bounds}
 
\begin{proof}
Define the regularized objective $f(c,w) = c^T w + \tfrac{\lambda}{2}\|w\|^2$
and let $w^{\star} = \Pi_S(-c/\lambda)$, $\hat w = \Pi_S(-\hat c/\lambda)$
denote the optimal and predicted regularized decisions, respectively.
The regularized regret is
$R(c,\hat w) = f(c,\hat w) - f(c,w^{\star})$.
 
With $u = (c - 2\hat c)/\lambda$, Proposition~3.2 gives
\[
    \ell_{\mathrm{RSPO+}}
    = \tfrac{\lambda}{2}\|u - w^{\star}\|^2
      - \tfrac{\lambda}{2}\,d_S^2(u).
\]
Since $\hat w \in S$, the point $\hat w$ is feasible for the distance
minimization, so $d_S^2(u) \le \|u - \hat w\|^2$.  Therefore
\begin{equation}\label{eq:surrogate-lb}
    \ell_{\mathrm{RSPO+}}
    \ge \tfrac{\lambda}{2}\|u - w^{\star}\|^2
       - \tfrac{\lambda}{2}\|u - \hat w\|^2.
\end{equation}
Expanding the difference of squares:
\begin{align*}
    \tfrac{\lambda}{2}\|u - w^{\star}\|^2 - \tfrac{\lambda}{2}\|u - \hat w\|^2
    &= \lambda\langle u,\, \hat w - w^{\star}\rangle
       + \tfrac{\lambda}{2}\|w^{\star}\|^2
       - \tfrac{\lambda}{2}\|\hat w\|^2 \\
    &= \langle c - 2\hat c,\, \hat w - w^{\star}\rangle
       + \tfrac{\lambda}{2}\|w^{\star}\|^2
       - \tfrac{\lambda}{2}\|\hat w\|^2.
\end{align*}
Rearranging the linear terms:
\begin{align*}
    &= \underbrace{\bigl(c^T \hat w + \tfrac{\lambda}{2}\|\hat w\|^2\bigr)
       - \bigl(c^T w^{\star} + \tfrac{\lambda}{2}\|w^{\star}\|^2\bigr)
       }_{R(c,\,\hat w)}
       + \lambda\|w^{\star}\|^2 - \lambda\|\hat w\|^2
       - 2\hat c^T(\hat w - w^{\star}) \\
    &= R(c,\hat w)
       + 2\Bigl[\bigl(\hat c^T w^{\star} + \tfrac{\lambda}{2}\|w^{\star}\|^2\bigr)
                - \bigl(\hat c^T \hat w + \tfrac{\lambda}{2}\|\hat w\|^2\bigr)\Bigr] \\
    &= R(c,\hat w) + 2\bigl[f(\hat c,w^{\star}) - f(\hat c,\hat w)\bigr].
\end{align*}
Since $f(\hat c,\cdot)$ is $\lambda$-strongly convex in $w$ and is minimized
over $S$ at $\hat w = \Pi_S(-\hat c/\lambda)$, the suboptimality gap satisfies
\[
    f(\hat c, w^{\star}) - f(\hat c, \hat w)
    \ge \tfrac{\lambda}{2}\|w^{\star} - \hat w\|^2.
\]
Combining with \eqref{eq:surrogate-lb} yields \cref{eq:rspo-regret-bound}:
\[
    \ell_{\mathrm{RSPO+}}(c,\hat c;\lambda)
    \ge R(c,\hat w) + \lambda\|\hat w - w^{\star}\|^2.
\]
Since $R(c,\hat w) \ge 0$ (by optimality of $w^{\star}$), the decision-error
bound \cref{eq:rspo-decision-bound} follows immediately:
\[
    \|\hat w - w^{\star}\|
    \le \sqrt{\frac{\ell_{\mathrm{RSPO+}}(c,\hat c;\lambda)}{\lambda}}.
\]

\end{proof}

\subsection*{Proof of \cref{cor:unreg-bound}: Unregularized Regret Bound for RSPO+}
\begin{proof}
By definition, the regularized regret from \cref{thm:decision-bound} expands as $R(c, \hat{w}) = \big(c^T \hat{w} + \frac{\lambda}{2}\|\hat{w}\|^2\big) - \big(c^T w^{\star}_{\mathrm{reg}} + \frac{\lambda}{2}\|w^{\star}_{\mathrm{reg}}\|^2\big)$. Rearranging to isolate the unregularized cost of the prediction gives:
\[
    c^T \hat{w} = R(c, \hat{w}) + c^T w^{\star}_{\mathrm{reg}} + \frac{\lambda}{2}\|w^{\star}_{\mathrm{reg}}\|^2 - \frac{\lambda}{2}\|\hat{w}\|^2.
\]
Because $w^{\star}_{\mathrm{reg}}$ minimizes the regularized objective over $S$, its regularized objective is lower than that of the true unregularized optimum $w^\star_{\mathrm{LP}} \in S$:
\[
    c^T w^{\star}_{\mathrm{reg}} + \frac{\lambda}{2}\|w^{\star}_{\mathrm{reg}}\|^2 \leq c^T w^\star_{\mathrm{LP}} + \frac{\lambda}{2}\|w^\star_{\mathrm{LP}}\|^2 = z^\star + \frac{\lambda}{2}\|w^\star_{\mathrm{LP}}\|^2.
\]
Substituting this upper bound into our expression for $c^T \hat{w}$ and subtracting $z^\star$ from both sides:
\[
    R_{\mathrm{unreg}}(c, \hat{w}) \leq R(c, \hat{w}) + \frac{\lambda}{2}\|w^\star_{\mathrm{LP}}\|^2 - \frac{\lambda}{2}\|\hat{w}\|^2.
\]
Substituting $R(c, \hat{w}) \leq \ell_{\mathrm{RSPO+}}(c, \hat{c}; \lambda) - \lambda\|\hat{w} - w^{\star}_{\mathrm{reg}}\|^2$ from \cref{thm:decision-bound} yields:
\[
    R_{\mathrm{unreg}}(c, \hat{w}) \leq \ell_{\mathrm{RSPO+}} + \frac{\lambda}{2}\|w^\star_{\mathrm{LP}}\|^2 - \lambda\|\hat{w} - w^\star_{\mathrm{reg}}\|^2 - \frac{\lambda}{2}\|\hat{w}\|^2.
\]
To maximize the negative quadratic components, let $x = \hat{w} - w^\star_{\mathrm{reg}}$, so $\|\hat{w}\|^2 = \|x + w^\star_{\mathrm{reg}}\|^2$. The $x$-dependent terms become $-\lambda\|x\|^2 - \frac{\lambda}{2}\|x + w^\star_{\mathrm{reg}}\|^2 = -\frac{3\lambda}{2}\|x\|^2 - \lambda x^T w^\star_{\mathrm{reg}} - \frac{\lambda}{2}\|w^\star_{\mathrm{reg}}\|^2$. This unconstrained concave quadratic is strictly maximized at $x = -w^\star_{\mathrm{reg}}/3$, yielding a maximum value of $-\frac{\lambda}{3}\|w^\star_{\mathrm{reg}}\|^2$. 
Substituting this maximum gives us,
\begin{equation}
    \label{eq:unreg-bound-tight}
    R_{\mathrm{unreg}}(c, \hat{w}) \leq \ell_{\mathrm{RSPO+}}(c, \hat{c}; \lambda) + \frac{\lambda}{2}\|w^\star_{\mathrm{LP}}\|^2 - \frac{\lambda}{3}\|w^\star_{\mathrm{reg}}\|^2.
\end{equation}
Since $\|w^\star_{\mathrm{reg}}\|^2 \ge 0$ and $\|w^\star_{\mathrm{LP}}\| \le R$, this yields the proof.
\end{proof}
\subsection*{Proof of \cref{prop:lambda-bound}: RSPO+ Stability under Perturbation of $\lambda$}
 
\begin{proof}
We prove the two bounds separately.
 
\paragraph{Gradient bound \cref{eq:rspo-grad-lambda-bound}}
By \cref{cor:rspo-gradient},
$\nabla_{\hat c}\,\ell_{\mathrm{RSPO+}}(\lambda)
 = -2\bigl(\Pi_S(u_\lambda) - \Pi_S(v_\lambda)\bigr)$,
where $u_\lambda = (c - 2\hat c)/\lambda$ and $v_\lambda = -c/\lambda$.
For two parameters $\lambda_1, \lambda_2 > 0$:
\begin{align*}
    \|\nabla\ell(\lambda_1) - \nabla\ell(\lambda_2)\|
    &\le 2\bigl(\|\Pi_S(u_{\lambda_1}) - \Pi_S(u_{\lambda_2})\|
               + \|\Pi_S(v_{\lambda_1}) - \Pi_S(v_{\lambda_2})\|\bigr),
\end{align*}
by the triangle inequality.
Since $\Pi_S$ is $1$-Lipschitz (non-expansive),
\[
    \|\Pi_S(u_{\lambda_1}) - \Pi_S(u_{\lambda_2})\|
    \le \|u_{\lambda_1} - u_{\lambda_2}\|
    = \|c - 2\hat c\|\,\bigl|\tfrac{1}{\lambda_1} - \tfrac{1}{\lambda_2}\bigr|,
\]
and similarly
$\|\Pi_S(v_{\lambda_1}) - \Pi_S(v_{\lambda_2})\|
 \le \|c\|\,|\frac{1}{\lambda_1} - \frac{1}{\lambda_2}|$.
Combining gives
\[
    \|\nabla\ell(\lambda_1) - \nabla\ell(\lambda_2)\|
    \le 2\bigl(\|c - 2\hat c\| + \|c\|\bigr)
        \Bigl|\frac{1}{\lambda_1} - \frac{1}{\lambda_2}\Bigr|.
\]
 
\paragraph{Loss bound \cref{eq:rspo-loss-lambda-bound}}
Write the RSPO+ loss via its Bregman divergence representation:
\[
    \ell(\lambda)
    = F_{\lambda,S}(c - 2\hat c)
      - F_{\lambda,S}(-c)
      - \langle \Pi_S(-c/\lambda),\, 2c - 2\hat c\rangle.
\]
We bound the variation of each component.
 
\emph{Envelope terms.}  Using the supremum form
$F_{\lambda,S}(u) = \sup_{s\in S}\{\langle u,s\rangle - \tfrac{\lambda}{2}\|s\|^2\}$,
for any fixed $u$ and $s\in S$:
\[
    \bigl|\langle u,s\rangle - \tfrac{\lambda_1}{2}\|s\|^2\bigr|
    - \bigl|\langle u,s\rangle - \tfrac{\lambda_2}{2}\|s\|^2\bigr|
    \le \tfrac{|\lambda_1 - \lambda_2|}{2}\|s\|^2
    \le \tfrac{R^2}{2}|\lambda_1 - \lambda_2|.
\]
Taking suprema over $s\in S$ on both sides yields
$|F_{\lambda_1,S}(u) - F_{\lambda_2,S}(u)| \le \tfrac{R^2}{2}|\lambda_1 - \lambda_2|$.
Applying this to $u = c - 2\hat c$ and $u = -c$ and summing:
\[
    \bigl|F_{\lambda_1,S}(c-2\hat c) - F_{\lambda_2,S}(c-2\hat c)\bigr|
    + \bigl|F_{\lambda_1,S}(-c) - F_{\lambda_2,S}(-c)\bigr|
    \le R^2|\lambda_1 - \lambda_2|.
\]
 
\emph{Linearization term.}  The variation in the inner-product term is
\begin{align*}
    \bigl|\langle \Pi_S(-c/\lambda_1) - \Pi_S(-c/\lambda_2),\, 2c-2\hat c\rangle\bigr|
    &\le 2\|c - \hat c\|\,\|\Pi_S(-c/\lambda_1) - \Pi_S(-c/\lambda_2)\| \\
    &\le 2\|c - \hat c\|\cdot\|c\|\,\bigl|\tfrac{1}{\lambda_1} - \tfrac{1}{\lambda_2}\bigr|,
\end{align*}
by non-expansiveness of $\Pi_S$.
Combining the envelope and linearization bounds by the triangle inequality:
\[
    |\ell_{\mathrm{RSPO+}}(\lambda_1) - \ell_{\mathrm{RSPO+}}(\lambda_2)|
    \le R^2|\lambda_1 - \lambda_2|
       + 2\|c\|\,\|c - \hat c\|\,
         \Bigl|\frac{1}{\lambda_1} - \frac{1}{\lambda_2}\Bigr|.
\]
\end{proof}

\subsection{Proof of \texorpdfstring{\cref{rem:mse-spo}:}:MSE-SPO+ homotopy}
\begin{proof}
Since $0\in\operatorname{int}(S)$, there exists $r>0$ such that
$B(0,r):=\{x\in\mathbb{R}^d:\|x\|<r\}\subseteq S$.
For any $p$ with $\|p\|<r$ we have $p\in S$, and since $S$ is convex and closed
the projection satisfies $\Pi_S(p)=p$.
Set $u=\tfrac{c-2\hat c}{\lambda}$ and $v=-\tfrac{c}{\lambda}$.
Under the boundedness assumptions,
\[
  \|u\| = \frac{\|c-2\hat c\|}{\lambda}
       \le \frac{C_{\max}+2B}{\lambda}
       \le \frac{C_{\max}+2B}{\bar\lambda}
       = r,
\]
and
\[
  \|v\| = \frac{\|c\|}{\lambda}
       \le \frac{C_{\max}}{\lambda}
       \le \frac{C_{\max}}{\bar\lambda}
       < r.
\]
Hence $u,v\in S$, so $\Pi_S(u)=u$ and $\Pi_S(v)=v$.
From \cref{prop:rspo}, the RSPO+ loss is
\[
  \ell_{\mathrm{RSPO+}}(c,\hat c;\lambda)
  = \frac{\lambda}{2}\bigl(\|u-\Pi_S(v)\|^2 - \|u-\Pi_S(u)\|^2\bigr).
\]
Substituting $\Pi_S(u)=u$ and $\Pi_S(v)=v$ yields
\[
  \ell_{\mathrm{RSPO+}}
  = \frac{\lambda}{2}\bigl(\|u-v\|^2 - 0\bigr)
  = \frac{\lambda}{2}\left\|\frac{c-2\hat c}{\lambda}+\frac{c}{\lambda}\right\|^2
  = \frac{\lambda}{2}\cdot\frac{4\|c-\hat c\|^2}{\lambda^2}
  = \frac{2}{\lambda}\|c-\hat c\|^2.
\]
From \cref{cor:rspo-gradient},
$\nabla_{\hat c}\,\ell_{\mathrm{RSPO+}} = -2\bigl(\Pi_S(u)-\Pi_S(v)\bigr)$.
Substituting the identity projections:
\[
  \nabla_{\hat c}\,\ell_{\mathrm{RSPO+}}
  = -2(u-v)
  = -2\left(\frac{c-2\hat c}{\lambda}+\frac{c}{\lambda}\right)
  = -\frac{4}{\lambda}(c-\hat c).
\]
This is exactly $\nabla_{\hat c}\bigl[\tfrac{2}{\lambda}\|c-\hat c\|^2\bigr]$.
\end{proof}
\subsection{Proof of \texorpdfstring{\cref{thm:main-rspo}}:: RSPO+ Heterogeneity Bound}
\label{app:proof-thm-rspo}

We prove Theorem~\ref{thm:main-rspo} through three supporting lemmas and an algebraic decomposition of the RSPO+ loss. Throughout, we use $\|\cdot\| = \|\cdot\|$ and write $v^\star_{\lambda, S}(c) := \Pi_S(-c/\lambda)$ for the optimal regularized decision.

\paragraph{Loss decomposition.}
Recall from Proposition~\ref{prop:rspo} that the RSPO+ loss is the generalized Bregman divergence of the Moreau envelope of the support function:
\begin{equation}
    \label{eq:app-rspo-bregman}
    \ell_{S}(\hat{c}, c) \;=\; F_{\lambda, S}(c - 2\hat{c}) \;-\; F_{\lambda, S}(-c) \;-\; \bigl\langle \nabla F_{\lambda, S}(-c),\; (c - 2\hat{c}) - (-c) \bigr\rangle.
\end{equation}
Since the Moreau envelope of the support function satisfies $\nabla F_{\lambda, S}(-c) = \Pi_S(-c/\lambda) = v^\star_{\lambda,S}(c)$ and evaluates to
$F_{\lambda,S}(-c) = \langle -c, v^\star_{\lambda,S}(c)\rangle - \tfrac{\lambda}{2}\|v^\star_{\lambda,S}(c)\|^2$,
we expand the linearization term as
\begin{align*}
    -\bigl\langle v^\star_{\lambda,S}(c),\; 2c - 2\hat{c}\bigr\rangle
    &= -2c^T v^\star_{\lambda,S}(c) + 2\hat{c}^T v^\star_{\lambda,S}(c) \\
    &= 2\Bigl(F_{\lambda,S}(-c) + \tfrac{\lambda}{2}\|v^\star_{\lambda,S}(c)\|^2\Bigr) + 2\hat{c}^T v^\star_{\lambda,S}(c).
\end{align*}
Substituting back into~\eqref{eq:app-rspo-bregman} yields
\begin{equation}
    \label{eq:app-rspo-split}
    \ell_{S}(\hat{c}, c) \;=\;
    \underbrace{F_{\lambda,S}(c - 2\hat{c}) + F_{\lambda,S}(-c)}_{\displaystyle M_{\lambda,S}(\hat{c},\, c)}
    \;+\;
    \underbrace{2\hat{c}^T v^\star_{\lambda,S}(c) + \lambda\|v^\star_{\lambda,S}(c)\|^2}_{\displaystyle B_{\lambda,S}(\hat{c},\, c)}.
\end{equation}
This decomposition mirrors the SPO+ structure and allows us to bound the envelope term~$M$ and the decision term~$B$ independently under perturbations to~$c$ and~$S$.

\paragraph{Hausdorff distance.} Unlike SPO+, the origin-centered quadratic penalty $\frac{\lambda}{2}\|w\|^2$ in the Moreau envelope breaks translation invariance. We therefore measure set perturbations via the standard Hausdorff distance $\delta := d_H(S_{\mathrm{ref}}, S_{\mathrm{cl}})$, and define the joint bounding radius $R_{\max} := \max\{R(S_{\mathrm{ref}}),\, R(S_{\mathrm{cl}})\}$.

\begin{lemma}[Sensitivity of the Moreau envelope to set perturbations]
    \label{lem:app-moreau-set}
    Let $S_1, S_2 \subseteq \mathbb{R}^d$ be nonempty compact convex sets with $\max\{R(S_1), R(S_2)\} \le R_{\max}$ and $\delta = d_H(S_1, S_2)$. Then for any $u \in \mathbb{R}^d$:
    \[
        \bigl|F_{\lambda, S_1}(u) - F_{\lambda, S_2}(u)\bigr| \;\le\; \delta\bigl(\|u\| + \lambda R_{\max}\bigr).
    \]
\end{lemma}

\begin{proof}
    Let $w_1 = \arg\max_{w \in S_1}\bigl(u^T w - \tfrac{\lambda}{2}\|w\|^2\bigr)$. By definition of the Hausdorff distance, there exists $\tilde{w}_2 \in S_2$ with $\|w_1 - \tilde{w}_2\| \le \delta$. Since $w_1$ is optimal for $S_1$ and $\tilde{w}_2$ is feasible for $S_2$,
    \begin{align*}
        F_{\lambda, S_1}(u) - F_{\lambda, S_2}(u)
        &\le \bigl(u^T w_1 - \tfrac{\lambda}{2}\|w_1\|^2\bigr) - \bigl(u^T \tilde{w}_2 - \tfrac{\lambda}{2}\|\tilde{w}_2\|^2\bigr) \\
        &= \bigl\langle u - \tfrac{\lambda}{2}(w_1 + \tilde{w}_2),\; w_1 - \tilde{w}_2 \bigr\rangle \\
        &\le \bigl(\|u\| + \tfrac{\lambda}{2}(\|w_1\| + \|\tilde{w}_2\|)\bigr)\|w_1 - \tilde{w}_2\| \\
        &\le \bigl(\|u\| + \lambda R_{\max}\bigr)\,\delta,
    \end{align*}
    where the second line uses the identity $a^T a - b^T b = \langle a+b, a-b\rangle$ applied to the quadratic terms, and the last line uses $\|w_1\|, \|\tilde{w}_2\| \le R_{\max}$. Swapping the roles of $S_1$ and $S_2$ and taking the maximum of both directions yields the absolute-value bound.
\end{proof}

\begin{lemma}[Diameter-Lipschitzness of the joint envelope term]
    \label{lem:app-moreau-diameter}
    Under the compactness assumption of \cref{assum:bounded-set}, for any $c_1, c_2, \hat{c} \in \mathbb{R}^d$ and fixed compact convex set $S$:
    \[
        \bigl|M_{\lambda,S}(\hat{c}, c_1) - M_{\lambda,S}(\hat{c}, c_2)\bigr| \;\le\; D(S)\,\|c_1 - c_2\|.
    \]
\end{lemma}

\begin{proof}
    By the maximization definition, $F_{\lambda,S}(x) - F_{\lambda,S}(y) \le \max_{w \in S}(x-y)^T w$ for any $x, y$ (the quadratic terms cancel in the upper bound by using the same maximizer for both terms). Applying this to both components of $M_{\lambda,S}$:
    \begin{align*}
        F_{\lambda,S}(c_1 - 2\hat{c}) - F_{\lambda,S}(c_2 - 2\hat{c})
        &\le \max_{w \in S} (c_1 - c_2)^T w, \\[4pt]
        F_{\lambda,S}(-c_1) - F_{\lambda,S}(-c_2)
        &\le \max_{v \in S} (c_2 - c_1)^T v
        \;=\; \max_{v \in S} (c_1 - c_2)^T(-v).
    \end{align*}
    Summing these inequalities:
    \begin{align*}
        M_{\lambda,S}(\hat{c}, c_1) - M_{\lambda,S}(\hat{c}, c_2)
        &\le \max_{w, v \in S} (c_1 - c_2)^T(w - v) \\
        &\le \|c_1 - c_2\| \cdot \sup_{w, v \in S}\|w - v\|
        \;=\; \|c_1 - c_2\|\, D(S).
    \end{align*}
    Swapping $c_1$ and $c_2$ and taking the maximum gives the absolute-value bound.
\end{proof}

\begin{lemma}[Set stability of Euclidean projections]
    \label{lem:app-proj-stability}
    Let $S_1, S_2 \subseteq \mathbb{R}^d$ be nonempty compact convex sets with $\max\{R(S_1), R(S_2)\} \le R_{\max}$ and $\delta = d_H(S_1, S_2)$. For any $p \in \mathbb{R}^d$:
    \[
        \|\Pi_{S_1}(p) - \Pi_{S_2}(p)\| \;\le\; \sqrt{2\delta\bigl(\|p\| + R_{\max}\bigr)}.
    \]
\end{lemma}

\begin{proof}
    Let $z_i = \Pi_{S_i}(p)$ for $i = 1, 2$. By the Hausdorff distance bound, there exist proxy points $\tilde{z}_1 \in S_1$ and $\tilde{z}_2 \in S_2$ satisfying $\|z_2 - \tilde{z}_1\| \le \delta$ and $\|z_1 - \tilde{z}_2\| \le \delta$.

    The characterization of Euclidean projections onto closed convex sets gives the variational inequalities
    \[
        \langle p - z_1,\; \tilde{z}_1 - z_1\rangle \le 0, \qquad \langle p - z_2,\; \tilde{z}_2 - z_2\rangle \le 0.
    \]
    We expand $\|z_1 - z_2\|^2$ by inserting the proxy points:
    \begin{align*}
        \|z_1 - z_2\|^2
        &= \langle z_1 - z_2,\; z_1 - z_2\rangle \\
        &= \langle p - z_2,\; z_1 - z_2\rangle + \langle p - z_1,\; z_2 - z_1\rangle \\
        &= \bigl[\langle p - z_2,\; z_1 - \tilde{z}_2\rangle + \langle p - z_2,\; \tilde{z}_2 - z_2\rangle\bigr] \\
        &\quad + \bigl[\langle p - z_1,\; z_2 - \tilde{z}_1\rangle + \langle p - z_1,\; \tilde{z}_1 - z_1\rangle\bigr] \\
        &\le \langle p - z_2,\; z_1 - \tilde{z}_2\rangle + 0 + \langle p - z_1,\; z_2 - \tilde{z}_1\rangle + 0 \\
        &\le \|p - z_2\|\,\|z_1 - \tilde{z}_2\| + \|p - z_1\|\,\|z_2 - \tilde{z}_1\| \\
        &\le \delta\bigl(\|p - z_1\| + \|p - z_2\|\bigr).
    \end{align*}
    Since $z_i \in S_i$ with $\|z_i\| \le R_{\max}$, the triangle inequality gives $\|p - z_i\| \le \|p\| + R_{\max}$ for each $i$. Therefore,
    \[
        \|z_1 - z_2\|^2 \;\le\; 2\delta\bigl(\|p\| + R_{\max}\bigr).
    \]
    Taking the square root completes the proof.
\end{proof}

%\paragraph{Remark.} Lemma~\ref{lem:app-proj-stability} is the reason RSPO+ presents greater stability over SPO+. In SPO+, the optimal decision $w^\star(c) \in \arg\max_{w \in S} c^T w$ selects a vertex of a polyhedral $S$. Under infinitesimal perturbations to the feasible set, this vertex can jump discontinuously across the entire polytope, introducing an irreducible diameter penalty $D_{\min}$ that persists even as heterogeneity vanishes. By contrast, the RSPO+ decision $v^\star_{\lambda,S}(c) = \Pi_S(-c/\lambda)$ is a Euclidean projection, which is $\mathcal{O}(\sqrt{\delta})$-stable for \emph{all} compact convex sets (including polyhedra) without any strong convexity assumption.

\begin{proof}[\textbf{Proof of Theorem~\ref{thm:main-rspo}}]
    Using decomposition~\eqref{eq:app-rspo-split}, the total discrepancy satisfies the triangle inequality:
    \[
        \bigl|\ell_{S_{\mathrm{ref}}}(\hat{c}, c_{\mathrm{ref}}) - \ell_{S_{\mathrm{cl}}}(\hat{c}, c_{\mathrm{cl}})\bigr|
        \;\le\;
        \bigl|M_{\lambda, S_{\mathrm{ref}}}(\hat{c}, c_{\mathrm{ref}}) - M_{\lambda, S_{\mathrm{cl}}}(\hat{c}, c_{\mathrm{cl}})\bigr|
        \;+\;
        \bigl|B_{\lambda, S_{\mathrm{ref}}}(\hat{c}, c_{\mathrm{ref}}) - B_{\lambda, S_{\mathrm{cl}}}(\hat{c}, c_{\mathrm{cl}})\bigr|.
    \]

    \paragraph{Step~1: Bounding $T(\hat{c})$.}
    We add and subtract the intermediate term $M_{\lambda, S_{\mathrm{ref}}}(\hat{c}, c_{\mathrm{cl}})$:
    \begin{align*}
        &\bigl|M_{\lambda, S_{\mathrm{ref}}}(\hat{c}, c_{\mathrm{ref}}) - M_{\lambda, S_{\mathrm{cl}}}(\hat{c}, c_{\mathrm{cl}})\bigr| \\
        &\quad\le
        \underbrace{\bigl|M_{\lambda, S_{\mathrm{ref}}}(\hat{c}, c_{\mathrm{ref}}) - M_{\lambda, S_{\mathrm{ref}}}(\hat{c}, c_{\mathrm{cl}})\bigr|}_{\text{objective shift at fixed } S_{\mathrm{ref}}}
        \;+\;
        \underbrace{\bigl|M_{\lambda, S_{\mathrm{ref}}}(\hat{c}, c_{\mathrm{cl}}) - M_{\lambda, S_{\mathrm{cl}}}(\hat{c}, c_{\mathrm{cl}})\bigr|}_{\text{set shift at fixed } c_{\mathrm{cl}}}.
    \end{align*}
    By Lemma~\ref{lem:app-moreau-diameter}, the objective-shift term is bounded by $D_{\mathrm{ref}}\, c_d$.

    For the set-shift term, recall that $M_{\lambda,S}(\hat{c}, c) = F_{\lambda,S}(c - 2\hat{c}) + F_{\lambda,S}(-c)$. Applying Lemma~\ref{lem:app-moreau-set} to each envelope evaluation independently at arguments $u_1 = c_{\mathrm{cl}} - 2\hat{c}$ and $u_2 = -c_{\mathrm{cl}}$:
    \begin{align*}
        \bigl|M_{\lambda, S_{\mathrm{ref}}}(\hat{c}, c_{\mathrm{cl}}) - M_{\lambda, S_{\mathrm{cl}}}(\hat{c}, c_{\mathrm{cl}})\bigr|
        &\le \delta\bigl(\|c_{\mathrm{cl}} - 2\hat{c}\| + \lambda R_{\max}\bigr) + \delta\bigl(\|c_{\mathrm{cl}}\| + \lambda R_{\max}\bigr) \\
        &= \delta\bigl(\|c_{\mathrm{cl}} - 2\hat{c}\| + \|c_{\mathrm{cl}}\| + 2\lambda R_{\max}\bigr).
    \end{align*}
    Combining these gives:
    \[
        \bigl|M_{\mathrm{ref}} - M_{\mathrm{cl}}\bigr|
        \;\le\;
        D_{\mathrm{ref}}\, c_d + \delta\bigl(\|c_{\mathrm{cl}} - 2\hat{c}\| + \|c_{\mathrm{cl}}\| + 2\lambda R_{\max}\bigr).
    \]
    Alternatively, one may pivot through $M_{\lambda, S_{\mathrm{cl}}}(\hat{c}, c_{\mathrm{ref}})$ instead, yielding the analogous bound with $D_{\mathrm{cl}}$, $\|c_{\mathrm{ref}} - 2\hat{c}\|$, and $\|c_{\mathrm{ref}}\|$. Taking the minimum of both pivots gives $T(\hat{c})$.

    \paragraph{Step~2: Bounding $B(\hat{c})$.}
    Write $v_i := v^\star_{\lambda, S_i}(c_i) = \Pi_{S_i}(-c_i/\lambda)$ for $i \in \{\mathrm{ref}, \mathrm{cl}\}$. The decision-term gap factors as:
    \begin{align*}
        \bigl|B_{\lambda, S_{\mathrm{ref}}}(\hat{c}, c_{\mathrm{ref}}) - B_{\lambda, S_{\mathrm{cl}}}(\hat{c}, c_{\mathrm{cl}})\bigr|
        &= \bigl|2\hat{c}^T v_{\mathrm{ref}} + \lambda\|v_{\mathrm{ref}}\|^2 - 2\hat{c}^T v_{\mathrm{cl}} - \lambda\|v_{\mathrm{cl}}\|^2\bigr| \\
        &= \bigl|2\hat{c}^T(v_{\mathrm{ref}} - v_{\mathrm{cl}}) + \lambda\langle v_{\mathrm{ref}} + v_{\mathrm{cl}},\; v_{\mathrm{ref}} - v_{\mathrm{cl}}\rangle\bigr| \\
        &= \bigl|\langle 2\hat{c} + \lambda(v_{\mathrm{ref}} + v_{\mathrm{cl}}),\; v_{\mathrm{ref}} - v_{\mathrm{cl}}\rangle\bigr| \\
        &\le \bigl(2\|\hat{c}\| + \lambda\|v_{\mathrm{ref}} + v_{\mathrm{cl}}\|\bigr)\,\|v_{\mathrm{ref}} - v_{\mathrm{cl}}\| \\
        &\le \bigl(2\|\hat{c}\| + 2\lambda R_{\max}\bigr)\,\|v_{\mathrm{ref}} - v_{\mathrm{cl}}\|,
    \end{align*}
    where the last line uses $\|v_{\mathrm{ref}} + v_{\mathrm{cl}}\| \le \|v_{\mathrm{ref}}\| + \|v_{\mathrm{cl}}\| \le 2R_{\max}$ since each $v_i \in S_i$.

    It remains to bound $\|v_{\mathrm{ref}} - v_{\mathrm{cl}}\|$. We decompose via the triangle inequality, pivoting through the intermediate projection $\Pi_{S_{\mathrm{ref}}}(-c_{\mathrm{cl}}/\lambda)$:
    \begin{align*}
        \|v_{\mathrm{ref}} - v_{\mathrm{cl}}\|
        &= \bigl\|\Pi_{S_{\mathrm{ref}}}(-c_{\mathrm{ref}}/\lambda) - \Pi_{S_{\mathrm{cl}}}(-c_{\mathrm{cl}}/\lambda)\bigr\| \\
        &\le \underbrace{\bigl\|\Pi_{S_{\mathrm{ref}}}(-c_{\mathrm{ref}}/\lambda) - \Pi_{S_{\mathrm{ref}}}(-c_{\mathrm{cl}}/\lambda)\bigr\|}_{\text{objective shift, fixed set}}
        \;+\;
        \underbrace{\bigl\|\Pi_{S_{\mathrm{ref}}}(-c_{\mathrm{cl}}/\lambda) - \Pi_{S_{\mathrm{cl}}}(-c_{\mathrm{cl}}/\lambda)\bigr\|}_{\text{set shift, fixed objective}}.
    \end{align*}
    The first term is bounded by $\|c_{\mathrm{ref}} - c_{\mathrm{cl}}\|/\lambda = c_d/\lambda$ since Euclidean projection onto a convex set is $1$-Lipschitz (firmly non-expansive). The second term is bounded by Lemma~\ref{lem:app-proj-stability} evaluated at $p = -c_{\mathrm{cl}}/\lambda$:
    \[
        \bigl\|\Pi_{S_{\mathrm{ref}}}(-c_{\mathrm{cl}}/\lambda) - \Pi_{S_{\mathrm{cl}}}(-c_{\mathrm{cl}}/\lambda)\bigr\|
        \;\le\;
        \sqrt{2\delta\Bigl(\frac{\|c_{\mathrm{cl}}\|}{\lambda} + R_{\max}\Bigr)}.
    \]
    Combining yields the decision bound:
    \[
        \|v_{\mathrm{ref}} - v_{\mathrm{cl}}\|
        \;\le\;
        \frac{c_d}{\lambda} + \sqrt{2\delta\Bigl(\frac{\|c_{\mathrm{cl}}\|}{\lambda} + R_{\max}\Bigr)}.
    \]
    Alternatively, pivoting through $\Pi_{S_{\mathrm{cl}}}(-c_{\mathrm{ref}}/\lambda)$ gives the bound with $\|c_{\mathrm{ref}}\|$ replacing $\|c_{\mathrm{cl}}\|$. Taking the minimum of both pivots and multiplying by the prefactor gives $B(\hat{c})$.

    \paragraph{Combining.} Summing the bounds from Steps~1 and~2 completes the proof:
    \[
        \bigl|\ell_{S_{\mathrm{ref}}}(\hat{c}, c_{\mathrm{ref}}) - \ell_{S_{\mathrm{cl}}}(\hat{c}, c_{\mathrm{cl}})\bigr|
        \;\le\; T(\hat{c}) + B(\hat{c}).
    \]
\end{proof}

\paragraph{Asymptotic comparison with SPO+.}

We verify the claims in Table~\ref{tab:bound-comparison} by examining the limiting behavior. As $c_d \to 0$ and $\delta \to 0$:
\begin{itemize}[leftmargin=*]
    \item \textbf{RSPO+:} Both $T(\hat{c})$ and $B(\hat{c})$ vanish since $T = \mathcal{O}(c_d + \delta)$ and $B = \mathcal{O}(c_d/\lambda + \sqrt{\delta})$, giving a combined rate of $\mathcal{O}(c_d + \sqrt{\delta})$. No geometric assumptions on $S$ are required.
    \item \textbf{SPO+:} The heterogeneity bound takes the form $H(\hat{c}) = \mathcal{O}(c_d + \delta_N + D_{\min})$, where $D_{\min} := \min\{D(S_{\mathrm{ref}}), D(S_{\mathrm{cl}})\}$ persists as an irreducible bias even when $c_d = \delta_N = 0$, unless the feasible sets are strongly convex (in which case $D_{\min}$ can be replaced by $\mathcal{O}(\sqrt{\delta_N})$).
\end{itemize}

\section{Experimental Details}
\label{appsec:implementation-details}

This section provides the full experimental details of the setup described in \cref{sec:experiments}.

Experiments were conducted on AMD EPYC Milan CPU nodes with 128 cores per node and 528 GB RAM. The total wall-clock runtime for the reported experiments was approximately 250 hours. For the Warcraft case study, each seed-method run took approximately 4 hours for MSE and SPO+ and approximately 6 hours for RSPO+ and the other DFL methods. For the synthetic knapsack experiments, one full heterogeneity-design run covering the three main methods (MSE, SPO+, RSPO+) and three heterogeneity levels across three axes and five seeds took approximately 90 minutes. The PJM experiments took approximately 2 minutes for five seeds.

\subsection{Shared Configuration}
\label{appsubsec:shared-config}

 For all experiments and baselines, we use a federated warm-start then local personalization. We use a shallow neural network with 64 neurons for the fractional knapsack and PJM experiments. For the Warcraft experiment, we use a truncated ResNet-18 \citep{he2016deep}, as per the literature \citep{pogancic2020Differentiation}. We implement our method and baselines using the \texttt{InferOpt.jl} framework \citep{dalle2022learning} in \texttt{Julia} \citep{bezanson2017julia} with the \texttt{Lux.jl} \citep{innes2018flux} deep learning package. We use the \texttt{Gurobi} \citep{gurobi} through the \texttt{JuMP} framework \citep{Lubin2023}.

While FedRSPO+ uses the per-client $\lambda$ and learning rates defined in \cref{eq:calibration}, we use a learning rate of $10^{-3}$ for the baselines. We use $\kappa^{\mathrm{LR}} =0.1, \kappa^\lambda =1$, and $\alpha=0.01$ for the fractional knapsack experiments, and $\alpha=10^{-3}$ for the PJM and Warcraft experiments, selected to set initial learning rates within the same magnitude as the baseline learning rate $10^{-3}$. We use a stopping threshold $\tau = 1\%$. We report absolute regret $R = z_{\mathrm{LP}}(c,\hat{w}) - z_{\mathrm{LP}}(c,w^\star)$ and relative regret $R/z^{\mathrm{Avg}}_{\mathrm{LP}}(c,w^\star)$ as metrics, where $z^{\mathrm{Avg}}_{\mathrm{LP}}$ is the per-client average ground-truth objective functions.
\paragraph{Optimizer.}
All experiments use Adam~\citep{kingma2014adam} with gradient clipping at norm $1.0$ and no weight decay. Tabular experiments (knapsack, PJM) use \texttt{Float64} precision; the Warcraft experiment uses \texttt{Float32}.
 
\paragraph{Freeze mechanism.}
The annealing schedules for $\lambda$ and $\eta$ in \cref{eq:calibration} use a scheduler: once the stopping criterion (\cref{eq:stopping-condition}) is satisfied, the schedule freezes and training continues. For SPO+ and MSE baselines, all schedules are held constant ($\kappa^\lambda = \kappa^{\mathrm{LR}} = 0$). We also tested early stopping a few rounds after freezing, and found that while it did help mitigate computational overhead, it did not significantly influence training performance, due to the already annealed learning rate slowing training in later rounds. We implement an early stopping condition 3 rounds after the schedule freezing criterion is met. 
 
\subsection{Synthetic Knapsack}
\label{appsubsec:synth-details}

\paragraph{Optimization problem.}
We employ a fractional knapsack with $d=50$ items.
Each client~$j$ solves
\[
  \max_{w\in S_j}\; c_j^T w,
  \qquad
  S_j = \bigl\{w\in[0,1]^{d} : \alpha_j^T w \le B_j\bigr\},
\]
where $c_j$ is the predicted cost (objective) vector, $\alpha_j$ the item weights, and $B_j$ is the knapsack capacity. Since the problem is a maximization problem, we use a sign transform in our loss function implementation, $\max_{w\in S_j}\; c_j^T w = \min_{w\in S_j}\; (-c_j)^T w.$

\paragraph{Ground-truth objective.}
A single shared basis $W_{\mathrm{true}}\!\in\!\mathbb{R}^{d\times p}$ is drawn once.
Each client receives a rotated copy
$W_j = W_{\mathrm{true}}\,R_j$, where
\[
  R_j = \exp\!\Bigl(\eta_{\mathrm{obj}}\,\frac{A_j}{\|A_j\|}\Bigr),
  \qquad
  A_j = G_j - G_j^T,\quad G_j\sim\mathcal{N}(0,I).
\]
Because $A_j$ is skew-symmetric, $\exp(\cdot)$ is an orthogonal matrix, so $R_j$ is a rotation in feature space.
When $\eta_{\mathrm{obj}}=0$ all clients share the same basis. Larger $\eta_{\mathrm{obj}}$ increases the  spread of the per-client bases.
Given a feature vector $x\in\mathbb{R}^{p}$, the true cost vector is
\begin{equation}
  \theta_{\mathrm{true}}(x)
    = -\bigl(1 + (1 + W_j\, x / \sqrt{p}\,)^{4}\bigr),
  \qquad\text{normalized by }\operatorname{mean}(|\theta_{\mathrm{true}}|).
\end{equation}

\paragraph{Constraint heterogeneity (\texorpdfstring{$\eta_{\mathrm{constr}}$}{eta\_constr}).}
A base capacity is set as $\bar B = 0.6\times d = 30$.
Each client draws a multiplicative factor
$f_j \sim 2\,\mathrm{Beta}(1/\eta_{\mathrm{constr}},\,1/\eta_{\mathrm{constr}})$
(mean~1; for $\eta_{\mathrm{constr}}=0$ the factor degenerates to 1)
and receives capacity $B_j = \bar B \cdot f_j$.
When the flag \texttt{eta\_constr\_affects\_weights} is set, each item weight also receives an independent per-item Beta factor drawn from the same family, otherwise item weights are shared across clients.

\paragraph{Covariate-shift heterogeneity (\texorpdfstring{$\eta_{\mathrm{data}}$}{eta\_data}).}
Each client draws a mean shift
$\mu_j = \eta_{\mathrm{data}}\,z_j$, $z_j\sim\mathcal{N}(0,I)$,
and per-feature positive scales $\sigma_j$ constructed with the same Beta-factor mechanism used for constraints.
Features are then sampled as
\[
  x \mid j \;\sim\; \mathcal{N}\!\bigl(\mu_j,\;\operatorname{diag}(\sigma_j^{2})\bigr).
\]

\paragraph{Experimental protocol.}
Each of the 20~clients receives 100 training, 10 validation, and 100 test samples.
Each heterogeneity axis $\eta\in\{0.0, 0.5, 1.0\}$ is varied independently while the others are held at zero.
Results are averaged over 5~seeds $\in\{42, 69, 1993, 67, 2026\}$ and reported in \cref{tab:results}.
 
\paragraph{Hyperparameters.}
Table~\ref{tab:hyper-synth} shows the full training configuration for both phases.

\paragraph{Additional sensitivity analysis.} We perform additional sensitivity analysis across more of the hyperparameters of FedRSPO+ and demonstrate the results in \cref{fig:multifig-appendix}. Overall FedRSPO+ is fairly insensitive to changes in annealing parameter $\kappa$, especially the learning rate parameter $\kappa_{\mathrm{LR}}$. Reducing $\tau$ improves the model slightly, but it reduces the fraction of clients that end up frozen in the 10 rounds of personalization. To achieve even better results, a smaller $\tau$ with more annealing rounds until the freezing criterion is met should be considered. We also varied the number of calibration samples per client used to compute the freeze statistic $\tilde{U}_t$. Performance was largely insensitive to this: as $|V_c|$ increased from 2 to 32, the relative-regret improvement of RSPO+ over SPO+ remained nearly constant, varying only from 3.61\% to 3.72\%. This suggests that the freezing rule is robust to the exact calibration subset size, so we use $|V_c|=8$ in the main experiments.

\paragraph{Additional ablation.} We perform an additional ablation with the fractional knapsack experiment to complement the results of \cref{tab:pjm-fedrspo-ablation} across different heterogeneity settings. The results are shown in \cref{tab:knapsack_component_ablation}.
\begin{table}[ht]
\centering
\caption{Synthetic knapsack hyperparameters. RSPO+-only entries are marked; SPO+ and MSE use the same configuration with constant schedules ($\kappa = 0$).}
\label{tab:hyper-synth}
\small
\begin{tabular}{lcc}
\toprule
\textbf{Hyperparameter} & \textbf{Warm-start} & \textbf{Personalization} \\
\midrule
Communication rounds & 10 & 10 \\
Local epochs per round & 3 & 1 \\
Batch size & 64 & 64 \\
Client fraction & 0.4 & 1.0 \\
$\kappa^\lambda$ (RSPO+ only) & 0.0 & 1.0 \\
$\eta_0$ (baseline) & $10^{-3}$ & $10^{-3}$ \\
$\kappa^{\mathrm{LR}}$ (RSPO+ only) & 0.0 & 0.1 \\
Per-client $\lambda/\eta$ (RSPO+ only) & \checkmark & \checkmark \\
$\alpha$ (\cref{eq:calibration}) & 0.01 & 0.01 \\
Freeze threshold $\tau$ & --- & 0.01 \\
Validation samples/client & --- & 8 \\
\bottomrule
\end{tabular}
\end{table}

\subsection{Warcraft Shortest Path}
\label{appsubsec:warcraft-details}
\paragraph{Projection model} While we use Dijkstra's for the downstream LP oracle, we use a Gurobi QLP for the RSPO+ projections. 
For a Warcraft grid with vertex set $\mathcal V$ and canonical source--sink pair
$(s,t)$, let $\mathcal A \subseteq \mathcal V \times \mathcal V$ denote the
directed 8-neighbor adjacency graph. For each vertex $v$, define incoming and
outgoing arc sets
\[
\delta^-(v) = \{a=(u,v)\in\mathcal A\}, \qquad
\delta^+(v) = \{a=(v,u)\in\mathcal A\}.
\]
Given transformed costs $c \in \mathbb R^{|\mathcal V|}$ and
regularization parameter $\lambda>0$, the RSPO+ projection oracle solves
\begin{subequations}
\label{eq:warcraft-rspo-projection}
\begin{align}
\Pi_S(c)
\in \arg\min_{x,e}\quad
& \sum_{v\in\mathcal V} c_v w_v
  + \frac{\lambda}{2}\sum_{v\in\mathcal V} w_v^2
\\
\text{s.t.}\quad
& w_s = 1,
\\
& w_v = \sum_{a\in\delta^-(v)} e_a,
&& \forall v\in\mathcal V\setminus\{s\},
\\
& w_v = \sum_{a\in\delta^+(v)} e_a,
&& \forall v\in\mathcal V\setminus\{t\},
\\
& \sum_{a\in\delta^-(s)} e_a = 0,
\\
& \sum_{a\in\delta^+(t)} e_a = 0,
\\
& 0 \le w_v \le 1, \quad 0 \le e_a \le 1,
&& \forall v\in\mathcal V, \quad \forall a\in\mathcal A.
\end{align}
\end{subequations}
The returned relaxed Warcraft decision is $w^\star$. Equivalently, $
\Pi_{S}(c) =
\Pi_{S}\!\left(-c / \lambda\right),
$
where $S$ is the relaxed Warcraft path-flow feasible set
defined by the constraints above.
For a true cost vector $c$ and prediction $\hat c$, RSPO+ calls this oracle
$w_u = \Pi_S(2\hat c - c),$ and $
w_v = \Pi_S(c),$
so that $w_u=\Pi_S((c-2\hat c)/\lambda)$ and
$w_v=\Pi_S(-c/\lambda)$.
\paragraph{Dataset and federated partition.}
The Warcraft Shortest Path dataset~\citep{pogancic2020Differentiation} consists of 10{,}000 training, 1{,}000 validation, and 1{,}000 test 96$\times$96 RGB terrain maps on a $12\times 12$ grid with 8-connectivity. In the benchmark-aligned rerun, we form the federated benchmark from the original training split: we reserve 80\% of the training set (8{,}000 samples) for federated training/validation and hold out the remaining 2{,}000 as a benchmark test set. The 8{,}000 federated samples are distributed round-robin (after shuffling) across 44 clients, yielding 182 samples per client before the $80/20$ intra-client split, i.e., roughly 146 training and 36 validation samples per client. Each client is evaluated on all 2{,}000 benchmark test maps. We report mean optimality gap: $(\mathrm{predicted}\;\mathrm{cost} - \mathrm{optimal} \; \mathrm{cost}) / \mathrm{optimal} \; \mathrm{cost}$), and its standard deviation across seeds.
 
\paragraph{Client heterogeneity.}
Rather than 4 corner-defined clients, we use a border-antipodes construction on the $12\times 12$ grid. Since the grid border contains $4\cdot 12 - 4 = 44$ cells, this yields 44 clients. Each client's source is one border cell and its sink is the antipodal border cell under a $180^\circ$ rotation of the grid, inducing all four diagonal endpoint orientations. We additionally enable objective heterogeneity: each client receives its own multiplicative reweighting of the 5 terrain values $c = c_0 \exp(0.6 \cdot x)$, where $c_0$ is the initial traversal cost of the terrain type and $x \sim \mathcal{N}(0,1)$. Therefore, clients differ in both endpoints and path costs.
 
\paragraph{Architecture.}
The model uses a truncated ResNet-18 backbone, followed by adaptive max-pooling to a $12\times 12$ grid and channel averaging to produce a $12 \times 12$ real-valued weight matrix per sample, following the methodology of \citep{mandi2024decision}. Predicted cost grids are decoded into hard paths using the benchmark shortest-path solver (Dijkstra). Float32 precision is used throughout.
 
\paragraph{Additional baselines.}
Beyond RSPO+, SPO+, and MSE, the Warcraft experiment includes perturbed Fenchel--Young loss (PFYL), decision-focused perturbed MSE (DPO). PFYL and DPO both use 10 perturbation samples with perturbation scale $\epsilon = 1.0$.
 
\paragraph{Hyperparameters.}
Table~\ref{tab:hyper-warcraft} reports the main configuration. The Warcraft rerun uses five seeds rather than a single seed; reported variability therefore reflects both seed-to-seed and client-level variation.
 
\begin{table}[ht]
\centering
\caption{Warcraft shortest path hyperparameters.}
\label{tab:hyper-warcraft}
\small
\begin{tabular}{lcc}
\toprule
\textbf{Hyperparameter} & \textbf{Warm-start} & \textbf{Personalization} \\
\midrule
Communication rounds & 10 & 10 \\
Local epochs per round & 3 & 2 \\
Batch size & 64 & 64 \\
Client fraction & 0.5 & 1.0 \\
$\lambda_0$ (RSPO+ only) & 2.0 & 2.0 \\
$\kappa^\lambda$ (RSPO+ only) & 0.0 & 1.0 \\
$\eta_0$ & $10^{-3}$ & $10^{-3}$ \\
$\kappa^{\mathrm{LR}}$ (RSPO+ only) & 0.0 & 0.1 \\
Per-client $\lambda/\eta$ (RSPO+ only) & \checkmark & \checkmark \\
$\alpha$ (\cref{eq:calibration}, RSPO+ only) & 0.001 & 0.001 \\
Perturbation samples (PFYL/DPO) & 10 & 10 \\
Perturbation scale $\epsilon$ (PFYL/DPO) & 1.0 & 1.0 \\
Freeze threshold $\tau$ (RSPO+ only) & 0.03 & 0.01 \\
\bottomrule
\end{tabular}
\end{table}
\begin{table*}[h]
\centering
\small
\setlength{\tabcolsep}{6pt}
\caption{Fractional-knapsack component ablation.
Values are mean test relative regret (\%) $\pm$ sample standard
deviation across five seeds, after averaging clients within each seed;
lower is better.}
\label{tab:knapsack_component_ablation}
\begin{tabular}{lccc}
\toprule
Method & $(0,0,0)$ $\downarrow$
       & $(1,0,0)$ $\downarrow$
       & $(0,1,0)$ $\downarrow$ \\
\midrule
Local-only RSPO+
    & $9.74 \pm 2.73$
    & $9.80 \pm 2.78$
    & $6.99 \pm 1.89$ \\
\midrule
FedSPO+
    & $9.69 \pm 2.08$
    & $10.93 \pm 2.19$
    & $12.89 \pm 1.99$ \\
RSPO+ \textit{(surrogate only)}
    & $2.26 \pm 0.40$
    & $3.05 \pm 0.25$
    & $4.71 \pm 0.49$ \\
RSPO+ \textit{w/ }$\lambda_j/\eta_j$
    & $2.26 \pm 0.40$
    & $3.05 \pm 0.25$
    & $4.63 \pm 0.46$ \\
RSPO+ \textit{w/ }$\lambda_j/\eta_j$\textit{ + P}
    & $2.16 \pm 0.36$
    & $2.41 \pm 0.38$
    & $4.08 \pm 0.41$ \\
RSPO+ (Full)
    & $\mathbf{1.86} \pm 0.26$
    & $\mathbf{2.02} \pm 0.26$
    & $\mathbf{2.58} \pm 0.40$ \\
\bottomrule
\end{tabular}
\end{table*}

\subsection{PJM Energy Pricing}
\label{appsubsec:pjm-details}
 We use hourly day-ahead locational marginal pricing (LMP) data from the PJM Interconnection, a regional transmission organization spanning 13 U.S. states \citep{pjm_lmp_2025}. We use the following nodes as clients: DOM, PEPCO, PPL, DUQ, BGE, AECO, ATSI, RECO, PSEG, DPL, DAY, METED, JCPL, APS, DEOK, AEP, COMED, PECO, and PENELEC. Prices are retrieved for the year 2025 and assigned to clients, with variation arising from localized congestion effects and are in dollars per MWh \citep{litvinov2004marginal}. Energy markets provide a natural setting for decision-focused federated learning, as operational data are often siloed across entities and subject to regulatory protections such as NERC CIP-011 \citep{nerc0114cip}.
 
\paragraph{Data processing.}
The 8 weather covariates are: mean/min/max temperature (${}^\circ$C), mean humidity, mean wind speed, mean/max solar irradiance, and day of week. The 24 target dimensions are $c_h = \mathrm{LMP}_h + \mathrm{congestion}_h$ for hours $h = 0, \ldots, 23$. The demand satisfaction constraint for client $k$ is $\sum_h w_h = \mathrm{demand}_k$ with $w_h \in [0,1]$, solved as a fractional knapsack with equality constraint.

Each client corresponds to a fixed daily demand customer with demand normalized to a range of 4 to 12 MWh. This setup captures heterogeneity across participants while maintaining realistic operational constraints.

\paragraph{Optimization problem.}
The downstream task is to minimize daily electricity cost subject to daily demand. This can be formulated as a capacity-constrained fractional knapsack problem, as described in \cref{subsec:frac-knap}. Specifically, the feasible set is defined as $S = \{w \in [0,1]^{24} : \mathbf{1}^Tw = D_j\}$, where $D_j$ is the demand at client $j$, and the problem is defined as $\min_{w \in S} c^Tw$. Demand is randomly assigned from $\{4,6,8,10,12\}$ to the PJM clients.

\paragraph{Data splits.}
The data is split chronologically with no shuffling: 50\% train, 10\% validation, 40\% test, allocated by date index, with 19 total clients.

\paragraph{Ablation Study.}
For our ablation study, we first evaluate a federated-only training run with standard MSE loss and evaluate the mean and standard deviation across 5 seeds. We then replace the MSE loss with the DFL SPO+ loss and further incorporate RSPO+. We use the average of the would-be per-client calibrated parameters, or $\frac{1}{|C|}\sum_{j=1}^{|C|} \lambda_j$ and $\frac{1}{|C|}\sum_{j=1}^{|C|} \eta_j$ from \cref{eq:calibration}, where $|C|$ is the number of clients. We set $\kappa^\lambda = \kappa^{\mathrm{LR}} = 0$ for the first 2 RSPO+ ablations. We then gradually introduce more aspects of the FedRSPO+ algorithm and evaluate similarly.
To showcase the modularity of the algorithm, we incorporate two popular FL methods: FedProx \citep{li2020federated} and Ditto \citep{li2021ditto}. During the federated stage, FedProx adds a regularization term $\mu/2 \|\theta_t - \theta\|^2$ to each client's local objective, where $\theta_t$ is the global model at the start of the round and $\theta$ is the client model. That means that FedProx only affects the federated training stage, not the personalization stage. We use $\mu = 1$ for our experiment, and keep all other parameters the same as in standard FedRSPO+.
For Ditto, the federated training phase is identical to standard FedRSPO+. Similarly to FedProx, each client solves a regularized version of their local objective, adding the term $\mu/2 \|\theta_t - \theta\|^2$. As opposed to FedProx, Ditto is introduced in the personalization phase only. We use $\mu = 1$ for the regularization parameter, and keep all other settings the same as before.
 
\begin{table}[ht]
\centering
\caption{PJM hyperparameters.}
\label{tab:hyper-pjm}
\small
\begin{tabular}{lcc}
\toprule
\textbf{Hyperparameter} & \textbf{Warm-start} & \textbf{Personalization} \\
\midrule
Communication rounds & 10 & 10 \\
Local epochs per round & 3 & 1 \\
Batch size & 64 & 64 \\
Client fraction & 0.4 & 1.0 \\
$\kappa^\lambda$ (RSPO+ only) & 0.0 & 1.0 \\
$\eta_0$ (baseline) & $10^{-3}$ & $10^{-3}$ \\
$\kappa^{\mathrm{LR}}$ (RSPO+ only) & 0.0 & 0.1 \\
Per-client $\lambda/\eta$ (RSPO+ only) & \checkmark & \checkmark \\
$\alpha$ (\cref{eq:calibration}) & 0.01 & $10^{-3}$ \\
Freeze threshold $\tau$ & --- & 0.03 \\
Validation samples/client & --- & 8 \\
\bottomrule
\end{tabular}
\end{table}

\paragraph{Leave-one-client-out PJM.} In the original experiment, each client has roughly six months of training data, and local-only RSPO+ already achieves 61.53 mean regret, compared with 60.65 for full FedRSPO+. The large local datasets leave little estimation error for federation to reduce. To investigate a case where federation would be very helpful for PJM, we design a leave-one-client-out experiment, where we train 18/19 clients normally in a federation, and leave one client out, who has access to only [3, 7, 14, 30, 60] days of data. The client then compares their local model with the federated model, fine-tuned on their data. We repeat this procedure for all 19 client nodes. In both cases we use RSPO+ to isolate the federated effect. The results are shown in \cref{tab:pjm-cold-start-regret,tab:pjm-cold-start-client-benefit}.

\begin{table*}[t]
\centering
\caption{Leave-one-client-out experiment for PJM. Results are reported in absolute decision regret; lower is better.}
\label{tab:pjm-cold-start-regret}
\small
\setlength{\tabcolsep}{3pt}
\begin{tabular}{@{}rcccccc@{}}
\toprule
\shortstack{Target training\\days} &
\shortstack{Local\\RSPO$^{+}$} &
\shortstack{Federated\\only} &
\shortstack{FedRSPO$^{+}$\\personalized} &
\shortstack{Gap vs. local\\(Fed+P)} &
\multicolumn{2}{c}{Relative reduction} \\
\cmidrule(lr){6-7}
& & & & & Fed-only & Fed+P \\
\midrule
3  & $168.48 \pm 10.29$ & $73.58 \pm 1.86$ & $73.06 \pm 2.19$ & $95.43 \pm 9.30$  & $56.3\%$ & $56.6\%$ \\
7  & $167.49 \pm 9.49$  & $73.58 \pm 1.86$ & $73.64 \pm 2.14$ & $93.85 \pm 8.71$  & $56.1\%$ & $56.0\%$ \\
14 & $164.44 \pm 10.91$ & $73.58 \pm 1.86$ & $73.31 \pm 2.15$ & $91.13 \pm 10.46$ & $55.3\%$ & $55.4\%$ \\
30 & $164.37 \pm 7.91$  & $73.58 \pm 1.86$ & $73.12 \pm 2.11$ & $91.26 \pm 7.62$  & $55.2\%$ & $55.5\%$ \\
60 & $153.86 \pm 8.17$  & $73.58 \pm 1.86$ & $72.20 \pm 1.84$ & $81.67 \pm 6.51$  & $52.2\%$ & $53.1\%$ \\
\bottomrule
\end{tabular}
\end{table*}

\begin{table}[t]
\centering
\caption{Client-level benefit of federated-only training over the original
local RSPO$^{+}$ baseline with the PJM experiment.}
\label{tab:pjm-cold-start-client-benefit}
\small
\begin{tabular}{@{}rcc@{}}
\toprule
\shortstack{Target training\\days} &
\shortstack{Federated-only\\beats local} &
\shortstack{Median client reduction\\{}[IQR]} \\
\midrule
3  & $19/19$ & $51.1\%\;[37.2\%,\,66.3\%]$ \\
7  & $19/19$ & $47.9\%\;[33.9\%,\,65.4\%]$ \\
14 & $19/19$ & $48.6\%\;[27.4\%,\,65.6\%]$ \\
30 & $19/19$ & $45.1\%\;[28.1\%,\,64.8\%]$ \\
60 & $18/19$ & $41.0\%\;[27.4\%,\,57.2\%]$ \\
\bottomrule
\end{tabular}
\end{table}

The federated-only stays constant since it does not train on the held-out target's dataset, so is unaffected by a larger target dataset. In the data-poor regimes the federated model provides the largest benefit, with the personalization providing marginal or no benefit (even 0.06 worse than federated-only at 7 days, which is negligible), due to the small target dataset. Personalization performance improves noticeably at 60 target days. In terms of per-client differences, 19/19 clients prefer the federated-only model over the local model for target days = 3,7,14, and 30, and 18/19 for target days = 60. We provide a second table in \cref{tab:pjm-cold-start-client-benefit} with inter-client median reduction and the interquartile range (IQR) for federated only versus local.

At 60 days, PENELEC is the only client with negative transfer from local to federated. These values are lower than those in \cref{tab:pjm-cold-start-regret} because the \cref{tab:pjm-cold-start-regret} computes a single ratio of aggregate means, which weights high-regret clients more heavily, whereas in \cref{tab:pjm-cold-start-client-benefit} each client contributes one percentage equally. We report only federated-only gains, since the benefits of personalization do not meaningfully change the results in this experiment. These IQRs are across the 19 client-level percentage reductions. For each client, regret is first averaged over five seeds and then converted to a percentage relative to that client’s local mean. Each experiment rotates all 19 PJM clients as the held-out target and uses five training seeds, giving $19 \times 5 = 95$ observations per target-data budget. For each seed, we first average the metric over the 19 held-out targets. 

\paragraph{Fixed demand in PJM.} To isolate the effect of feasible set heterogeneity in the PJM experiment and to show that optimization heterogeneity plays a role and FedRSPO+ performs better than federated SPO+ under heterogeneity, we reran the PJM experiment but set each client's demand to be equal. We repeat this for all 5 demand values [4,6,8,10,12] for 5 seeds. At each demand level, the gap between FedRSPO+ and FedSPO+ decreased on average. For seed $s$, let $G_{\mathrm{het},s}$ and $G_{d,s}$ denote the
regret gaps under heterogeneous demands and common demand $d$,
respectively. Each gap is defined as federated-only SPO$^{+}$ regret
minus personalized FedRSPO$^{+}$ regret. We compute,
\[
\Delta_{d,s}
=
100\,\frac{G_{d,s}-G_{\mathrm{het},s}}{G_{\mathrm{het},s}}.
\]
In \cref{tab:pjm-equal-demand-gap} we report the percentage change, where a negative gap corresponds to a narrower gap. In other words, a negative gap means FedRSPO+ and federated SPO+'s performance was closer with homogeneous feasible sets. The mean percentage change
is negative at every demand level. This is consistent with feasible-set
heterogeneity contributing to the observed performance difference.

\begin{table}[t]
\centering
\caption{Effect of equal client demands on the regret gap between
federated-only SPO$^{+}$ and personalized FedRSPO$^{+}$ on the
PJM-derived constructed-cost benchmark.}
\label{tab:pjm-equal-demand-gap}
\small
\begin{tabular}{@{}rc@{}}
\toprule
Common demand & Percentage change in gap \\
\midrule
4  & $-18.50\% \pm 9.73\%$  \\
6  & $ -8.37\% \pm 10.94\%$ \\
8  & $ -4.97\% \pm 10.03\%$ \\
10 & $-24.33\% \pm 8.71\%$  \\
12 & $-20.13\% \pm 8.16\%$  \\
\bottomrule
\end{tabular}
\end{table}
\subsection{Centralized and Local-only Baselines}
\label{appsubsec:additional-baselines}
 
To contextualize the federated results in \cref{tab:results} and \cref{tab:pjm-fedrspo-ablation}, we additionally evaluate two non-federated training regimes.
\textbf{Centralized} pools all client training data into a single dataset and trains one global model, evaluated on each client's downstream objective at test time. It is a centralized proxy for the value of data sharing when no privacy constraint applies. \textbf{Local-only} trains a separate model per client on that client's data alone, with no cross-client communication. It is a lower-bound proxy for federation that quantifies how much each client gains from collaboration. Both regimes use the same architectures, optimizer, and number of
gradient updates per client as the federated runs to keep compute comparable, and we report mean $\pm$ std over the same 5 seeds.
 
\begin{table}[h]
\centering
\small
\caption{Synthetic knapsack: relative regret (\%) of centralized and local-only baselines on representative heterogeneity configurations, alongside the federated results from \cref{tab:results} for reference. Mean $\pm$ std over 5 seeds. Lower is better ($\downarrow$).}
\label{tab:additional-knapsack}
\begin{tabular}{lcccc}
\toprule
$(\eta_o,\eta_c,\eta_d)$ & Centralized & Local-only & Fed & Fed + P \\
\midrule
\multicolumn{5}{l}{\textit{SPO+}} \\
$(0,0,0)$   & $0.32\pm 0.06$ & $2.42\pm 0.30$ & $9.7\pm 2.1$ & $5.6\pm 1.4$ \\
$(0,0,0.5)$ & $3.49\pm 0.62$ & $2.93\pm 0.54$ & $14.1\pm 2.0$ & $6.9\pm 1.3$ \\
$(0,0,1)$   & $6.97\pm 0.89$ & $2.47\pm 0.39$ & $17.7\pm 2.1$ & $6.9\pm 0.8$ \\
\midrule
\multicolumn{5}{l}{\textit{RSPO+ (Ours)}} \\
$(0,0,0)$   & $0.75\pm 0.08$ & $9.74\pm 2.73$ & $2.3\pm 0.4$ & $1.9\pm 0.3$ \\
$(0,0,0.5)$ & $4.12\pm 0.40$ & $6.93\pm 1.15$ & $6.5\pm 0.8$ & $3.8\pm 0.2$ \\
$(0,0,1)$   & $7.42\pm 0.72$ & $4.15\pm 0.69$ & $9.3\pm 1.2$ & $3.4\pm 0.2$ \\
\bottomrule
\end{tabular}
\end{table}

\paragraph{Synthetic knapsack.}
Table~\ref{tab:additional-knapsack} reports relative regret for different heterogeneity configurations. We can see two patterns. First, centralized training is the strongest baseline when clients share a data distribution (rows with $\eta_d=0$): centralized SPO+ achieves $0.32\%$ relative regret at $(0,0,0)$, far below any federated or local-only number. But centralized performance degrades sharply once clients differ in their feature distributions: at $(0,0,1)$, centralized SPO+ rises to $6.97\%$, and centralized RSPO+ to $7.42\%$, because pooling distributionally distinct clients creates a misspecified training problem. Second, local-only training stays relatively stable under heterogeneity: local-only SPO+ averages $2.96\pm 0.73\%$ across all 27 heterogeneity configurations and beats centralized SPO+ in every configuration with $\eta_d>0$. Local-only RSPO+, by contrast, is much weaker than its federated counterpart ($9.74\pm 2.73\%$ at $(0,0,0)$ versus $1.9\%$ for FedRSPO+ with personalization in \cref{tab:results}), supporting the interpretation that RSPO+'s advantage lies in federation and mitigating heterogeneity. 

\begin{table}[h]
\centering
\small
\caption{PJM energy pricing: mean absolute regret across training regimes. Mean $\pm$ std over 5 seeds. Federated rows are reproduced from \cref{tab:pjm-fedrspo-ablation} for reference. Lower is better ($\downarrow$).}
\label{tab:additional-pjm}
\begin{tabular}{lccc}
\toprule
Regime & MSE & SPO+ & RSPO+ (Ours) \\
\midrule
Centralized & $70.26\pm 0.61$  & $69.37\pm 0.44$  & $69.32\pm 1.18$ \\
Local-only & $163.76\pm 10.12$& $69.57\pm 1.72$  & $61.53\pm 0.66$ \\
Federated (naive)  & $175.45\pm 7.26$ & $81.68\pm 1.77$  & $72.89\pm 4.23$ \\
FedRSPO+ (Full)    & ---              & ---              & $\mathbf{60.65\pm 0.80}$ \\
\bottomrule
\end{tabular}
\end{table}

\paragraph{PJM energy pricing.}
\cref{tab:additional-pjm} reports the same comparison on the PJM case study, using mean absolute regret to match \cref{tab:pjm-fedrspo-ablation}. Three observations stand out.
First, centralized training close the gap between the three methods: MSE, SPO+, and RSPO+ all sit between $69$ and $71$, indicating that the PJM clients share enough structure that pooled fitting is a strong strategy regardless of the loss. Second, local-only MSE remains poor ($163.76\pm 10.12$, close to the federated MSE result of $175.45\pm 7.26$), confirming that the
improvements in the federated pipeline require decision-aware training to achieve competitive results. Third, local-only RSPO+ achieves $61.53\pm 0.66$, nearly matching the
full FedRSPO+ result of $60.65\pm 0.80$ from
Table~\ref{tab:pjm-fedrspo-ablation}. This reflects the relatively homogeneous downstream structure of PJM
(all clients face the same top-$k$ selection problem on overlapping price signals), where personalization captures most of the available
gain. On synthetic instances with structurally heterogeneous feasible sets, the analogous local-only RSPO+ baseline is much weaker, indicating
that federation contributes most when clients differ in their optimization structure rather than only in scale.

\paragraph{Discussion.} Together, these baselines reinforce the interpretation of the main results. Centralized training is a useful upper-bound comparator but is not uniformly dominant: it excels when clients are statistically aligned and degrades sharply under data heterogeneity. Local-only training is a strong personalization baseline in two specific regimes (SPO+ in synthetic data-heterogeneous settings and RSPO+ on PJM) but offers no shared model and is inconsistent across heterogeneity types, due to the addition of the downstream objective. The main FedRSPO+ result is a robust middle ground that retains performance across the heterogeneity axes while supporting data privacy. The PJM observation that local-only RSPO+ is competitive with FedRSPO+ also identifies a regime where the federated machinery is nearly redundant. We view formalizing this regime, where shared training adds little beyond the surrogate itself, as a useful direction for follow-up work.

\subsection{QP--LP decision alignment}
\label{appsubsec:qp-lp}
To complement the oracle-bias freeze statistic $\tilde U_t$, we directly measure the decision-space discrepancy between the regularized predicted decision and the LP predicted decision,
\[
\Delta_{\mathrm{QP-LP}}
=
\mathbb{E}\left[
\left\|
\Pi_S(-\hat c/\lambda_t)
-
\arg\min_{w\in S} \hat c^T w
\right\|
\right].
\]
On synthetic knapsack, personalization consistently reduces this gap across all configurations: the mean gap decreases from $3.715$ after the federated warm start to $1.166$ after personalization, a $69\%$ reduction, and this decrease occurs in all $135/135$ configurations. The freeze statistic aligns with this result. Clients that satisfy the freeze criterion have lower average decision gaps ($1.112$) than clients that do not freeze ($1.338$), and lower final oracle-bias values $\tilde{U}_t$ ($0.00858$ versus $0.01735$). The hardest cases are primarily driven by data heterogeneity, where increasing $\eta_d$ from $0$ to $1$ raises the personalized decision gap from $1.036$ to $1.335$ and reduces the frozen-client fraction from $0.917$ to $0.578$.

We observe the same qualitative effect in the PJM case study, where personalization reduces the predicted QP--LP gap from $1.955$ to $1.149$, a $41\%$ reduction. This residual gap is not negligible, about $42.6\%$ of the average predicted LP decision norm, but it confirms that personalization moves the regularized predicted decision substantially closer to the LP decision used at evaluation. Finally, the final freeze statistic is strongly correlated with the predicted QP--LP decision gap on knapsack at the configuration level ($r=0.807$), suggesting that $\tilde U_t$ is a useful proxy for QP--LP alignment. The direct decision-gap diagnostic captures predicted-decision alignment more explicitly, but adds computational overhead if tracked in situ ($|V_c|$ additional projections per client per round).

\begin{table}[t]
\centering
\caption{QP--LP discrepancies after personalization.
Values are mean normalized gaps (\%; lower is better),
averaged over clients and five seeds, and additionally over
27 heterogeneity configurations for knapsack.
The PJM diagnostic uses 10 personalization rounds.}
\label{tab:qp_lp_gaps}
\begin{tabular}{lcc}
\toprule
Dataset & Decision gap (\%) & Objective gap (\%) \\
\midrule
Knapsack    & 21.44 & 3.26 \\
PJM & 42.57 & 10.43 \\
\bottomrule
\end{tabular}
\end{table}

We also evaluate the decision and objective gaps for the fractional knapsack and PJM experiments. For client $j$, let $\widehat w^{\mathrm{QP}}_{ji}$ and
$\widehat w^{\mathrm{LP}}_{ji}$ denote the QP and LP decisions
under the same predicted cost vector $\widehat c_{ji}$,
with the QP using the client's final smoothing coefficient $\lambda$.
Let
$z^\star_{ji}=\min_{w\in S_j} c_{ji}^{\top}w$
be the LP optimum under the true cost vector.
With $\epsilon=10^{-8}$, the reported values in \cref{tab:qp_lp_gaps} use the following normalizations,
\begin{align}
D_j
&=
100\,
\frac{
  \frac{1}{n_j}\sum_{i=1}^{n_j}
  \left\|\widehat w^{\mathrm{QP}}_{ji}
       -\widehat w^{\mathrm{LP}}_{ji}\right\|_2
}{
  \frac{1}{n_j}\sum_{i=1}^{n_j}
  \left\|\widehat w^{\mathrm{LP}}_{ji}\right\|_2
  +\epsilon
},
\\
O_j
&=
\frac{100}{n_j}\sum_{i=1}^{n_j}
\frac{
  \left|c_{ji}^{\top}
  \left(\widehat w^{\mathrm{QP}}_{ji}
       -\widehat w^{\mathrm{LP}}_{ji}\right)\right|
}{
  |z^\star_{ji}|+\epsilon
}.
\end{align}
Both objectives are evaluated
using the true linear costs, without the quadratic penalty.

\begin{table}[t]
\centering
\caption{QP--LP decision alignment on synthetic knapsack. Personalization substantially reduces the predicted decision discrepancy between the regularized RSPO+ decision and the LP decision used at evaluation.}
\label{tab:qplp-gap}
\begin{tabular}{lccc}
\toprule
Metric & Mean & Median & Range \\
\midrule
Global QP--LP gap & 3.715 & 3.654 & 3.080--4.456 \\
Personalized QP--LP gap & 1.166 & 1.118 & 0.922--1.521 \\
Personalized / global ratio & 0.314 & 0.316 & 0.258--0.354 \\
Final freeze statistic & 0.01067 & 0.00959 & 0.00822--0.01761 \\
Frozen-client fraction & 0.761 & 0.800 & 0.300--1.000 \\
\bottomrule
\end{tabular}
\end{table}

\subsection{Linear predictor results}

\begin{table*}[t]
\centering
\caption{Bias-free linear predictor results on synthetic knapsack.
Entries report mean client relative regret (\%) as mean $\pm$
sample standard deviation across five seeds.
The heterogeneity tuple represents objective, constraint, and
data-distribution heterogeneity, respectively.
``Fed'' denotes the federated endpoint; ``+P'' includes subsequent
personalization. Bold indicates the lowest mean regret within
each training stage.}
\label{tab:linear_knapsack}
\small
\setlength{\tabcolsep}{3pt}
\resizebox{\textwidth}{!}{%
\begin{tabular}{l*{10}{c}}
\toprule
& \multicolumn{2}{c}{MSE}
& \multicolumn{2}{c}{PFYL}
& \multicolumn{2}{c}{DPO}
& \multicolumn{2}{c}{SPO+}
& \multicolumn{2}{c}{RSPO+} \\
\cmidrule(lr){2-3}
\cmidrule(lr){4-5}
\cmidrule(lr){6-7}
\cmidrule(lr){8-9}
\cmidrule(lr){10-11}
Heterogeneity
& Fed & +P
& Fed & +P
& Fed & +P
& Fed & +P
& Fed & +P \\
\midrule
$(0,0,0)$
& $37.76\pm2.44$ & $34.89\pm2.16$
& $41.08\pm3.14$ & $38.93\pm3.22$
& $44.99\pm4.03$ & $44.43\pm4.32$
& $39.57\pm3.16$ & $36.79\pm3.14$
& $\mathbf{12.64\pm0.91}$ & $\mathbf{11.04\pm0.70}$ \\

$(0.5,0,0)$
& $37.88\pm2.45$ & $34.97\pm2.20$
& $41.17\pm3.10$ & $39.01\pm3.19$
& $45.02\pm4.00$ & $44.45\pm4.27$
& $39.68\pm3.10$ & $36.90\pm3.16$
& $\mathbf{13.30\pm1.04}$ & $\mathbf{11.54\pm0.82}$ \\

$(1,0,0)$
& $38.42\pm2.43$ & $35.42\pm2.04$
& $41.57\pm2.95$ & $39.38\pm3.06$
& $45.08\pm3.78$ & $44.49\pm4.05$
& $40.20\pm3.00$ & $37.35\pm3.02$
& $\mathbf{15.09\pm1.14}$ & $\mathbf{13.00\pm0.97}$ \\

$(0,0.5,0)$
& $42.26\pm2.69$ & $39.38\pm2.68$
& $47.22\pm2.96$ & $45.43\pm2.89$
& $49.87\pm3.27$ & $49.38\pm3.31$
& $45.00\pm3.06$ & $42.46\pm2.95$
& $\mathbf{14.83\pm1.19}$ & $\mathbf{12.98\pm0.93}$ \\

$(0,1,0)$
& $44.03\pm2.24$ & $41.08\pm2.22$
& $49.45\pm2.82$ & $47.93\pm2.75$
& $51.97\pm2.87$ & $51.57\pm2.97$
& $47.11\pm2.76$ & $44.68\pm2.87$
& $\mathbf{15.71\pm0.58}$ & $\mathbf{13.73\pm0.52}$ \\

$(0,0,0.5)$
& $40.93\pm2.02$ & $38.34\pm1.67$
& $43.42\pm2.59$ & $41.60\pm2.66$
& $46.07\pm3.04$ & $45.52\pm3.32$
& $42.43\pm2.49$ & $40.01\pm2.54$
& $\mathbf{18.85\pm1.67}$ & $\mathbf{16.62\pm1.48}$ \\

$(0,0,1)$
& $44.53\pm1.99$ & $41.90\pm1.51$
& $45.96\pm2.62$ & $44.20\pm2.56$
& $47.39\pm2.81$ & $46.90\pm2.88$
& $45.35\pm2.32$ & $43.05\pm2.26$
& $\mathbf{28.29\pm2.31}$ & $\mathbf{25.19\pm2.35}$ \\
\bottomrule
\end{tabular}%
}
\end{table*}

To strengthen the connection to \cref{thm:fedrspo-convergence}, we reran all seven synthetic-knapsack settings using a bias-free linear predictor, consistent with the predictor class assumed in the convergence analysis. RSPO+ achieves the lowest regret in all 14 federated and personalized comparisons. Averaged across settings, RSPO+ obtains 16.96\% federated regret, compared with 40.83\% for the strongest baseline, and 14.87\% personalized regret, compared with 38.00\% for the strongest baseline. This corresponds to approximately $2.4\times$ and $2.6\times$ lower regret, respectively. Detailed results are shown in \cref{tab:linear_knapsack}.

\subsection{Hyperparameter Selection}
\label{appsubsec:hyperparameter-selection}

Hyperparameters were not selected by automated search. The RSPO+-specific parameters ($\lambda_j$, $\eta_j$) were set using the principled per-client defaults derived in \cref{eq:calibration}. The remaining parameters (rounds, epochs, batch size, client fraction) were fixed to standard values from the FL literature. A sensitivity analysis over personalization parameters for the knapsack experiment is available in the supplementary code and the results are shown in \cref{fig:multifig-appendix}.

\begin{figure}
    \centering
    \begin{minipage}{0.333\linewidth}
        \centering
        \resizebox{\linewidth}{!}{\newcommand{\figfont}{\large}
\begin{tikzpicture}[font=\Large]
\begin{axis}[
scale only axis,
  legend style={
    at={(0.5,1.05)}, anchor=south,
    draw=none, fill=none,
    font=\Large,
    legend columns=-1,
    column sep=8pt
  },
  ylabel={Relative Regret},
  xlabel={Freezing criterion $\tau$},
  width=9cm, height=6cm,
  xmin=0.45, xmax=3.55,
  ymin=0, ymax=0.082,
  xtick={1, 2, 3},
  xticklabels={$1\%$, $5\%$, $10\%$},
    yticklabel = {%
  \pgfmathparse{100*\tick}%
  \pgfmathprintnumber[fixed,precision=0]{\pgfmathresult}\%
},
  ytick={0, 0.02, 0.04, 0.06, 0.08},
  tick align=outside,
  scaled x ticks=false,
  scaled y ticks=false,
  ymajorgrids=false,
  xmajorgrids=false,
  axis line style={draw=black},
  tick style={draw=black},
  clip=false,
  axis lines=left,
]

\addplot+[
  ybar, bar width=9pt, bar shift=-4.5pt,
  fill=paperggplotrspo, draw=black, line width=0.4pt,
  mark=none
] coordinates {
  (1.0, 0.03413376074714526)
  (2.0, 0.0413656)
  (3.0, 0.04629643806916572)
};

\addplot+[
  ybar, bar width=9pt, bar shift=4.5pt,
  fill=paperggplotspo, draw=black, line width=0.4pt,
  mark=none
] coordinates {
  (1.0, 0.07023837314775863)
  (2.0, 0.07023837314775863)
  (3.0, 0.07023837314775863)
};

% Tau = 1
\draw[->, line width=0.9pt, paperggplotdelta]
  ([xshift=-4.5pt]axis cs:1, 0.034134) -- ([xshift=-4.5pt]axis cs:1, 0.070238);
\draw[dash pattern=on 2pt off 1.5pt, line width=0.4pt, paperggplotdelta]
  ([xshift=4.5pt]axis cs:1, 0.070238) -- ([xshift=-9pt]axis cs:1, 0.070238);
\node[font=\figfont, inner sep=1pt, text=paperggplotdelta, anchor=west]
  at ([xshift=-32pt]axis cs:1, 0.052186) {$3.6\%$};

% Tau = 2
\draw[->, line width=0.9pt, paperggplotdelta]
  ([xshift=-4.5pt]axis cs:2, 0.037880) -- ([xshift=-4.5pt]axis cs:2, 0.070238);
\draw[dash pattern=on 2pt off 1.5pt, line width=0.4pt, paperggplotdelta]
  ([xshift=4.5pt]axis cs:2, 0.070238) -- ([xshift=-9pt]axis cs:2, 0.070238);
\node[font=\figfont, inner sep=1pt, text=paperggplotdelta, anchor=west]
  at ([xshift=-32pt]axis cs:2, 0.054059) {$3.2\%$};

% Tau = 3
\draw[->, line width=0.9pt, paperggplotdelta]
  ([xshift=-4.5pt]axis cs:3, 0.046296) -- ([xshift=-4.5pt]axis cs:3, 0.070238);
\draw[dash pattern=on 2pt off 1.5pt, line width=0.4pt, paperggplotdelta]
  ([xshift=4.5pt]axis cs:3, 0.070238) -- ([xshift=-9pt]axis cs:3, 0.070238);
\node[font=\figfont, inner sep=1pt, text=paperggplotdelta, anchor=west]
  at ([xshift=-32pt]axis cs:3, 0.058267) {$2.4\%$};

\end{axis}
\end{tikzpicture}}
    \end{minipage}%
    \begin{minipage}{0.333\linewidth}
        \centering
        \resizebox{\linewidth}{!}{\newcommand{\figfont}{\large}
\begin{tikzpicture}[font=\Large]
\begin{axis}[
scale only axis,
  legend style={
    at={(0.5,0.9)}, anchor=south,
    draw=none, fill=none,
    font=\Large,
    legend columns=-1,
    column sep=8pt
  },
  ylabel={Relative Regret},
  xlabel={$\kappa_\lambda$},
  width=9cm, height=6cm,
  xmin=0.45, xmax=5.55,
  ymin=0, ymax=0.082,
  xtick={1, 2, 3, 4, 5},
  xticklabels={0, 0.5, 1, 1.5, 2},
    yticklabel = {%
  \pgfmathparse{100*\tick}%
  \pgfmathprintnumber[fixed,precision=0]{\pgfmathresult}\%
},
  ytick={0, 0.02, 0.04, 0.06, 0.08},
  tick align=outside,
  scaled x ticks=false,
  scaled y ticks=false,
  ymajorgrids=false,
  xmajorgrids=false,
  axis line style={draw=black},
  tick style={draw=black},
  clip=false,
  axis lines=left,
]

\addplot+[
  ybar, bar width=9pt, bar shift=-4.5pt,
  fill=paperggplotrspo, draw=black, line width=0.4pt,
  mark=none,
  legend image code/.code={
    \draw[draw=black, line width=0.4pt, fill=paperggplotrspo]
      (0cm,-0.1cm) rectangle (0.3cm,0.2cm);}
] coordinates {
  (1.0, 0.03799537477828649)
  (2.0, 0.02276967316105169)
  (3.0, 0.020226868234788822)
  (4.0, 0.01899967589710587)
  (5.0, 0.018177237723071768)
};
\addlegendentry{RSPO+}

\addplot+[
  ybar, bar width=9pt, bar shift=4.5pt,
  fill=paperggplotspo, draw=black, line width=0.4pt,
  mark=none,
  legend image code/.code={
    \draw[draw=black, line width=0.4pt, fill=paperggplotspo]
      (0cm,-0.1cm) rectangle (0.3cm,0.2cm);}
] coordinates {
  (1.0, 0.07023837314775863)
  (2.0, 0.07023837314775863)
  (3.0, 0.07023837314775863)
  (4.0, 0.07023837314775863)
  (5.0, 0.07023837314775863)
};
\addlegendentry{SPO+}

% kappa_lambda = 1
\draw[->, line width=0.9pt, paperggplotdelta]
  ([xshift=-4.5pt]axis cs:1, 0.037995) -- ([xshift=-4.5pt]axis cs:1, 0.070238);
\draw[dash pattern=on 2pt off 1.5pt, line width=0.4pt, paperggplotdelta]
  ([xshift=4.5pt]axis cs:1, 0.070238) -- ([xshift=-9pt]axis cs:1, 0.070238);
\node[font=\figfont, inner sep=1pt, text=paperggplotdelta, anchor=west]
  at ([xshift=-28pt]axis cs:1, 0.054117) {$3.2\%$};

% kappa_lambda = 2
\draw[->, line width=0.9pt, paperggplotdelta]
  ([xshift=-4.5pt]axis cs:2, 0.022770) -- ([xshift=-4.5pt]axis cs:2, 0.070238);
\draw[dash pattern=on 2pt off 1.5pt, line width=0.4pt, paperggplotdelta]
  ([xshift=4.5pt]axis cs:2, 0.070238) -- ([xshift=-9pt]axis cs:2, 0.070238);
\node[font=\figfont, inner sep=1pt, text=paperggplotdelta, anchor=west]
  at ([xshift=-32pt]axis cs:2, 0.046504) {$4.7\%$};

% kappa_lambda = 3
\draw[->, line width=0.9pt, paperggplotdelta]
  ([xshift=-4.5pt]axis cs:3, 0.020227) -- ([xshift=-4.5pt]axis cs:3, 0.070238);
\draw[dash pattern=on 2pt off 1.5pt, line width=0.4pt, paperggplotdelta]
  ([xshift=4.5pt]axis cs:3, 0.070238) -- ([xshift=-9pt]axis cs:3, 0.070238);
\node[font=\figfont, inner sep=1pt, text=paperggplotdelta, anchor=west]
  at ([xshift=-32pt]axis cs:3, 0.045233) {$5.0\%$};

% kappa_lambda = 4
\draw[->, line width=0.9pt, paperggplotdelta]
  ([xshift=-4.5pt]axis cs:4, 0.019000) -- ([xshift=-4.5pt]axis cs:4, 0.070238);
\draw[dash pattern=on 2pt off 1.5pt, line width=0.4pt, paperggplotdelta]
  ([xshift=4.5pt]axis cs:4, 0.070238) -- ([xshift=-9pt]axis cs:4, 0.070238);
\node[font=\figfont, inner sep=1pt, text=paperggplotdelta, anchor=west]
  at ([xshift=-32pt]axis cs:4, 0.044619) {$5.1\%$};

% kappa_lambda = 5
\draw[->, line width=0.9pt, paperggplotdelta]
  ([xshift=-4.5pt]axis cs:5, 0.018177) -- ([xshift=-4.5pt]axis cs:5, 0.070238);
\draw[dash pattern=on 2pt off 1.5pt, line width=0.4pt, paperggplotdelta]
  ([xshift=4.5pt]axis cs:5, 0.070238) -- ([xshift=-9pt]axis cs:5, 0.070238);
\node[font=\figfont, inner sep=1pt, text=paperggplotdelta, anchor=west]
  at ([xshift=-32pt]axis cs:5, 0.044208) {$5.2\%$};

\end{axis}
\end{tikzpicture}}
    \end{minipage}%
    \begin{minipage}{0.333\linewidth}
        \centering
        \resizebox{\linewidth}{!}{\newcommand{\figfont}{\large}
\begin{tikzpicture}[font=\Large]
\begin{axis}[
scale only axis,
  legend style={
    at={(0.5,1.05)}, anchor=south,
    draw=none, fill=none,
    font=\Large,
    legend columns=-1,
    column sep=8pt
  },
  ylabel={Relative Regret},
  xlabel={$\kappa_{\mathrm{LR}}$},
  width=9cm, height=6cm,
  xmin=0.45, xmax=3.55,
  ymin=0, ymax=0.082,
  xtick={1, 2, 3},
  xticklabels={0.2, 0.5, 0.8},
    yticklabel = {%
  \pgfmathparse{100*\tick}%
  \pgfmathprintnumber[fixed,precision=0]{\pgfmathresult}\%
},
  ytick={0, 0.02, 0.04, 0.06, 0.08},
  tick align=outside,
  scaled x ticks=false,
  scaled y ticks=false,
  ymajorgrids=false,
  xmajorgrids=false,
  axis line style={draw=black},
  tick style={draw=black},
  clip=false,
  axis lines=left,
]

\addplot+[
  ybar, bar width=9pt, bar shift=-4.5pt,
  fill=paperggplotrspo, draw=black, line width=0.4pt,
  mark=none
] coordinates {
  (1.0, 0.023005081488838338)
  (2.0, 0.028120062014441602)
  (3.0, 0.03196318201721555)
};

\addplot+[
  ybar, bar width=9pt, bar shift=4.5pt,
  fill=paperggplotspo, draw=black, line width=0.4pt,
  mark=none
] coordinates {
  (1.0, 0.07023837314775863)
  (2.0, 0.07023837314775863)
  (3.0, 0.07023837314775863)
};

% kappa_lr = 1
\draw[->, line width=0.9pt, paperggplotdelta]
  ([xshift=-4.5pt]axis cs:1, 0.023005) -- ([xshift=-4.5pt]axis cs:1, 0.070238);
\draw[dash pattern=on 2pt off 1.5pt, line width=0.4pt, paperggplotdelta]
  ([xshift=4.5pt]axis cs:1, 0.070238) -- ([xshift=-9pt]axis cs:1, 0.070238);
\node[font=\figfont, inner sep=1pt, text=paperggplotdelta, anchor=west]
  at ([xshift=-32pt]axis cs:1, 0.046622) {$4.7\%$};

% kappa_lr = 2
\draw[->, line width=0.9pt, paperggplotdelta]
  ([xshift=-4.5pt]axis cs:2, 0.028120) -- ([xshift=-4.5pt]axis cs:2, 0.070238);
\draw[dash pattern=on 2pt off 1.5pt, line width=0.4pt, paperggplotdelta]
  ([xshift=4.5pt]axis cs:2, 0.070238) -- ([xshift=-9pt]axis cs:2, 0.070238);
\node[font=\figfont, inner sep=1pt, text=paperggplotdelta, anchor=west]
  at ([xshift=-32pt]axis cs:2, 0.049179) {$4.2\%$};

% kappa_lr = 3
\draw[->, line width=0.9pt, paperggplotdelta]
  ([xshift=-4.5pt]axis cs:3, 0.031963) -- ([xshift=-4.5pt]axis cs:3, 0.070238);
\draw[dash pattern=on 2pt off 1.5pt, line width=0.4pt, paperggplotdelta]
  ([xshift=4.5pt]axis cs:3, 0.070238) -- ([xshift=-9pt]axis cs:3, 0.070238);
\node[font=\figfont, inner sep=1pt, text=paperggplotdelta, anchor=west]
  at ([xshift=-32pt]axis cs:3, 0.051101) {$3.8\%$};

\end{axis}
\end{tikzpicture}}
    \end{minipage}
    \caption{(a) Relative regret (\%) of SPO+ and RSPO+ across different (a) values of the freezing threshold $\tau$, (b) values of $\kappa^\lambda$, and (c) values of $\kappa^{\mathrm{LR}}$. Lower is better ($\downarrow$).}
    \label{fig:multifig-appendix}
\end{figure}
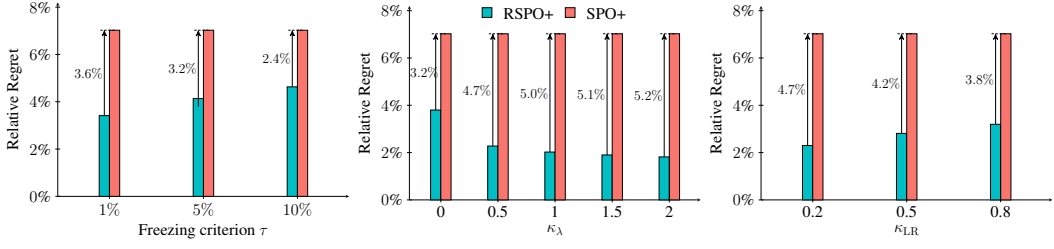
\newpage
\subsection{Algorithms}
\label{appsubsec:results}
\begin{algorithm}[H]
\caption{FedRSPO+: federated warm start with co-annealing scheduling and personalization.}
\label{alg:fedrspoplus}
\DontPrintSemicolon
\SetKwFunction{LocalTrain}{LocalTrain}
\SetKwFunction{StopMetric}{StopMetric}
\KwIn{client datasets $\{(\mathcal D_i^{\mathrm{tr}},\mathcal V_i,n_i)\}_{i=1}^N$, rounds $R_{\mathrm{fed}},R_{\mathrm{per}}$, local epochs $E$, participation rate $\rho$, freeze threshold $\tau$, post-freeze budget $K$}
initialize global model $\phi^0$\;
\For{$i=1$ \KwTo $N$}{
    initialize per-client schedules $(\lambda_{i,0},\,\eta_{i,0})$\;
}
global freeze flag $g\gets 0$\;
\medskip
\tcc{Phase 1: Federated training}
\For{$r=0$ \KwTo $R_{\mathrm{fed}}-1$}{
    \lIf{$g=0$}{advance schedule counters $\lambda_i^{(r)},\eta_i^{(r)}$ for all $i$}
    sample $I_r\subseteq\{1,\dots,N\}$ with $|I_r|=\lceil\rho N\rceil$\;
    \ForAll{$i\in I_r$ \textbf{in parallel}}{
        $\tilde{\phi}_i^{r+1}\gets{}$\LocalTrain{$\phi^r,\mathcal D_i^{\mathrm{tr}},\lambda_i^{(r)},\eta_i^{(r)},E$}\;
        $\Delta_i^r\gets\tilde{\phi}_i^{r+1}-\phi^r$\;
    }
    $\phi^{r+1}\gets\phi^r+\sum_{i\in I_r}\frac{n_i}{\sum_{j\in I_r}n_j}\,\Delta_i^r$\;
    \lIf{\StopMetric{$\phi^{r+1},\{\mathcal V_i\},\{\lambda_i^{(r)}\}$} $\le\tau$}{$g\gets 1$ \tcp*{freeze annealing}}
}
\medskip
\tcc{Phase 2: Personalization (warm-started from $\phi^{R_{\mathrm{fed}}}$)}
\ForAll{$i=1,\dots,N$ \textbf{in parallel}}{
    $\phi_i^{\mathrm{per}}\gets\phi^{R_{\mathrm{fed}}}$\;
    initialize personal schedules $(\lambda_{i,\mathrm{per}}^{(0)},\,\eta_{i,\mathrm{per}}^{(0)})$\;
    client freeze flag $g_i\gets 0$,\; freeze round $t_i\gets+\infty$\;
    \For{$r=0$ \KwTo $R_{\mathrm{per}}-1$}{
        \lIf{$r>t_i+K$}{\textbf{break}}
        \lIf{$g_i=0$}{advance personal schedule counters $\lambda_{i,\mathrm{per}}^{(r)},\eta_{i,\mathrm{per}}^{(r)}$}
        $\phi_i^{\mathrm{per}}\gets{}$\LocalTrain{$\phi_i^{\mathrm{per}},\mathcal D_i^{\mathrm{tr}},\lambda_{i,\mathrm{per}}^{(r)},\eta_{i,\mathrm{per}}^{(r)},E$}\;
        \lIf{\StopMetric{$\phi_i^{\mathrm{per}},\mathcal V_i,\lambda_{i,\mathrm{per}}^{(r)}$} $\le\tau$ \textbf{and} $g_i=0$}{$g_i\gets 1$,\; $t_i\gets r$}
    }
}
\Return $\phi^{R_{\mathrm{fed}}}$ and $\{\phi_{i,\star}^{\mathrm{per}}\}_{i=1}^N$\;
\end{algorithm}
\begin{comment}
\subsection{Bounding Point-wise RSPO\texorpdfstring{+}{+} Losses}
\label{sec:rspo-bounds}

For the learning rate, we use $\eta_t = \eta_0 \frac{\lambda_t}{\lambda_0},$ with $\eta_0 = 10^{-2}.$

\paragraph{Stopping criterion} Then
\[
c_i^T \hat{w}_{i,t} - z_i^\star
=
\underbrace{c_i^T (\hat{w}_{i,t} - w_{i,t}^{\star,\mathrm{reg}})}_{\text{estimation error}}
+
\underbrace{(c_i^T w_{i,t}^{\star,\mathrm{reg}} - z_i^\star)}_{\text{regularization bias}}.
\]
The first term depends on model quality and may remain bounded away from zero due to noise or approximation error, whereas the second term is only affected by annealing $\lambda_t$. Since $w_{i,t}^{\star,\mathrm{reg}}$ depends only on $(c_i,\lambda_t)$, this criterion isolates the portion of the downstream gap controlled by the regularization schedule.
\end{comment}
\clearpage
\section*{NeurIPS Paper Checklist}

\begin{enumerate}

\item {\bf Claims}
    \item[] Question: Do the main claims made in the abstract and introduction accurately reflect the paper's contributions and scope?
    \item[] Answer: \answerYes{} % Replace by \answerYes{}, \answerNo{}, or \answerNA{}.
    \item[] Justification: The abstract and introduction clearly state our three contributions: the RSPO+ surrogate with regret/error bounds, FedRSPO+ with heterogeneity bounds that vanish at homogeneity, and empirical evaluation across three tasks.

\item {\bf Limitations}
    \item[] Question: Does the paper discuss the limitations of the work performed by the authors?
    \item[] Answer: \answerYes{} % Replace by \answerYes{}, \answerNo{}, or \answerNA{}.
    \item[] Justification: We discuss limitations of our approach in \cref{sec:conclusions}.

\item {\bf Theory assumptions and proofs}
    \item[] Question: For each theoretical result, does the paper provide the full set of assumptions and a complete (and correct) proof?
    \item[] Answer: \answerYes{} % Replace by \answerYes{}, \answerNo{}, or \answerNA{}.
    \item[] Justification: All proofs are provided in the Appendix, \cref{appsec:proofs}.

    \item {\bf Experimental result reproducibility}
    \item[] Question: Does the paper fully disclose all the information needed to reproduce the main experimental results of the paper to the extent that it affects the main claims and/or conclusions of the paper (regardless of whether the code and data are provided or not)?
    \item[] Answer: \answerYes{} % Replace by \answerYes{}, \answerNo{}, or \answerNA{}.
    \item[] Justification: We have implementation details for each experiment, including hyperparameter choices and data splits used, to reproduce all our results. All data sources are provided and well-documented. We also provide the anonymized code for reproducibility.

\item {\bf Open access to data and code}
    \item[] Question: Does the paper provide open access to the data and code, with sufficient instructions to faithfully reproduce the main experimental results, as described in supplemental material?
    \item[] Answer: \answerYes{} % Replace by \answerYes{}, \answerNo{}, or \answerNA{}.
    \item[] Justification: Our paper provides the anonymized code necessary to reproduce the experiments, and all data sources are open and well-documented.

\item {\bf Experimental setting/details}
    \item[] Question: Does the paper specify all the training and test details (e.g., data splits, hyperparameters, how they were chosen, type of optimizer) necessary to understand the results?
    \item[] Answer: \answerYes{} % Replace by \answerYes{}, \answerNo{}, or \answerNA{}.
    \item[] Justification: We explicitly provide all the experimental details and setups for each experiment in the Appendix, \cref{appsec:implementation-details}.

\item {\bf Experiment statistical significance}
    \item[] Question: Does the paper report error bars suitably and correctly defined or other appropriate information about the statistical significance of the experiments?
    \item[] Answer: \answerYes{} % Replace by \answerYes{}, \answerNo{}, or \answerNA{}.
    \item[] Justification: We report mean $\pm 1\sigma$ (standard deviation) over 5 random seeds for all experiments and describe error bars and methodology in the Appendix, \cref{appsec:implementation-details}.

\item {\bf Experiments compute resources}
    \item[] Question: For each experiment, does the paper provide sufficient information on the computer resources (type of compute workers, memory, time of execution) needed to reproduce the experiments?
    \item[] Answer: \answerYes{} % Replace by \answerYes{}, \answerNo{}, or \answerNA{}.
    \item[] Justification: We provide the computer resources in the implementation details in the Appendix.
    
\item {\bf Code of ethics}
    \item[] Question: Does the research conducted in the paper conform, in every respect, with the NeurIPS Code of Ethics \url{https://neurips.cc/public/EthicsGuidelines}?
    \item[] Answer: \answerYes{} % Replace by \answerYes{}, \answerNo{}, or \answerNA{}.
    \item[] Justification: We read carefully and align with the Ethics Guidelines shown.

\item {\bf Broader impacts}
    \item[] Question: Does the paper discuss both potential positive societal impacts and negative societal impacts of the work performed?
    \item[] Answer: \answerYes{} % Replace by \answerYes{}, \answerNo{}, or \answerNA{}.
    \item[] Justification: We state the potential broader impacts in \cref{sec:conclusions}.
    
\item {\bf Safeguards}
    \item[] Question: Does the paper describe safeguards that have been put in place for responsible release of data or models that have a high risk for misuse (e.g., pre-trained language models, image generators, or scraped datasets)?
    \item[] Answer: \answerNA{} % Replace by \answerYes{}, \answerNo{}, or \answerNA{}.
    \item[] Justification: Our datasets are widely used ones, such as PJM energy markets and the Warcraft DFL benchmark. Their safety has been proven by existing work.

\item {\bf Licenses for existing assets}
    \item[] Question: Are the creators or original owners of assets (e.g., code, data, models), used in the paper, properly credited and are the license and terms of use explicitly mentioned and properly respected?
    \item[] Answer: \answerYes{} % Replace by \answerYes{}, \answerNo{}, or \answerNA{}.
    \item[] Justification: We cite all works and tools used to properly credit their authors.

\item {\bf New assets}
    \item[] Question: Are new assets introduced in the paper well documented and is the documentation provided alongside the assets?
    \item[] Answer: \answerYes{} % Replace by \answerYes{}, \answerNo{}, or \answerNA{}.
    \item[] Justification: We have documented our models and methods used, and release them along with the anonymized code.

\item {\bf Crowdsourcing and research with human subjects}
    \item[] Question: For crowdsourcing experiments and research with human subjects, does the paper include the full text of instructions given to participants and screenshots, if applicable, as well as details about compensation (if any)? 
    \item[] Answer: \answerNA{} % Replace by \answerYes{}, \answerNo{}, or \answerNA{}.
    \item[] Justification: This paper does not involve crowdsourcing or human subjects.

\item {\bf Institutional review board (IRB) approvals or equivalent for research with human subjects}
    \item[] Question: Does the paper describe potential risks incurred by study participants, whether such risks were disclosed to the subjects, and whether Institutional Review Board (IRB) approvals (or an equivalent approval/review based on the requirements of your country or institution) were obtained?
    \item[] Answer: \answerNA{} % Replace by \answerYes{}, \answerNo{}, or \answerNA{}.
    \item[] Justification: This paper does not need to do IRB approvals or equivalent since it does not involve human subject research.

\item {\bf Declaration of LLM usage}
    \item[] Question: Does the paper describe the usage of LLMs if it is an important, original, or non-standard component of the core methods in this research? Note that if the LLM is used only for writing, editing, or formatting purposes and does \emph{not} impact the core methodology, scientific rigor, or originality of the research, declaration is not required.
    %this research? 
    \item[] Answer: \answerNA{} % Replace by \answerYes{}, \answerNo{}, or \answerNA{}.
    \item[] Justification: We do not use LLMs for any non-standard uses.

\end{enumerate}
\end{document}